\documentclass{article}
\pdfoutput=1
\usepackage{arxiv}
\setcitestyle{authoryear,open={(},close={)}}
\renewcommand{\cite}{\citep}

\newenvironment{noverticalspace}
{%
	\par %
	\offinterlineskip %
}
{\par}%

\usepackage[utf8]{inputenc} %
\usepackage[T1]{fontenc}    %
\usepackage{hyperref}       %
\usepackage{url}            %
\usepackage{booktabs}       %
\usepackage{amsfonts}       %
\usepackage{nicefrac}       %
\usepackage{microtype, color}      %

\usepackage{tabulary}
\usepackage{tabularx}
\usepackage{array}
\usepackage{algorithm, algorithmicx}
\usepackage[noend]{algpseudocode}
\usepackage{amsmath}
\usepackage{amssymb}
\usepackage{amsthm}
\usepackage{bbm}
\usepackage{graphicx}
\usepackage{indentfirst}

\newlength\savewidth\newcommand\shline{\noalign{\global\savewidth\arrayrulewidth
		\global\arrayrulewidth 1pt}\hline\noalign{\global\arrayrulewidth\savewidth}}
\newcommand{\tablestyle}[2]{\setlength{\tabcolsep}{#1}\renewcommand{\arraystretch}{#2}\centering\footnotesize}

\newtheorem{theorem}{Theorem}[section]
\newtheorem{claim}[theorem]{Claim}
\newtheorem{definition}[theorem]{Definition}
\newtheorem{lemma}[theorem]{Lemma}

\newtheorem{proposition}[theorem]{Proposition}

\newtheorem{example}[theorem]{Example}
\newtheorem{assumption}[theorem]{Assumption}

\makeatletter
\newtheorem*{rep@definition}{\rep@title}
\newcommand{\newrepdefinition}[2]{%
	\newenvironment{rep#1}[1]{%
		\def\rep@title{#2 \ref{##1}}%
		\begin{rep@definition}}%
		{\end{rep@definition}}}
\makeatother

\newrepdefinition{definition}{Definition}
\newrepdefinition{lemma}{Lemma}
\newrepdefinition{proposition}{Proposition}

\def\shownotes{1}  \ifnum\shownotes=0
\newcommand{\authnote}[2]{{[#1: #2]}}
\else
\newcommand{\authnote}[2]{}
\fi

\newcommand{\tnote}[1]{{\color{orange}\authnote{TM}{#1}}}
\newcommand{\jnote}[1]{{\color{purple}\authnote{JH}{#1}}}
\def\shownotes{1}  \ifnum\shownotes=1
\newcommand{\authornotenonurgent}[2]{{[#1: #2]}}
\else
\newcommand{\authornotenonurgent}[2]{}
\fi

\title{Provable Guarantees for Self-Supervised Deep Learning with Spectral Contrastive Loss}

\author{%
 \large{Jeff Z. HaoChen$^{1}$~~~~ Colin Wei$^{1}$~~~~ Adrien Gaidon$^{2}$~~~~ Tengyu Ma$^{1}$}\\
\vspace{0.1cm}\\
\large{$^1$ Stanford University~~~ $^2$ Toyota Research Institute}\\
\vspace{0.01cm}\\
{{\texttt{\{jhaochen, colinwei, tengyuma\}@stanford.edu}}~~~~ \texttt{adrien.gaidon@tri.global}}
}
\ifdefined\usebigfont

\usepackage{times}
\usepackage[fontsize=13pt]{scrextend}
\AtBeginDocument{
	\newgeometry{left=1.56in,right=1.56in,top=1.71in,bottom=1.77in}
}
\else
\fi

\newcommand{\nds}{{n_{\rm{down}}}}
\newcommand{\npt}{{n_{\rm{pre}}}}

\newcommand{\R}{\mathbb{R}}

\DeclareMathOperator*{\argmax}{arg\,max}
\DeclareMathOperator*{\argmin}{arg\,min}

\DeclareMathOperator{\Tr}{Tr}

\newcommand{\setgaussian}{R}

\newcommand{\eigvadj}{\gamma}

\newcommand{\hyp}{\mathcal{F}}

\newcommand{\globalminf}{{f^*_{\textup{pop}}}}
\newcommand{\empminf}{{\hat{f}_{\textup{emp}}}}

\newcommand{\rad}[1]{\widehat{\mathcal{R}}_{#1}}
\newcommand{\fstarf}{f^*_{\mathcal{F}}}
\newcommand{\boundx}{{C_x}}
\newcommand{\boundwi}[1]{{C_{w, #1}}}
\newcommand{\boundw}{{C_w}}
\newcommand{\boundf}{{\kappa}}
\newcommand{\Exp}[1]{\mathrm{\mathbb{E}}_{#1}}
\newcommand{\wpair}[2]{w_{#1#2}}
\newcommand{\wnode}[1]{w_{#1}}
\newcommand{\ballr}{\mu}
\newcommand{\ma}{\textup{mf}}
\newcommand{\sizead}{N}
\newcommand{\scf}{\textup{sc}}

\newcommand{\eigv}{{v}}
\newcommand{\matrixw }{{B}}
\newcommand{\vectorw}{{b}}
\newcommand{\laplacian}{\mathcal{L}}
\newcommand{\id}[1]{\mathbbm{1}\left[#1\right]}
\newcommand{\Real}{\mathbb{R}}
\newcommand{\norm}[1]{\left\lVert#1\right\rVert}

\newcommand{\Loss}[1]{\mathcal{L}({#1})}
\newcommand{\eLoss}[2]{\widehat{\mathcal{L}}_{#1}({#2})}
\newcommand{\sLoss}[1]{\widehat{\mathcal{L}}_S({#1})}
\newcommand{\Lossmc}[1]{\mathcal{L}_{\ma}({#1})}

\newcommand{\poly}{\textup{poly}}

\newcommand{\edata}{\widehat{\mathcal{X}}}
\newcommand{\aug}[1]{\mathcal{A}(\cdot|#1)}
\newcommand{\augp}[2]{\mathcal{A}(#1|#2)}
\newcommand{\pndata}{\mathcal{P}_{\overline{\mathcal{X}}}}
\newcommand{\epndata}{\hat{\mathcal{P}}_{\mathcal{X}}}
\newcommand{\pred}{g}
\newcommand{\npred}{\bar{g}}
\newcommand{\ndata}{\overline{\mathcal{X}}}
\newcommand{\adata}{\mathcal{X}}

\begin{document}
	
\maketitle

\begin{abstract}
	Recent works in self-supervised learning have advanced the state-of-the-art by relying on the \textit{contrastive learning} paradigm, which learns representations by pushing positive pairs, or similar examples from the same class, closer together while keeping negative pairs far apart. Despite the empirical successes, theoretical foundations are limited -- prior analyses assume conditional independence of the positive pairs given the same class label, but recent empirical applications use heavily correlated positive pairs (i.e., data augmentations of the same image). Our work analyzes contrastive learning without assuming conditional independence of positive pairs using a novel concept of the \textit{augmentation graph} on data.  Edges in this graph connect augmentations of the same datapoint, and ground-truth classes naturally form connected sub-graphs. 
	We propose a loss that performs spectral decomposition on the population augmentation graph and can be succinctly written as a contrastive learning objective on neural net representations.
	Minimizing this objective leads to features with provable accuracy guarantees under linear probe evaluation. By standard generalization bounds, these accuracy guarantees also hold when minimizing the training contrastive loss. 
Empirically, the features learned by our objective can match or outperform several strong baselines on benchmark vision datasets. 
	In all, this work provides the first provable analysis for contrastive learning where guarantees for linear probe evaluation  can apply to realistic empirical settings.
\end{abstract}

\section{Introduction}


Recent empirical breakthroughs have demonstrated the effectiveness of self-supervised learning, which trains representations on unlabeled data with surrogate losses and self-defined supervision signals~\cite{wu2018unsupervised, oord2018representation, hjelm2018learning, ye2019unsupervised, henaff2020data, bachman2019learning, tian2019contrastive, misra2020self, caron2020unsupervised, zbontar2021barlow, bardes2021vicreg, tian2020makes, chen2020exploring}. Self-supervision signals in computer vision are often defined by using data augmentation to produce multiple views of the same image. For example, the recent contrastive learning objectives~\cite{arora2019theoretical, chen2020simple, chen2020big, he2020momentum, chen2020improved}  encourage closer representations for augmentations/views of the same natural datapoint than for randomly sampled pairs of data. 

Despite the empirical successes, there is a limited theoretical understanding of why self-supervised losses learn representations that can be adapted to downstream tasks, for example, using linear heads. 
Recent mathematical analyses for contrastive learning by \citet{arora2019theoretical, tosh2020contrastive,tosh2021contrastive}  provide guarantees under the assumption that two views are somewhat conditionally independent
given the label or a hidden variable.
However, in practical algorithms for computer vision applications, the two views are augmentations of a natural image and usually exhibit a strong correlation that is difficult to be de-correlated by conditioning. They are not independent conditioned on the label, and we are only aware that they are conditionally independent given the natural image, which is too complex to serve as a hidden variable with which prior works can be meaningfully applied. 
Thus the existing theory does not appear to explain the practical success of self-supervised learning.

This paper presents a theoretical framework for self-supervised learning without requiring conditional independence. 
We design a principled, practical loss function for learning neural net representations that resembles state-of-the-art contrastive learning methods. We prove that, under a simple and realistic data assumption, linear classification using representations learned on a polynomial number of unlabeled data samples can recover the ground-truth labels of the data with high accuracy.

The fundamental data property that we leverage is a notion of continuity of the \textit{population} data within the same class. 
Though a random pair of images from the same class can be far apart, the pair is often connected by (many) sequences of natural images, where consecutive images in the sequences are close neighbors within the same class. 
As shown in Figure~\ref{figure:augmentation_graph} (images on the left top part), two very different French bulldogs can be connected by a sequence of French bulldogs (which may not be in the training set but are in the support of the population distribution). 
Prior work~\cite{wei2020theoretical} empirically demonstrates this type of connectivity property and uses it in the analysis of pseudolabeling algorithms.
This property is more salient when the neighborhood of an example includes many different types of augmentations.

More formally, we define the \textit{population augmentation graph}, whose vertices are all the augmented data in the \textit{population} distribution, which can be an exponentially large or infinite set.  Two vertices are connected with an edge if they are augmentations of the same natural example. Our main assumption is that for some proper $m\in\mathcal{Z}^+$, we cannot partition the graph into $m+1$ sub-graphs between which there are few connections (Assumption~\ref{assumption:at_most_m_clusters}). 
In other words, this intuitively states that there are at most $m$ clusters in the population augmentation graph. 
This assumption can be seen as a graph-theoretic version of the continuity assumption on the \textit{population} distribution. We also assume that there are very few edges across different ground-truth classes (Assumption~\ref{definition:accurate_partition}). 
Figure~\ref{figure:augmentation_graph} (left) illustrates a realistic scenario where dog and cat are the ground-truth categories, between which edges are very rare. Each breed forms a sub-graph that has sufficient inner connectivity and thus cannot be further partitioned.

Our assumption fundamentally does not require independence of the two views (the positive pairs) conditioned on the class and can allow disconnected sub-graphs within a class.
The classes in the downstream task can be also somewhat flexible as long as they are disconnected in the augmentation graph. 
For example, when the augmentation graph consists of $m$ disconnected sub-graphs corresponding to fine-grained classes, our assumptions allow the downstream task to have any $r \le m$ coarse-grained classes containing these fine-grained classes as a sub-partition. 
Prior work~\cite{wei2020theoretical} on pseudolabeling algorithms essentially requires an exact alignment between sub-graphs and downstream classes (i.e., $r=m$). They face this limitation because their analysis requires fitting discrete pseudolabels on the unlabeled data. We avoid this difficulty because we consider directly learning continuous representations on the unlabeled data.

\begin{figure*}
	\setlength{\lineskip}{0pt}
	\centering
	\begin{tabular}{p{0.55\textwidth} p{0.4\textwidth}}
		\begin{noverticalspace}
			\includegraphics[width=0.55\textwidth]{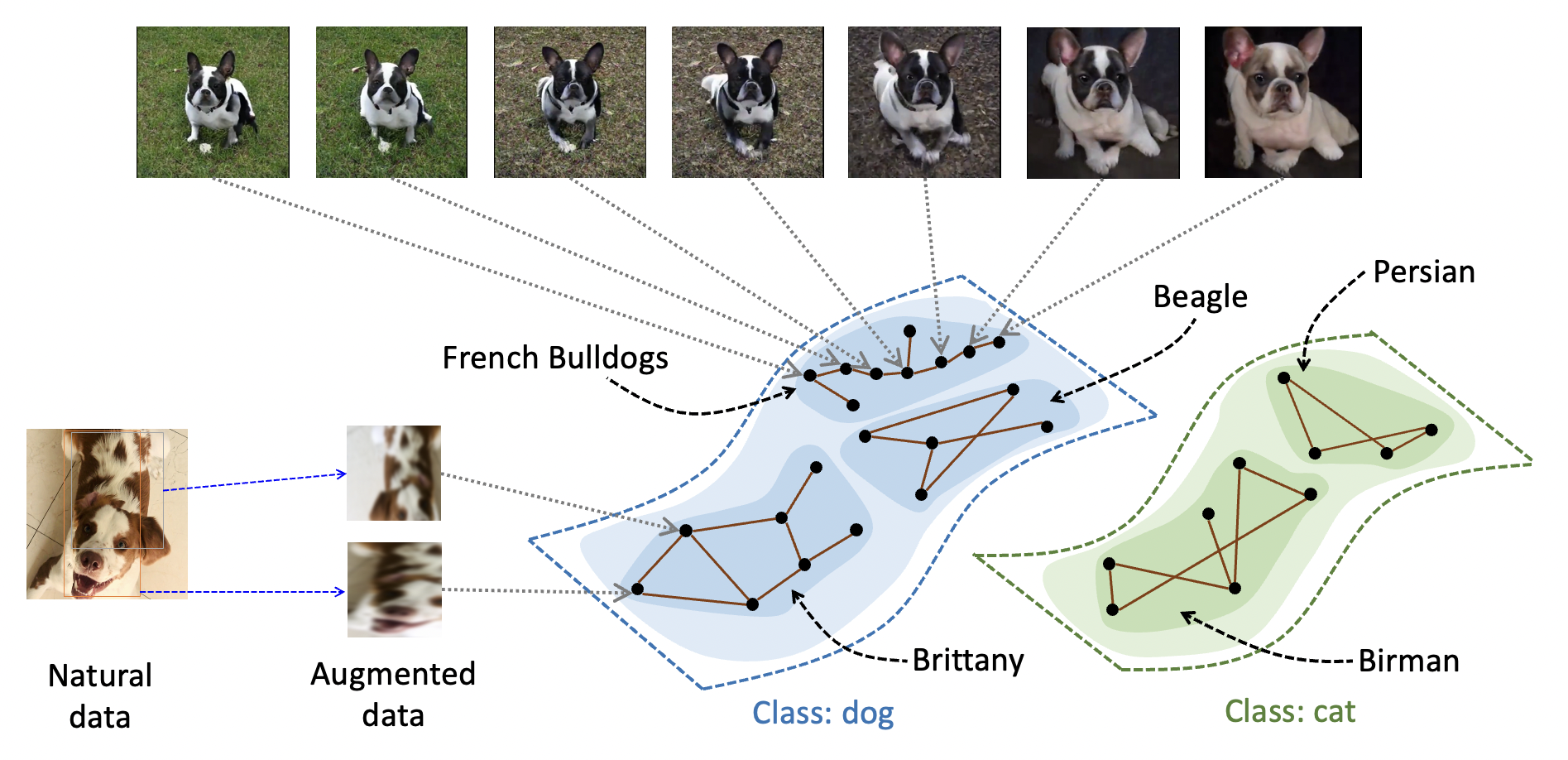}%
		\end{noverticalspace} 
		&
		\begin{noverticalspace}
			\includegraphics[width=0.4\textwidth]{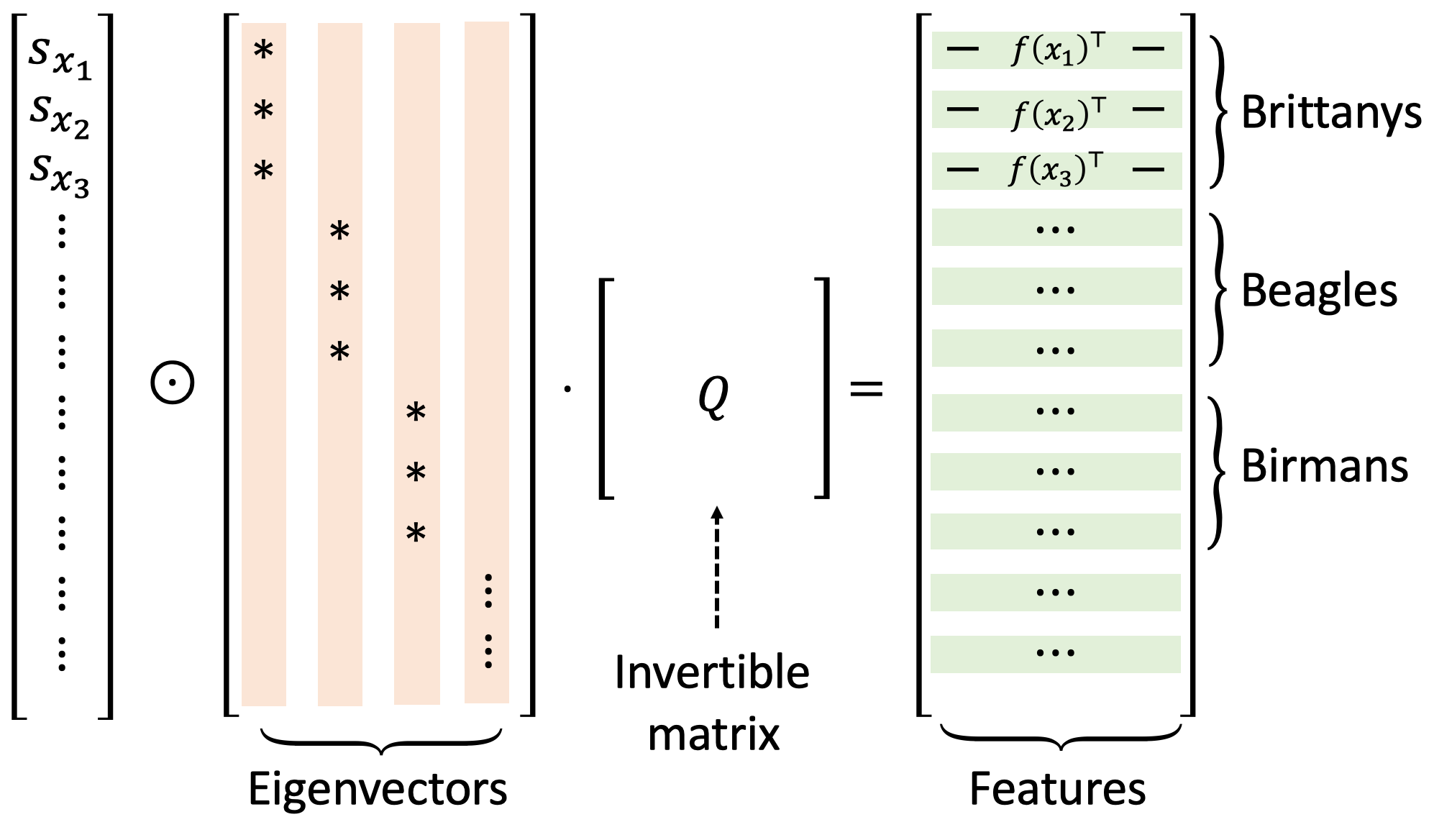}%
		\end{noverticalspace}

	\end{tabular}
	\label{figure:augmentation_graph}
	\caption[..]{
		\textbf{Left: demonstration of the population augmentation graph.} 
		Two augmented data are connected if they are views of the same natural datapoint. Augmentations of data from different classes in the downstream tasks are assumed to be nearly disconnected, whereas there are more connections within the same class. We allow the existence of disconnected sub-graphs within a class corresponding to potential sub-classes.
		\textbf{Right: decomposition of the learned representations.} The representations  (rows in the RHS) learned by minimizing the population spectral contrastive loss can be decomposed as the LHS. The 
		scalar $s_{x_i}$ is positive for every augmented datapoint $x_i$. Columns of the matrix labeled ``eigenvectors'' are the top eigenvectors of the normalized adjacency matrix of the augmentation graph defined in Section~\ref{section:augmentation_graph}. The
		operator $\odot$ multiplies row-wise each $s_{x_i}$ with the $x_i$-th row of the eigenvector matrix. When classes (or sub-classes) are exactly disconnected in the augmentation graph, the eigenvectors are sparse and align with the sub-class structure. The invertible $Q$ matrix does not affect the performance of the rows under the linear probe.  
	}	
\end{figure*}

The main insight of the paper is that contrastive learning can be viewed as a parametric form of spectral clustering~\cite{ng2001spectral,shi2000normalized} on the augmentation graph. 
Concretely, suppose we apply spectral decomposition or spectral clustering---a classical approach for graph partitioning---to the adjacency matrix defined on the population augmentation graph. We form a matrix where the top-$k$ eigenvectors are the columns and interpret each row of the matrix as the representation (in $\Real^k$) of an example. 
Somewhat surprisingly, we prove that this feature extractor can be also recovered (up to some linear transformation) by minimizing the following population objective which is a variant of the standard contrastive loss: 
\begin{align*}
\Loss{f} = -2\cdot \Exp{x, x^+} \left[f(x)^\top f(x^+)\right] + \Exp{x, x^-} \big[\left(f(x)^\top f(x^-)\right)^2\big],
\end{align*}
where $(x,x^+)$ is a pair of augmentations of the same datapoint, $(x,x^-)$ is a pair of independently random augmented data, and $f$ is a parameterized function from augmented data to $\Real^k$. Figure~\ref{figure:augmentation_graph} (right) illustrates the relationship between the eigenvector matrix and the learned representations. 
We call this loss the \textit{population spectral contrastive loss}. 

We analyze the linear classification performance of the representations learned by minimizing the population spectral contrastive loss. Our main result (Theorem~\ref{thm:combine_with_cheeger_simplified}) shows that when the representation dimension exceeds the maximum number of disconnected sub-graphs, linear classification with learned representations is guaranteed to have a small error. Our theorem reveals a trend that a larger representation dimension is needed when there are a larger number of disconnected sub-graphs. Our analysis relies on novel techniques tailored to linear probe performance, which have not been studied in the spectral graph theory community to the best of our knowledge.

The spectral contrastive loss also works on empirical data. 
Since our approach optimizes parametric loss functions, guarantees involving the population loss can be converted to finite sample results using off-the-shelf generalization bounds. The end-to-end result (Theorem~\ref{theorem:end_to_end}) shows that the number of unlabeled examples required is polynomial in the Rademacher complexity of the model family and other relevant parameters, whereas the number of downstream labeled examples only needs to be linear in the representation dimension (which needs to be linear in the number of clusters in the graph). This demonstrates that contrastive learning reduces the amount of labeled examples needed. 

In summary, our main theoretical contributions are: 1) we propose a simple contrastive loss motivated by spectral decomposition of the population data graph, 2) under simple and realistic assumptions, we provide downstream classification guarantees for the representation learned by minimizing this loss on population data, and 3) our analysis is easily applicable to deep networks with polynomial unlabeled samples via off-the-shelf generalization bounds. Our theoretical framework can be viewed as containing two stages: we first analyze the population loss and the representation that minimizes it (Section~\ref{section:framework}), then study the empirical loss where the representation is learned with a neural network with bounded capacity (Section~\ref{section:finite_sample}).  

In addition, we implement and test the proposed spectral contrastive loss on standard vision benchmark datasets. Our algorithm is simple and doesn't rely on tricks such as stop-gradient which is essential to SimSiam~\cite{chen2020exploring}. We demonstrate that the features learned by our algorithm can match or outperform several strong baselines~\citep{chen2020simple,chen2020improved,chen2020exploring,grill2020bootstrap} when evaluated using a linear probe.



\section{Additional related works}

\noindent
\textbf{Empirical works on self-supervised learning.} 
Self-supervised learning algorithms have been shown to successfully learn representations that benefit downstream tasks~\cite{wu2018unsupervised, oord2018representation, hjelm2018learning, ye2019unsupervised, henaff2020data, bachman2019learning, tian2019contrastive, misra2020self, chen2020improved, chen2020simple, he2020momentum, chen2020big, caron2020unsupervised, zbontar2021barlow, bardes2021vicreg, tian2020makes, xie2019unsupervised}.
Many recent self-supervised learning algorithms learn features with siamese networks~\cite{bromley1993signature}, where two neural networks of shared weights are applied to pairs of augmented data. 
Introducing asymmetry to siamese networks either with a momentum encoder like BYOL~\cite{grill2020bootstrap} or by stopping gradient propagation for one branch of the siamese network like SimSiam~\cite{chen2020exploring} has been shown to effectively avoid collapsing. 
Contrastive methods~\cite{chen2020simple, he2020momentum, chen2020improved} minimize the InfoNCE loss~\cite{oord2018representation}, where two views of the same data are attracted while views from different data are repulsed. 


\noindent
\textbf{Theoretical works on self-supervised learning.} 
As briefly discussed in the introduction, several theoretical works have studied self-supervised learning. 
\citet{arora2019theoretical} provide guarantees for representations learned by contrastive learning on downstream linear classification tasks under the assumption that the positive pairs are conditionally independent given the class label. Theorem 3.3 and Theorem 3.7 of the work of \citet{lee2020predicting} show that, under conditional independence given the label and/or additional latent variables, representations learned by reconstruction-based self-supervised learning algorithms can achieve small errors in the downstream linear classification task. \citet[Theorem 4.1]{lee2020predicting} generalizes it to approximate conditional independence for Gaussian data and Theorem 4.5 further weakens the assumptions significantly.
\citet{tosh2020contrastive} show that contrastive learning representations can linearly recover any continuous functions of the underlying topic posterior under a topic modeling assumption (which also requires conditional independence of the positive pair given the hidden variable). More recently, Theorem 11 of the work of~\citet{tosh2021contrastive} provide novel guarantees for contrastive learning under the assumption that there exists a hidden variable $h$ such that the positive pair $(x,x^+)$ are conditionally independent given $h$ and the random variable $p(x\vert h)p(x^+\vert h)/p(x)p(x^+)$ has a small variance. 
However, in practical algorithms for computer vision applications, the two views are two augmentations and thus they are highly correlated. They might be only independent when conditioned on very complex hidden variables such as the original natural image, which might be too complex for the previous results to be meaningfully applied. 

We can also compare the assumptions and results on a concrete generative model for the data, our Example~\ref{example:gaussian} in Section~\ref{section:gaussian_example}, 
where the data are generated by a mixture of Gaussian or a mixture of manifolds, the label is the index of the mixture,  and the augmentations are small Gaussian blurring (i.e., adding Gaussian noise). In this case, the positive pairs $(x,x^+)$ are two points that are very close to each other. To the best of our knowledge, applying Theorem 11 of~\citet{tosh2021contrastive} to this case with $h=\bar{x}$ (the natural datapoint) would result in requiring a large (if not infinite) representation dimension. Because $x^+$ and $x$ are very close, the reconstruction-based algorithms in~\citet{lee2020predicting}, when used to predict $x^+$ from $x$, will not be able to produce good representations as well.\footnote{On a technical level,  Example~\ref{example:gaussian} does not satisfy the  requirement regarding the $\beta$ quantity in Assumption 4.1 of~\citet{lee2020predicting}, if $(X_1,X_2)$ in that paper is equal to $(x,x+)$ here---it requires the label to be correlated with the raw input $x$, which is not necessarily true in Example~\ref{example:gaussian}. This can likely be addressed by using a different $X_2$. } 

On a technical level, to relate prior works' assumptions to ours, we can consider an almost equivalent version of our assumption (although our proofs do not directly rely on or relate to the discussion below). 
Let $(x,x^+)$ be a positive pair and let $p(\cdot \vert x)$ be the conditional distribution of $x^+$ given $x$. Starting from $x_0$, let us consider a hypothetical Markov chain $x_0,\dots, x_T, \cdots$ where $x_t$ is drawn from $p(\cdot \vert x_{t-1})$. Our assumption essentially means that this hypothetical Markov chain of sampling neighbors will mix within the same class earlier than it mixes across the entire population (which might not be possible or takes exponential time). More concretely, the assumption that $\rho_{\lfloor k/2\rfloor}$ is large compared to $\alpha$ in Theorem~\ref{thm:combine_with_cheeger_simplified} is roughly equivalent to the existence of a (potentially large)  $T$ such that $x_0$ and $x_T$ are still likely to have the same label, but are sufficiently independent conditioned on this label or some hidden variable. 
 Roughly speaking, prior works~\cite{arora2019theoretical, tosh2020contrastive,tosh2021contrastive} assume probabilistic structure about $x_0$ and $x_1$ (instead of $x_0$ and $x_T$), e.g., ~\citet{arora2019theoretical} and Theorem 11 of~\citet{tosh2021contrastive}  assume that $x_0$ and $x_1$ are independent conditioned on the label and/or a hidden variable. 
Similar Markov chains on augmentated data have also been used in previous work~\cite{dao2019kernel} to study properties of data augmentation.

Several other works~\citep{tsai2020self, wang2020understanding, tian2020understanding, bansal2020self,mitrovic2020representation} also theoretically study self-supervised learning. The work  \citet{tsai2020self} prove that self-supervised learning methods can extract task-relevant information and discard task-irrelevant information, but lacks guarantees for solving downstream tasks efficiently with simple (e.g., linear) models.
\citet{tian2020understanding} study why non-contrastive self-supervised learning methods can avoid feature collapse. \citet{zimmermann2021contrastive} prove that for a specific data generating process, contrastvie learning can learn representations that recover the latent variable. 
\citet{cai2021theory} analyze domain adaptation algorithms for subpopulation shift with a similar expansion condition as~\cite{wei2020theoretical} while also allowing disconnected parts within each class, but require access to ground-truth labels during training. In contrast, our algorithm doesn't need labels during pre-training. 

Co-training and multi-view learning are related settings which leverage two distinct “views” (i.e., feature subsets) of the
data~\citep{blum1998combining,dasgupta2002pac,balcan2005co}.  The original co-training algorithms~\citep{blum1998combining,dasgupta2002pac} assume that the two views are independent conditioned on the true label and leverage this
independence to obtain accurate pseudolabels for the unlabeled data.  ~\citet{balcan2005co} relax the requirement on independent views of co-training,
by using an “expansion” assumption, which is closely related to our assumption that $\rho_{\lfloor k/2\rfloor}$ is not too small in Theorem~\ref{thm:combine_with_cheeger_simplified}. Besides recent works  (e.g., the work of~\citet{tosh2021contrastive}), most co-training or multi-view learning algorithms are quite different from the modern contrastive learning algorithms which use neural network parameterization for vision applications. 


Our analysis relies on the normalized adjacency matrix (see Section~\ref{section:augmentation_graph}), which is closely related to the  graph Laplacian regularization that has been studied in the setting of semi-supervised learning~\cite{zhu2003semi, nadler2009semi}. In their works, the Laplacian matrix is used to define a regularization term that smooths the predictions on unlabeled data. This regularizer is further added to the supervised loss on labeled data during training. In contrast, we use the normalized adjacency matrix to define the unsupervised training objective in this paper.

\section{Spectral contrastive learning on population data }\label{section:framework}
\newcommand{\ones}{\textbf{1}}
\newcommand{\ind}[1]{1[#1]}
\newcommand{\hatf}{\hat{f}}
\newcommand{\hatF}{\widehat{F}}
\newcommand{\hatB}{\widehat{B}}
\newcommand{\diag}{\textup{diag}}
\newcommand{\eigF}{F^*}
\newcommand{\norA}{\overline{A}}
\newcommand{\eval}{\mathcal{E}}
In this section, we introduce our theoretical framework, the spectral contrastive loss, and the main analysis of the performance of the representations learned on population data.

We use $\ndata$ to denote the set of all natural data (raw inputs without augmentation). We assume that each $\bar{x}\in \ndata$ belongs to one of $r$ classes, and let $y: \ndata\rightarrow [r]$ denote the ground-truth (deterministic) labeling function.
Let $\pndata$ be the population distribution over $\ndata$ from which we draw training data and test our final performance. 
In the main body of the paper, for the ease of exposition, we assume $\ndata$ to be a finite but exponentially large set (e.g., all real vectors in $\R^d$ with bounded precision). This allows us to use sums instead of integrals and avoid non-essential nuances/jargons related to functional analysis. See Section~\ref{sec:infinite_data} for the straightforward extensions to the case where $\ndata$ is an infinite compact set (with mild regularity conditions).\footnote{In Section~\ref{sec:infinite_data}, we will deal with an infinite graph, its adjacency operator (instead of adjacency matrix), and the eigenfunctions of the adjacency operator (instead of eigenvectors) essentially in the same way.}

We next formulate data augmentations. 
Given a natural data sample $\bar{x}\in \ndata$, we use $\aug{\bar{x}}$ to denote the distribution of its augmentations. For instance, when $\bar{x}$ represents an image, $\aug{\bar{x}}$ can be the distribution of common augmentations~\cite{chen2020simple} that includes Gaussian blur, color distortion and random cropping. 
We use $\adata$ to denote the set of all augmented data, which is the union of supports of all $\aug{\bar{x}}$ for $\bar{x}\in\ndata$. As with $\ndata$, we also assume that $\adata$ is a finite but exponentially large set, and denote $N = |\adata|$.  None of the bounds will depend on $N$ --- it is only defined and assumed to be finite for the ease of exposition. 

We will learn an embedding function $f:\adata \rightarrow \R^k$, and then evaluate its quality by the minimum error achieved with a linear probe. Concretely, a linear classifier has weights $B\in \R^{k\times r}$ and predicts $\pred_{f,B}(x) = \argmax_{i\in [r]} (f(x)^\top B)_i$ for an augmented datapoint $x$ ($\arg\max$ breaks tie arbitrarily).
Then, given a natural data sample $\bar{x}$, we ensemble the predictions on augmented data and predict: 
\begin{align*}
\npred_{f,B}(\bar{x}) := \argmax_{i\in[r]} \Pr_{x \sim \aug{\bar{x}}} \left[\pred_{f, B}(x)=i\right].
\end{align*}
We denote the error of the representation and the linear head as:
\begin{align*}
\eval(f, B) := \Pr_{\bar{x}\sim \pndata}[y(\bar{x}) \neq \npred_{f, B}(\bar{x})].
\end{align*}
Define the \textit{linear probe} error as the error of the best possible linear classifier on the representations:
\begin{align}
\eval(f) := \min_{B\in \R^{k\times r}} \eval(f, B) = \min_{B\in \R^{k\times r}}\Pr_{\bar{x}\sim \pndata}[y(\bar{x}) \neq \npred_{f, B}(\bar{x})].
\end{align}

\subsection{Augmentation graph and spectral decomposition}\label{section:augmentation_graph}
Our approach is based on the central concept of \textbf{population augmentation graph}, denoted by $G(\adata, w)$, where the vertex set is all augmentation data $\adata$ and $w$ denotes the edge weights defined below. For any two augmented data $x, x'\in \adata$, define the weight $\wpair{x}{x'}$ as the marginal probability of generating the pair $x$ and $x'$ from a random natural data $\bar{x}\sim \pndata$:
\begin{align}
\wpair{x}{x'} := \Exp{\bar{x}\sim \pndata}\left[\augp{x}{\bar{x}} \augp{x'}{\bar{x}}\right]
\end{align}

Therefore, the weights sum to 1 because the total probability mass is 1: $\sum_{x, x'\in\adata}\wpair{x}{x'}=1$. 
The relative magnitude intuitively captures the closeness between $x$ and $x'$ with respect to the augmentation transformation.  
For most of the unrelated $x$ and $x'$, the value $\wpair{x}{x'}$ will be significantly smaller than the average value.  For example, when $x$ and $x'$ are random croppings of a cat and a dog respectively, $\wpair{x}{x'}$ will be essentially zero because no natural data can be augmented into both $x$ and $x'$. 
On the other hand, when $x$ and $x'$ are very close in $\ell_2$-distance or very close in $\ell_2$-distance up to color distortion, $\wpair{x}{x'}$ is nonzero because they may be augmentations of the same image with Gaussian blur and color distortion. We say that $x$ and $x'$ are connected with an edge if $\wpair{x}{x'}>0$.  See Figure~\ref{figure:augmentation_graph} (left) for more illustrations.

We emphasize that we \textit{only} work with the population graph rather than the empirical graph (i.e., the corresponding graph constructed with the empirical dataset as the vertex set).  
The population graph is very sparse but not empty---many similar images exist in the population. In contrast, the empirical graph would be nearly empty, since two images in the empirical dataset almost never share the same augmentation image. 
Our analysis will apply to minimizing contrastive loss on an empirical dataset (see Section~\ref{section:finite_sample}), but \textit{not} via analyzing the property of the empirical graph. 
Instead, we will show that contrastive learning on \textit{empirical} data with parametrized models is similar to decomposing the \textit{population} graph (see technical discussions in Section~\ref{section:proof}).  
This is a key difference between our work and classical spectral clustering work---we only require properties of the population graph rather than the empirical graph. 

\noindent{\bf A simplified running example} with Gaussian perturbation augmentation. 
Suppose the natural data is supported on manifolds in Euclidean space, and the data augmentation is adding random noise sampled from $\mathcal{N}(0, \sigma^2 \cdot I_{d\times d})$ where $\sigma$ is a small quantity (e.g., the norm of the perturbation $\sigma\sqrt{d}$ should be much smaller than the norm of the original datapoint). Then the edge between two augmented datapoints would be have near zero weight unless the two datapoints have small $\ell_2$ distance. Hence, the resulting graph is essentially the $\epsilon$-ball proximity graph~\cite{zemel2004proximity} or geometric graph~\citep{penrose2003random} in Euclidean space. 

Given the structure of the population augmentation graph, we apply spectral decomposition to the population graph to construct principled embeddings. The eigenvalue problems are closely related to graph partitioning as shown in spectral graph theory~\cite{chung1997spectral} for both worst-case graphs~\cite{d63036efc9d24f07b8908864667e28aa, kannan2004clusterings,louis2011algorithmic,lee2014multiway} and random graphs~\cite{mcsherry2001spectral,lei2015consistency,abbe2017community}. In machine learning, spectral clustering~\cite{ng2001spectral, shi2000normalized} is a classical algorithm that learns embeddings by eigendecomposition on an empirical distance graph and invoking $k$-means on the embeddings. 

We will apply eigendecomposition to the \textit{population} augmentation graph (and then later use linear probe for classification). Let $\wnode{x} = \sum_{x'\in \adata} \wpair{x}{x'}$ be the total weights associated to $x$, which is often viewed as an analog of the degree of $x$ in weighted graph.  A central object in spectral graph theory is the so-called \textit{normalized adjacency matrix}: 
\begin{align}
\norA := {D}^{-1/2}{A}{D}^{-1/2}
\end{align} 
where ${A}\in\Real^{\sizead\times \sizead}$ is adjacency matrix with entires $A_{xx'} = \wpair{x}{x'}$ and ${D}\in \Real^{\sizead\times \sizead}$ is a diagonal matrix with ${D}_{xx} = \wnode{x}$.\footnote{We index the matrix $A$, $D$ by $(x,x')\in \adata\times \adata$. Generally we index $N$-dimensional axis by $x\in \adata$.} 

Standard spectral graph theory approaches produce vertex embeddings as follows. 
Let $\eigvadj_1, \eigvadj_2, \cdots, \eigvadj_k$ be the $k$ largest eigenvalues of $\norA$, and ${\eigv}_1, {\eigv}_2, \cdots, {\eigv}_k$ be the corresponding unit-norm eigenvectors. Let $\eigF=[{\eigv}_1, {\eigv}_2, \cdots, {\eigv}_k]\in\Real^{\sizead\times k}$ be the matrix that collects these eigenvectors in columns, and we refer to it as the eigenvector matrix. Let $u_x^*\in \R^k$ be the $x$-th row of the matrix $\eigF$. 
It turns out that $u_x^*$'s can serve as desirable embeddings of $x$'s because they exhibit clustering structure in Euclidean space that resembles the clustering structure of the graph $G(\adata, w)$.

\subsection{From spectral decomposition to spectral contrastive learning}\label{section:spectral_contrastive_loss}

The embeddings $u_x^*$ obtained by eigendecomposition are nonparametric---a $k$-dimensional parameter is needed for every $x$---and therefore cannot be learned with a realistic amount of data. The embedding matrix $\eigF$ cannot be even stored efficiently. Therefore, we will instead parameterize the rows of the eigenvector matrix $\eigF$ as a neural net function, and assume embeddings $u_x^*$ can be represented by $f(x)$ for some $f\in\hyp$, where $\hyp$ is the hypothesis class containing neural networks.
As we'll show in Section~\ref{section:finite_sample}, this allows us to leverage the extrapolation power of neural networks and learn the representation on a finite dataset.

Next, we design a proper loss function for the feature extractor $f$, such that minimizing this loss could recover $\eigF$ up to some linear transformation. 
As we will show in Section~\ref{section:finite_sample}, the resulting population loss function on $f$ also admits an unbiased estimator with finite training samples. 
Let $F$ be an embedding matrix with $u_x$ on the $x$-th row, we will first design a loss function of $F$ that can be decomposed into parts about individual rows of $F$. 

We employ the following matrix factorization based formulation for eigenvectors. Consider the objective
\begin{align}\label{eq:matrix_fac}
\min_{F\in \R^{N\times k}} \Lossmc{{F}} := \norm{\norA- {F}{F}^\top}_F^2.
\end{align} 

By the classical theory on low-rank approximation (Eckart–Young–Mirsky theorem~\cite{eckart1936approximation}), any minimizer $\hatF$ of $\Lossmc{F}$ contains scaling of the largest eigenvectors of $\norA$ up to a right transformation---for some orthonormal matrix $R\in \Real^{k\times k}$, we have 
$
\hatF = \eigF\cdot \diag([\sqrt{\eigvadj_1},\dots, \sqrt{\eigvadj_k}])R
$. 
Fortunately, multiplying the embedding matrix by any matrix on the right and any diagonal matrix on the left does not change its linear probe performance, which is formalized by the following lemma.

\begin{lemma}\label{lemma:prediction_same_with_matrix}
	Consider an embedding matrix $F\in \R^{N\times k}$ and a linear classifier $B\in \Real^{k\times r}$. Let $D\in \R^{N\times N}$ be a diagonal matrix with positive diagonal entries and $Q\in \R^{k\times k}$ be an invertible matrix.
	Then, for any embedding matrix $\widetilde{F} = D\cdot F\cdot Q$, the linear classifier $\tilde{B}=Q^{-1}B$ on $\widetilde{F}$ has the same prediction as $B$ on $F$. 
	As a consequence, 
	we have 
	\begin{align}
	\eval(F) = \eval(\widetilde{F}).
	\end{align}
	where $\eval{(F)}$ denotes the linear probe performance when the rows of $F$ are used as embeddings. 
\end{lemma}

\begin{proof}[Proof of Lemma~\ref{lemma:prediction_same_with_matrix}]
	Let $D=\diag(s)$ where $s_x>0$ for $x\in\adata$. Let $u_x, \tilde{u}_{x}\in \R^k$ be the $x$-th row of matrices $F$ and $\widetilde{F}$, respectively.
	Recall that $\pred_{u,B}(x) = \argmax_{i\in [r]} (u_x^\top B)_i$ is the prediction on an augmented datapoint $x\in\ndata$ with representation $u_x$ and linear classifier $B$. Let $\widetilde{B}=Q^{-1}B$, it's easy to see that $\pred_{\tilde{u},\widetilde{B}}(x) = \argmax_{i\in [r]} (s_x\cdot u_x^\top B)_i$. Notice that $s_x>0$ doesn't change the prediction since it changes all dimensions of $u_x^\top B$ by the same scale, we have $\pred_{\tilde{u},\widetilde{B}}(x)=\pred_{u, B}(x)$ for any augmented datapoint $x\in\adata$. The equivalence of loss naturally follows.
\end{proof}

The main benefit of objective $\Lossmc{F}$ is that it's based on the rows of $F$. Recall that vectors $u_x$ are the rows of $F$. Each entry of $FF^\top$ is of the form $u_x^\top u_{x'}$, and thus $\Lossmc{F}$ can be decomposed into a sum of $N^2$ terms involving terms $u_x^\top u_{x'}$. Interestingly, if we reparameterize each row $u_x$ by $w_x^{1/2}f(x)$, we obtain a very similar loss function for $f$ that resembles the contrastive learning loss used in practice~\citep{chen2020simple} as shown below in Lemma~\ref{lem:spectral-contrastive-loss}. See Figure~\ref{figure:augmentation_graph} (right) for an illustration of the relationship between the eigenvector matrix and the representations learned by minimizing this loss.

We formally define the positive and negative pairs to introduce the loss.  Let $\bar{x}\sim \pndata$ be a random natural datapoint and draw $x\sim\aug{\bar{x}}$ and $ x^+\sim\aug{\bar{x}}$ independently to form a positive pair $(x,x^+)$. 
Draw $\bar{x}' \sim \pndata$ and $x^-\sim \aug{\bar{x}'}$ independently with $\bar{x}, x, x^+$.  We call $(x,x^-)$ a negative pair.\footnote{Though $x$ and $x^-$ are simply two independent draws, we call them negative pairs following the literature~\cite{arora2019theoretical}.}  


\begin{lemma}[Spectral contrastive loss]\label{lem:spectral-contrastive-loss}
	Recall that $u_x$ is the $x$-th row of $F$. Let $u_x = w_x^{1/2} f(x)$ for some function $f$. 
	Then, the loss function $\Lossmc{F}$ is equivalent to the following loss function for $f$, called spectral contrastive loss,  up to an additive constant: 
	\begin{align}\label{equation:spectral_loss}
	\Lossmc{F} & = 	    \Loss{f} + \textup{const} \nonumber\\
	\textup{where }  &~\Loss{f} \triangleq -2\cdot \Exp{x, x^+} \big[f(x)^\top f(x^+) \big]
	+ \Exp{x, x^-}\left[\left(f(x)^\top f(x^-) \right)^2\right]
	\end{align}
\end{lemma}

\begin{proof}[Proof of Lemma~\ref{lem:spectral-contrastive-loss}]
	We can expand $\Lossmc{F}$ and obtain
	\begin{align}\label{equation:loss_ma_derive}
	\Lossmc{{F}} &= \sum_{x, x'\in \adata} \left(\frac{\wpair{x}{x'}}{\sqrt{\wnode{x} \wnode{x'}}} - u_x^\top u_{x'}\right)^2\nonumber\\
	&= \sum_{x, x'\in \adata} \left(\frac{\wpair{x}{x'}^2}{\wnode{x} \wnode{x'}} -2\cdot \wpair{x}{x'} \cdot f(x)^\top f(x')  + w_x w_{x'}\cdot\left( f(x)^\top f(x')\right)^2\right)
	\end{align}
	
	Notice that the first term is a constant that only depends on the graph but not the variable $f$.  By the definition of augmentation graph, $\wpair{x}{x'}$ is the probability of a random positive pair being $(x, x')$ while $\wnode{x}$ is the probability of a random augmented datapoint being $x$. We can hence rewrite the sum of last two terms in Equation~\eqref{equation:loss_ma_derive} as Equation~\eqref{equation:spectral_loss}.
\end{proof}

We note that spectral contrastive loss is similar to many popular contrastive losses~\cite{oord2018representation, chen2020simple, sohn2016improved, wu2018unsupervised}. 
For instance, the contrastive loss in SimCLR~\cite{chen2020simple} can be rewritten as (with simple algebraic manipulation)
\begin{align*}
-f(x)^\top f(x^+) + \log \left(\exp\left( f(x)^\top f(x^+) \right) + \sum_{i=1}^n \exp\left(f(x)^\top f(x_i)\right)\right)\,.
\end{align*}

Here $x$ and $x^+$ are a positive pair and $x_1,\cdots, x_n$ are augmentations of other data.
Spectral contrastive loss can be seen as removing $f(x)^\top f(x^+)$ from the second term, and replacing the log sum of exponential terms with the average of the squares of $f(x)^\top f(x_i)$. 
We will show in Section~\ref{section:experiments} that our loss has a similar empirical performance as SimCLR without requiring a large batch size.

\subsection{Theoretical guarantees for spectral contrastive loss on population data}\label{section:data_assumption}

In this section, we introduce the main assumptions on the data and state our main theoretical guarantee for spectral contrastive learning on population data. 

To formalize the idea that $G$ cannot be partitioned into too many disconnected sub-graphs, we introduce the notions of \textit{Dirichlet conductance} and \textit{sparsest $m$-partition}, which are standard in spectral graph theory. Dirichlet conductance represents the fraction of edges from $S$ to its complement:
\begin{definition}[Dirichlet conductance]\label{definition:dirichlet_conductance}
	For a graph $G=(\adata, \wnode{})$ and a subset $S\subseteq \adata$, we define the Dirichlet conductance of $S$ as
	\begin{align*}
	\phi_G(S) := \frac{\sum_{x\in S, x'\notin S} \wpair{x}{x'}}{\sum_{x\in S}\wnode{x}}.
	\end{align*}
\end{definition}
We note that when $S$ is a singleton, there is $\phi_G(S)=1$ due to the definition of $\wnode{x}$. For $i\in\mathbb{Z}^+$, we introduce the sparsest $i$-partition to represent the number of edges between $i$ disjoint subsets. 
\begin{definition}[Sparsest $i$-partition]\label{definition:multi_way_expansion_constant}
	Let $G=(\adata, \wnode{})$ be the augmentation graph. For an integer $i \in [2, |\adata|]$, we define the sparsest $i$-partition as
	\begin{align*}
	\rho_i := \min_{S_1, \cdots, S_i}\max\{\phi_G(S_1),\dots, \phi_G(S_i)\}
	\end{align*}
	where $S_1, \cdots, S_i$ are non-empty sets that form a partition of $\adata$.  
\end{definition}

We note that $\rho_i$ increases as $i$ increases.\footnote{To see this, consider $3\le i\le |\adata|$. Let $S_1, \cdots, S_{i}$ be the partition of $\adata$ that minimizes the RHS of Definition~\ref{definition:multi_way_expansion_constant}
Define set $S_{i-1}' :=S_i\cup S_{i-1}$. It is easy to see that $\phi_G(S_{i-1}') = \frac{\sum_{x\in S_{i-1}', x'\notin S_{i-1}'} \wpair{x}{x'}}{\sum_{x\in S_{i-1}'}\wnode{x}} \le \frac{\sum_{j=i-1}^{i}\sum_{x\in S_j, x'\notin S_j} \wpair{x}{x'}}{\sum_{j=i-1}^{i}\sum_{x\in S_j}\wnode{x}} \le \max\{\phi_G(S_{i-1}), \phi_G(S_{i})\}$. Notice that $S_1, \cdots, S_{i-2}, S_{i-1}'$ are $i-1$ non-empty sets that form a partition of $\adata$, by Definition~\ref{definition:multi_way_expansion_constant} we have $\rho_{i-1}\le \max\{\phi_G(S_1), \cdots, \phi_G(S_{i-2}), \phi_G(S_{i-1}')\}\le \max\{\phi_G(S_{1}), \cdots, \phi_G(S_{i})\}=\rho_{i}$. }
When $r$ is the number of underlying classes, we might expect $\rho_r\approx 0$ since the augmentations from different classes almost compose a disjoint $r$-way partition of $\adata$. 
However, for $i>r$, we can expect $\rho_i$ to be much larger. For instance, in the extreme case when $i=|\adata|=N$, every set $S_j$ is a singleton, which implies that $\rho_{N}=1$. More generally, as we will show later (Lemma~\ref{proposition:rho_lower_bound}), $\rho_i$ can be expected to be at least inverse polynomial in data dimension when $i$ is larger than the number of underlying semantic classes in the data. 

\begin{assumption}[at most $m$ clusters]\label{assumption:at_most_m_clusters}
	We assume that  $\rho_{m+1}\ge \rho$. A prototypical case would be that there are at most $m$ clusters in the population augmentation graph, and each of them cannot be broken into two subsets both with conductance less than $\rho$. 
\end{assumption}

When there are $m$ clusters that have sufficient inner connections (corresponding to, e.g., $m$ semantically coherent subpopulations), we expect $\rho_{m+1}$ to be much larger than $\rho_m$ because any $m+1$ partition needs to break one sub-graph into two pieces and incur a large conductance. 
In other words, suppose the graph is consists of $m$ clusters, the quantity $\rho$ is characterizing the level of internal connection within each cluster.
Furthermore, in many cases we expect $\rho_{m+1}$ to be inverse polynomial in dimension. 
In the running example of Section~\ref{section:augmentation_graph} (where augmentation is adding Gaussian noise), $\rho$ is related to the Cheeger constant or the isoperimetric number of the data manifolds, which in many cases is believed to be at least inverse polynomial in dimension (e.g., see~\citet{bobkov1997isoperimetric} for the Cheeger constant of the Gaussian distribution.) Indeed, in Section~\ref{section:gaussian_example} we will formally lowerbound $\rho_{m+1}$ by the product of the augmentation strength and the Cheeger constant of the subpopulation distributions (Proposition~\ref{proposition:rho_lower_bound}), and lowerbound the Cheeger constant by inverse polynomial for concrete settings where the data come from a mixture of manifolds (Theorem~\ref{theorem:gaussian_example}). 




Assumption~\ref{assumption:at_most_m_clusters} also implies properties of the graph spectrum. Recall that $\gamma_i$ is the $i$-th largest eigenvalue of the normalized adjacency matrix $\norA$ and $\gamma_1 = 1$. 
According to Cheeger's inequality (Lemma~\ref{lemma:higher_order_cheeger}), Assumption~\ref{assumption:at_most_m_clusters} implies that $\gamma_{2m}\le 1-\Omega(\rho^2/\log{m})$, which suggests that there is a gap between $\gamma_1$ and $\gamma_{2m}$ and will be useful in our analysis. 

Next, we formalize the assumption that very few edges cross different ground-truth classes. It turns out that it suffices to assume that the labels are recoverable from the augmentations (which is also equivalent to that two examples in different classes can rarely be augmented into the same point).
\begin{assumption}[Labels are recoverable from augmentations]\label{definition:accurate_partition} 
	Let $\bar{x}\sim \pndata$ and $y(\bar{x})$ be its label. Let the augmentation $x\sim \aug{\bar{x}}$. We assume that there exists a classifier $g$ that can predict $y(\bar{x})$ given $x$ with error at most $\alpha$. That is, $g(x)=y(\bar{x})$ with probability at least $1-\alpha$.
\end{assumption}

A small $\alpha$ in Assumption~\ref{definition:accurate_partition} means that different classes are ``separated'' in the sense that data from different classes have very few (at most $O(\alpha)$) shared augmentations. Alternatively, one can think of this assumption as assuming that the augmentation graph can be partitioned into $r$ clusters each corresponding to augmentations from one class, and there are at most $O(\alpha)$ edges across the clusters. This is typically true for real-world image data like ImageNet, since for any two images from different classes (e.g., images of a Husky and a Birman cat), using the typical data augmentations such as adding noise and random cropping can rarely (with exponentially small probability) lead to the same augmented image. 

Typically, both $\rho$ in Assumption~\ref{assumption:at_most_m_clusters} and $\alpha$ in Assumption~\ref{assumption:realizable} are small positive values that are much less than 1. 
However, $\rho$ can be much larger than $\alpha$. 
Recall that $\rho$ can be expected to be at least inverse polynomial in dimension.
In contrast, $\alpha$ characterizes the separation between classes and are expected to be exponentially small in typical cases. For example, in the running example of Section~\ref{section:augmentation_graph} with Gaussian perturbation augmentation, if $\sigma\sqrt{d}$ is smaller than the minimum distance between two subpopulations, we can rarely augment two datapoints from distinct subpopulations into a shared augmentation, and therefore $\alpha$ is expected to exponentially small. 
Our analysis below operates in the reasonable regime where $\rho^2$ is larger than $\alpha$, which intuitively means that the internal connection within the cluster is bigger than the separation between the clusters. 


We also introduce the following assumption which states that some minimizer of the population spectral contrastive loss can be realized by the hypothesis class.

\begin{assumption}[Expressivity of the hypothesis class]\label{assumption:realizable} 
	Let $\hyp$ be a hypothesis class containing functions from $\adata$ to $\Real^k$. We assume that at least one of the global minima of $\Loss{f}$ belongs to $\hyp$. 
\end{assumption}

Our main theorem bound from above the linear probe error of the feature learned by minimizing the \textit{population} spectral contrastive loss. In Theorem~\ref{theorem:end_to_end} we extend this result to the case where both the feature and the linear head are learned from empirical datasets.
\begin{theorem}[Main theorem for infinite/population pretraining data case]\label{thm:combine_with_cheeger_simplified} 
	Assume the representation dimension $k\ge 2r$ and Assumption~\ref{definition:accurate_partition} holds for $\alpha>0$. Let $\hyp$ be a hypothesis class that satisfies Assumption~\ref{assumption:realizable} and let $\globalminf\in\hyp$ be a minimizer of $\Loss{f}$.
	Then, we have
	\begin{align*}
	\eval(\globalminf)
	\le \widetilde{O}\left(\alpha/\rho^2_{\lfloor k/2\rfloor}\right).
	\end{align*}
    In particular, if Assumption~\ref{assumption:at_most_m_clusters} also holds and $k > 2m$, we have $\eval(\globalminf)
    \le \widetilde{O}\left(\alpha/\rho^2\right)$.
\end{theorem}
Here we use $\widetilde{O}(\cdot)$ to hide universal constant factors and logarithmic factors in $k$.  We note that $\alpha=0$ when augmentations from different classes are perfectly disconnected in the augmentation graph, in which case the above theorem guarantees the exact recovery of the ground truth. Generally, we expect $\alpha$ to be an extremely (exponentially) small constant independent of $k$, whereas $\rho_{\lfloor k/2\rfloor}$ increases with $k$ and can be at least inverse polynomial when $k$ is reasonably large, hence much larger than $\sqrt{\alpha}$. 
We characterize the $\rho_k$'s growth on more concrete distributions in the next subsection. When $k > 2m$, as argued below Assumption~\ref{definition:accurate_partition}, we expect that $\alpha \ll \rho^2 \le \rho^2_{m+1}$ and thus the error $\alpha/\rho^2$ is sufficiently small. 

Previous works on graph partitioning~\cite{lee2014multiway, arora2009expander, leighton1999multicommodity} often analyze the rounding algorithms that conduct clustering based on the representations of unlabeled data and do not analyze the performance of linear probe (which has access to labeled data). These results provide guarantees on the approximation ratio---the ratio between the conductance of the obtained partition to the best partition---which may depend on graph size~\cite{arora2009expander} that can be exponentially large in our setting. The approximation ratio guarantee does not lead to a guarantee on the representations' performance on downstream tasks. Our guarantees are on the linear probe accuracy on the downstream tasks and independent of the graph size. We rely on the formulation of the downstream task's labeling function (Assumption~\ref{definition:accurate_partition}) as well as a novel analysis technique that characterizes the linear structure of the representations. In Section~\ref{section:proof_for_main_result}, we provide the proof of  Theorem~\ref{thm:combine_with_cheeger_simplified} as well as its more generalized version where $k/2$ is relaxed to be any constant fraction of $k$. 
A proof sketch of Theorem~\ref{thm:combine_with_cheeger_simplified} is given in Section~\ref{section:proof_sketch}.

\subsection{Provable instantiation of Theorem~\ref{thm:combine_with_cheeger_simplified} to mixture of manifold data}\label{section:gaussian_example}

In this section, we exemplify Theorem~\ref{thm:combine_with_cheeger_simplified} 
on examples where the natural data distribution is a mixture of manifolds. 

\newcommand{\cheeger}{\tau}
We first show that in the running example given in Section~\ref{section:augmentation_graph}, Assumption~\ref{assumption:at_most_m_clusters} holds for some $\rho$ that is closely related to the Cheeger constant of the data manifolds. Recall that the Cheeger constant or isoperimetric number~\cite{buser1982note} of a distribution $\mu$ with density $p$ over $\R^d$ is defined as  
\begin{align}
	h_\mu := \inf_{S\subset \R^d} \frac{\int_{\partial S}p(x) dx}{\min\{\int_S p(x)dx, \int_{\R^d \backslash S} p(x)dx \}}.
\end{align}
Here the denominator is the smaller one of volumes of $S$ and $\Real^d\backslash S$, whereas the numerator is the surface area of $S$. (See e.g.,\cite{chen2021almost} for the precise definition of the boundary measure $\partial S$.) 
The following proposition (proved in Section~\ref{section:proof_rho_lower_bound}) shows that $\rho$ scales linearly in the augmentation strength and  the Cheeger constant. 
\begin{proposition}\label{proposition:rho_lower_bound}
	Suppose the natural data distribution $\pndata$ is a mixture of $m$ distributions $P_1, \cdots, P_m$ supported on disjoint subsets of $\R^d$, and the data augmentation is Gaussian perturbation sampled from $\mathcal{N}(0, \sigma^2 \cdot I_{d\times d})$. Then, 
	\begin{align}
		\lim_{\sigma \rightarrow 0^+}\frac{\rho_{m+1}}{\sigma} \gtrsim  \min_{i\in[m]} h_{P_i}
	\end{align}
	That is, $\rho_{m+1}$ is at least linear in the augmentation size $\sigma$ and the Cheeger constants of subpopulations. 
\end{proposition}

In many cases, the Cheeger constant is at least inverse polynomial in the data dimension~\cite{chen2021almost, lee2016eldan}. 
When the manifolds $P_i$ are spherical Gaussian with unit identity covariance, the Cheeger constant is $\Omega(1)$\cite{bobkov1997isoperimetric},  and thus the distribution $\pndata$ in Proposition~\ref{proposition:rho_lower_bound} satisfies Assumption~\ref{assumption:at_most_m_clusters} with $\rho\gtrsim \sigma$.
Furthermore, when the distribution is transformed by a function with Lipschitzness $\kappa>0$, the Cheeger constant changes by a factor at most $\kappa$. Therefore, Proposition~\ref{proposition:rho_lower_bound} also applies to a mixture of manifolds setting defined below. 

 In the rest of this section, we instantiate Theorem~\ref{thm:combine_with_cheeger_simplified} on a mixture of manifolds example where the data is generated from a Lipschitz transformation of a mixture of Gaussian distributions, and give an error bound for the downstream classification task.
\begin{example}[Mixture of manifolds]\label{example:gaussian}
	Suppose $\pndata$ is mixture of $r\le d$ distributions $P_1,\cdots, P_r$, where each $P_i$ is generated by some $\kappa$-bi-Lipschitz\footnote{A $\kappa$ bi-Lipschitz function satisfies $\frac{1}{\kappa}\norm{f(x)-f(y)}_2\le \norm{x-y}_2 \le \kappa\norm{f(x)-f(y)}_2$.} generator $Q:\Real^{d'}\rightarrow\Real^d$ on some latent variable $z\in\Real^{d'}$ with $d'\le d$ which as a mixture of Gaussian distribution:
	\begin{align*}
	x\sim P_i \iff x = Q(z), z\sim \mathcal{N}(\mu_i, \frac{1}{d'}\cdot I_{d'\times d'}).
	\end{align*}
	Let the data augmentation of a natural data sample $\bar{x}$ be $\bar{x}+\xi$ where $\xi\sim\mathcal{N}(0, \frac{\sigma^2}{d} \cdot I_{d\times d})$ is isotropic Gaussian noise with $0<\sigma\lesssim \frac{1}{\sqrt{d}}$. We also assume $\min_{i\ne j}\norm{\mu_i-\mu_j}_2\gtrsim \frac{\kappa\cdot\sqrt{\log d}}{\sqrt{d'}}$. 
	
	Let $\bar{y}(x)$ be the most likely mixture index $i$ that generates $x$: $\bar{y}(x) := \arg\max_{i}P_i(x)$. The simplest downstream task can have label $y(x) = \bar{y}(x)$. More generally, 
	let $r'\le r$ be the number of labels, and the label $y(x)\in [r']$ in the downstream task be  equal to $\pi(\bar{y}(x))$ where $\pi$ is a function that maps $[r]$ to $[r']$.
\end{example}

We note that the intra-class distance in the latent space is on the scale of $\Omega(1)$, which can be much larger than the distance between class means which is assumed to be $\gtrsim \frac{\kappa\cdot\sqrt{\log d}}{\sqrt{d'}}$. Therefore, distance-based clustering algorithms do not apply. Moreover, in the simple downstream tasks, the label for $x$ could be just the index of the mixture where $x$ comes from. We also allow downstream tasks that merge the $r$ components into $r'$ labels as long as each mixture component gets the same label. 
We apply Theorem~\ref{thm:combine_with_cheeger_simplified} and get the following theorem:
\begin{theorem}[Theorem for the mixture of manifolds example]\label{theorem:gaussian_example}
	When $k\ge2r+2$, Example~\ref{example:gaussian} satisfies Assumption~\ref{definition:accurate_partition} with $\alpha\le\frac{1}{\textup{poly}(d)}$, and has $\rho_{\lfloor k/2\rfloor} \gtrsim \frac{\sigma}{\kappa\sqrt{d}}$. As a consequence, the error bound is $\eval(\globalminf)\le\widetilde{O}\left(\frac{\kappa^2}{\sigma^2 \cdot \textup{poly}(d)}\right)$. 
\end{theorem}

The theorem above guarantees small error even when $\sigma$ is polynomially small. In this case, the augmentation noise has a much smaller scale than the data (which is at least on the order of $1/\kappa$). This suggests that contrastive learning can non-trivially leverage the structure of the underlying data and learn good representations with relatively weak augmentation. To the best of our knowledge, it is difficult to apply the theorems in previous works~\citep{arora2019theoretical,lee2020predicting,tosh2020contrastive,tosh2021contrastive,wei2020theoretical} to this example and get similar guarantees with polynomial dependencies on $d,\sigma,\kappa$. The work of~\citet{wei2020theoretical} can apply to the setting where $r$ is known and the downstream label is equal to $\bar{y}(x)$, but cannot handle the case when $r$ is unknown or when two mixture component can have the same label. 
We refer the reader to the related work section for more discussions and comparisons.
The proof can be found in Section~\ref{section:proof_Gaussian_thm}.

\section{Finite-sample generalization bounds}\label{section:finite_sample}

\subsection{Unlabeled sample complexity for pretraining}
In Section~\ref{section:framework}, we provide guarantees for spectral contrastive learning on population data. In this section, we show that these guarantees can be naturally extended to the finite-sample regime with standard concentration bounds.
In particular, given a unlabeled pretraining dataset $\{\bar{x}_1, \bar{x}_2, \cdots, \bar{x}_\npt\}$ with $\bar{x}_i\sim\pndata$, we learn a feature extractor by minimizing the following \textit{empirical spectral contrastive loss}:
\begin{align*}
\eLoss{\npt}{f} := -\frac{2}{n}\sum_{i=1}^{\npt}\Exp{ {x}\sim\aug{\bar{x}_i}\atop {x}^+\sim\aug{\bar{x}_i}}\left[f({x})^\top f(x^+)\right] + \frac{1}{\npt(\npt-1)}\sum_{i\ne j}\Exp{{x}\sim \aug{\bar{x}_i}\atop  {x}^-\sim\aug{\bar{x}_j}}\left[\left(f({x})^\top f({x}^-)\right)^2\right].
\end{align*}

It is worth noting that $\eLoss{\npt}{f}$ is an unbiased estimator of the population spectral contrastive loss $\Loss{f}$. (See Claim~\ref{claim:eloss_equal_loss} for a proof.)
Therefore, we can derive generalization bounds via off-the-shelf concentration inequalities. Let $\hyp$ be a hypothesis class containing feature extractors from $\adata$ to $\Real^k$. We extend Rademacher complexity to function classes with high-dimensional outputs and define the Rademacher complexity of $\hyp$ on $n$ data as 
$
\rad{n}(\hyp) := \max_{x_1, \cdots, x_n \in \adata} \Exp{\sigma}\left[\sup_{f\in\hyp, i\in[k]} \frac{1}{n} \left(\sum_{j=1}^n \sigma_jf_i(x_j)\right)\right],
$
where $\sigma$ is a uniform random vector in $\{-1, 1\}^n$ and $f_i(z)$ is the $i$-th dimension of $f(z)$. 

Recall that $\globalminf\in\hyp$ is a minimizer of $\Loss{f}$. The following theorem with proofs in Section~\ref{section:proof_for_nn_generalization} bounds the population loss of a feature extractor trained with finite data:
\begin{theorem}[Excess contrasitve loss]\label{theorem:nn_generalization}
	For some $\boundf>0$, assume $\norm{f(x)}_\infty\le \boundf$ for all $f\in\hyp$ and $x\in\adata$.
	Let $\globalminf\in\hyp$ be a minimizer of the population loss $\Loss{f}$. Given a random dataset of size $\npt$, let $\empminf\in\hyp$ be a minimizer of empirical loss $\eLoss{\npt}{f}$. 
	Then, when Assumption~\ref{assumption:realizable} holds, with probability at least $1-\delta$ over the randomness of data, we have
	\begin{align*}
	\Loss{\empminf} \le \Loss{\globalminf} + c_1\cdot \rad{\npt/2}(\hyp) + c_2\cdot\left(\sqrt{\frac{\log 2/\delta}{\npt}} +{\delta}\right),
	\end{align*}
	where constants $c_1\lesssim k^2\boundf^2+k\boundf$ and $c_2 \lesssim k\boundf^2 + k^2\boundf^4$. 
\end{theorem}


The Rademacher complexity usually looks like $\rad{n}(\hyp) = \sqrt{{R}/{n}}$ where $R$ measures the complexity of $\hyp$ (hence only depends on $\hyp$). This suggests that when $\kappa$ is $O(1)$, the sample complexity for acheiving suboptimality $\epsilon$ on population loss is $O(k^4R/\epsilon^2)$.
We can apply Theorem~\ref{theorem:nn_generalization} to any hypothesis class $\hyp$ of interest (e.g., deep neural networks) and plug in off-the-shelf Rademacher complexity bounds. For instance, in Section~\ref{section:proof_of_nn_norm_contral} we give a corollary of Theorem~\ref{theorem:nn_generalization} when $\hyp$ contains deep neural networks with ReLU activation.

The theorem above shows that we can achieve near-optimal population loss by minimizing empirical loss up to some small excess loss. The following theorem characterizes how the error propagates to the linear probe performance mildly under some spectral gap conditions.

\begin{theorem}[Minimum downstream error]\label{theorem:linear_probe_suboptimal}
	In the setting of Theorem~\ref{theorem:nn_generalization}, suppose Assumption~\ref{assumption:at_most_m_clusters} holds for $\rho>0$, Assumption~\ref{definition:accurate_partition} holds for $\alpha>0$, Assumption~\ref{assumption:realizable} holds, and the representation dimension $k\ge \max\{4r+2, 2m\}$,.
	Then, with $1-\delta$ probability over the randomness of data, for any $\empminf\in\hyp$ that minimizes the empirical loss $\eLoss{\npt}{f}$, we have that 
	\begin{align*}
		\eval(\empminf) \lesssim \frac{\alpha}{\rho^2}\cdot\log k + \frac{ck}{\Delta_\eigvadj^2}  \left(\rad{\npt/2}(\hyp) + \sqrt{\frac{\log 2/\delta}{\npt}} +{\delta}\right),
	\end{align*}
	where $c\lesssim (k\kappa+k\kappa^2+1)^2$, and $\Delta_\eigvadj:=\eigvadj_{\lfloor{3k}/4\rfloor} - \eigvadj_{k}$ is the eigenvalue gap between the $\lfloor{3k}/4\rfloor$-th and the $k$-th eigenvalue.
\end{theorem}



This theorem shows that the error on the downstream task only grows linearly with the excess loss during pretraining. Roughly speaking, one can think of $\Delta_\eigvadj$ as on the order of $1-\gamma_k$, hence by Cheeger's inequality it's larger than $\rho^2$. When  $\rad{\npt/2}(\hyp) = \sqrt{{2R}/{\npt}}$ and $\kappa\le O(1)$, we have that the number of unlabeled samples required to achieve $\epsilon$ downstream error is $O(m^6R/\epsilon^2\rho^4)$. 
We can relax Assumption~\ref{assumption:realizable} to approximate realizability in the sense that $\hyp$ contains some sub-optimal feature extractor under the population spectral loss and pay an additional error term in the linear probe error bound.
The proof of Theorem~\ref{theorem:linear_probe_suboptimal} can be found in Section~\ref{section:roof_of_suboptimal_representation_bound}.




\subsection{Labeled sample complexity for linear probe}\label{section:linear_probe_finite_sample}

In this section, we provide sample complexity analysis for learning a linear probe with \textit{labeled} data. Theorem~\ref{thm:combine_with_cheeger_simplified} guarantees the existence of a linear probe that achieves a small downstream classification error. However, a priori it is unclear how large the margin of the linear classifier can be, so it is hard to apply margin theory to provide generalization bounds for 0-1 loss.
We could in principle control the margin of the linear head, but using capped quadratic loss turns out to suffice and mathematically more convenient.
We learn a linear head with the following \textit{capped quadratic loss}:
given a tuple $(z, y(\bar{x}))$ where $z\in\Real^k$ is a representation of augmented datapoint $x\sim\aug{\bar{x}}$ and $y(\bar{x})\in[r]$ is the label of $\bar{x}$, for a linear probe $B\in\Real^{k\times r}$ we define loss 
$
\ell((z, y(\bar{x})), B) := \sum_{i=1}^r \min\big\{\left(B^\top z - \vec{y}(\bar{x})\right)_i^2,1\big\},
$
where $\vec{y}(\bar{x})$ is the one-hot embedding of $y(\bar{x})$ as a $r$-dimensional vector ($1$ on the $y(\bar{x})$-th dimension, $0$ on other dimensions).
This is a standard modification of quadratic loss in statistical learning theory that ensures the boundedness of the loss for the ease of analysis~\cite{mohri2018foundations}.

The following Theorem~\ref{theorem:end_to_end} provides a generalization guarantee for the linear classifier that minimizes capped quadratic loss on a labeled downstream dataset of size $\nds$.
The key challenge of the proof is showing the existence of a small-norm linear head $B$ that gives small population quadratic loss, which is not obvious from Theorem~\ref{theorem:linear_probe_suboptimal} where only small 0-1 error is guaranteed. Given a labeled dataset $\{(\bar{x}_i, y(\bar{x}_i))\}_{i=1}^\nds$ where $\bar{x}_i\sim\pndata$ and $y(\bar{x}_i)$ is its label, we sample ${x}_i\sim\aug{\bar{x}_i}$ for $i\in[\nds]$.  Given a norm bound $C_k>0$, we learn a linear probe $\widehat{B}$ by minimizing the capped quadratic loss subject to a norm constraint:
\begin{align}\label{eqn:linear_head_finite_data}
\widehat{B} \in \argmin_{\norm{B}_F\le C_k}\sum_{i=1}^\nds \ell((\empminf(x_i), y(\bar{x}_i)), B).
\end{align}


\begin{theorem}[End-to-end error bounds with finite pretraining and downstream samples]\label{theorem:end_to_end}
	In the setting of Theorem~\ref{theorem:linear_probe_suboptimal}, choose $C_k>0$ such that $C_k\ge\frac{2(k+1)}{\gamma_{k}}$. 
	Then, with probability at least $1-\delta$ over the randomness of data, for any $\empminf\in\hyp$ that minimizes the empirical pre-training loss $\eLoss{\npt}{f}$ and a linear head $\widehat{B}$ learned from Equation~\eqref{eqn:linear_head_finite_data}, we have
	\begin{align*}
		\eval(\empminf, \widehat{B}) \lesssim  \frac{\alpha}{\rho_{\lfloor k/2\rfloor}^2}\cdot\log k + \frac{ck}{\Delta_\eigvadj^2}  \left(\rad{\npt/2}(\hyp) + \sqrt{\frac{\log 2/\delta}{\npt}} +{\delta}\right) + \left(rC_k\sqrt{\frac{k}{\nds} } + \sqrt{\frac{\log 1/\delta}{\nds}}\right).
	\end{align*}
\end{theorem}

Here the first term is an error caused by the property fo the population data, which is unavoidable even with infinite pretraining and downstream samples (but it can be small as argued in Section~\ref{section:data_assumption}). The second term is caused by finite pretraining samples, and the third term is caused by finite samples in the linear classification on the downstream task.

Typically, the Rademacher complexity is roughly $\rad{\npt/2}(\hyp) = \sqrt{{2R}/{\npt}}$ where $R$ is captures the complexity of the model architecture. 
Thus, to achieve final linear probe error no more than $O(\epsilon)$, we would need to select $k$ such that $\frac{\alpha}{\rho_{\lfloor k/2\rfloor}^2} \cdot \log k \le \epsilon$, and we need $\poly(k, \frac{1}{\Delta_\eigvadj}, R, \frac{1}{\epsilon})$ pretraining samples and $\poly(k, r, \frac{1}{\gamma_k}, \frac{1}{\epsilon})$ downstream samples.  

When $r<k$, 
the eigengap $\Delta_\eigvadj$ is on the order of $1-\gamma_{k}$ which is larger than $\rho^2$ by Cheeger inequality. Recall that $\rho$ is at least inverse polynomial in $d$ as argued in Section~\ref{section:gaussian_example}, one can expect $\frac{1}{\Delta_\eigvadj}$ to be at most $\poly(d)$. On the other hand, $\gamma_k\approx 1$ so $\frac{1}{\gamma_{k}}$ can be thought of as a constant. Thus, the final required number of pretraining samples is $\npt = \poly(k,d,R, \frac{1}{\epsilon})$ and number of downstream samples is $\nds = \poly(r, k, \frac{1}{\epsilon})$. We note that the downstream sample complexity doesn't depend on the complexity of the hypothesis class $R$, suggesting that pretraining helps reduce the sample complexity of the supervised downstream task. 

The proof of Theorem~\ref{theorem:end_to_end} is in Section~\ref{section:proof_linear_probe_finite_sample}.

\begin{figure*}
	\setlength{\lineskip}{0pt}
	\centering
	\begin{noverticalspace}
		\includegraphics[width=0.95\textwidth]{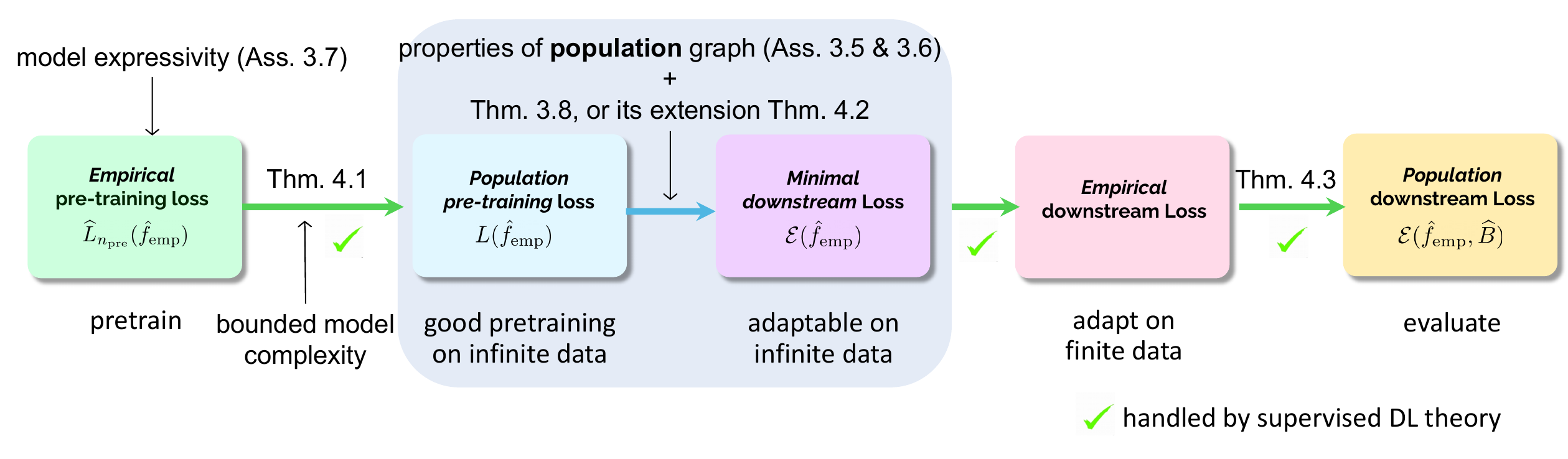}%
	\end{noverticalspace} 
	\caption[..]{A diagram of our analysis framework.  We decompose the problem into a core step that shows a small population pretraining loss implies a small minimal downstream loss (Theorem~\ref{thm:combine_with_cheeger_simplified}, or its extension Theorem~\ref{theorem:linear_probe_suboptimal}) and a few other somewhat standard steps that link empirical losses to population losses (Theorems~\ref{theorem:nn_generalization} and Theorem~\ref{theorem:end_to_end}). 
	}\label{figure:framework}
\end{figure*}
\section{Analysis Framework and Proof Sketch}
\label{section:proof}

As discussed before and suggested by the structured of Section~\ref{section:framework} and~\ref{section:finite_sample}, our analysis framework decompose the problem into a key step about the population cases (Section~\ref{section:framework}) and a few other somewhat standard steps that link empirical losses to population losses (Section~\ref{section:finite_sample}). As depicted in Figure~\ref{figure:framework}, the core step (Theorem~\ref{thm:combine_with_cheeger_simplified}, or its extension Theorem~\ref{theorem:linear_probe_suboptimal}) is to show that a small population pretraining loss implies the existence of a linear classifier for the downstream task, that is, a small minimal downstream loss.


We first remark that a feature of our analysis framework is that we link the population pretraining data case to the finite sample case by showing the empirical and population pretraining losses are similar when the feature extractors are a \textit{parameterized} family of models with capacity bounds (the first arrow in Figure~\ref{figure:framework}). 
Hypothetically, suppose such a connection between population and empirical data case was built through the relationship between the population and empirical graphs, e.g., by proving that the empirical graph has similar spectral properties as the population graph, then the sample complexity will be exponential. 
Intuitively, this is because the population graph is very sparse, and the empirical graph is with high probability empty if the number of samples is only polynomial in dimension (e.g. consider the case when the augmentation simply adds small perturbation, as in the running example in Section~\ref{section:augmentation_graph}). The empirical graph essentially follows the well-studied random geometric graph model~\citep{penrose2003random}, and tends to have no structure in high dimension~\cite{bubeck2016testing,liu2021testing,brennan2020phase}. 
The fundamental difference between this hypothetical and our framework is that the empirical graph's definition does not involve any parameterization, and thus the resemblance between the empirical and population graphs does not leverage the extrapolation (or inductive bias) of the model parameterization as our framework does for the pretraining losses. 



We note that the inductive bias of the parameterized model is indeed used in the analysis for finite-sample case. We assume that the model family $\mathcal{F}$ can express the eigenfunctions/eigenvectors of the graph  (Assumption~\ref{assumption:realizable}) and also implicitly assume bounds on its Rademacher complexity (in Theorem~\ref{theorem:end_to_end}). 

Once we obtained that the existence of a linear classifier, the remaining steps (the third and fourth arrows in Figure~\ref{figure:framework}) follow from standard supervised learning theory. 


In the rest of this section, we will give a proof sketch of the population case, which is the more challenging step. 

\subsection{Proof Sketch of Theorem~\ref{thm:combine_with_cheeger_simplified}}\label{section:proof_sketch}
\newcommand{\globalminF}{F_{\textup{pop}}^*}
In this section, we give a proof sketch of Theorem~\ref{thm:combine_with_cheeger_simplified} in a simplified binary classification setting where there are only two classes in the downstream task. 

Recall that $N$ is  the size of $\adata$. Recall that $w_x$ is the total weight associated with an augmented datapoint $x\in\adata$, which can also be thought of as the probability mass of $x$ as a randomly sampled augmented datapoint. In the scope of this section, for demonstrating the key idea, we also assume that $x$ has uniform distribution, i.e., $w_x=\frac{1}{N}$ for any $x\in\adata$. 

Let $g: \adata\rightarrow\{0, 1\}$ be the Bayes optimal classifier for predicting the label given an augmented datapoint. By Assumption~\ref{definition:accurate_partition}, $g$ has an error at most $\alpha$ (which is assumed to be small). Thus, we can think of it as the ``target'' classifier that we aim to recover. 
We will show that $g$ can be approximated by a linear function on top of the learned features. 
Recall that ${\eigv}_1, {\eigv}_2, \cdots, {\eigv}_k$ are the top-$k$ unit-norm eigenvectors of $\norA$ and the feature $u_x^*$ for $x$ is the $x$-th row of the eigenvector matrix $\eigF=[{\eigv}_1, {\eigv}_2, \cdots, {\eigv}_k]\in\Real^{\sizead\times k}$. As discussed in Section~\ref{section:spectral_contrastive_loss}, the spectral contrastive loss was designed to compute a variant of the eigenvector matrix $F^*$ up to row scaling and right rotation. More precisely, letting $\globalminF\in \Real^{N\times k}$ be the matrix whose rows contain all the learned embeddings, Section~\ref{section:spectral_contrastive_loss} shows that $\globalminF = D\cdot F^*\cdot R$ for a positive diagonal matrix $D$ and an orthonormal matrix $R$, and Lemma~\ref{lemma:prediction_same_with_matrix} shows that these transformations do not affect the feature quality. Therefore, in the rest of the section, it suffices to show that linear models on top of $F^*$ gives the labels of ${g}$. 
Let $\vec{g}\in\{0,1\}^N$ be the vector that contains the labels of all the data under the optimal $g$, i.e., $\vec{g}_x=g(x)$.  Given a linear head $b$, note that $F^*b$ gives the prediction (before the threshold function)  for all examples. Therefore, it suffices to show the existence of a vector $b$  such that 
\begin{align}
F^*b \approx \vec{g}
\end{align}

Let $\laplacian\triangleq {I} - \norA$ be the normalized Laplacian matrix. Then, $v_i$'s are the $k$ smallest unit-norm eigenvectors of $\laplacian$ with eigenvalues $\lambda_i = 1-\gamma_i$.  Elementary derivations can give a well-known, important property of the Laplacian matrix $L$: the quadratic form $\vec{g}^\top L\vec{g}$ captures the amount of edges across the two groups that are defined by the binary vector $\vec{g}$~\cite[section 1.2]{chung1997spectral}:
\begin{align}
\vec{g}^\top \laplacian \vec{g} &= \frac{1}{2}\cdot \sum_{x, x'\in\adata}\frac{w_{xx'}}{\sqrt{w_xw_{x'}}} \left(\vec{g}_x-\vec{g}_{x'}\right)^2\label{eq:diff_g}
\end{align}
With slight abuse of notation, suppose $(x,x^+)$ is the random variable for a positive pair. Using that $w$ is the density function for the positive pair and the simplification that $w_x= 1/N$, we can rewrite equation~\eqref{eq:diff_g} as 
\begin{align}
\vec{g}^\top \laplacian \vec{g} & = \frac{N}{2}\cdot \Exp{x, x^+}[\left(\vec{g}_x-\vec{g}_{x^+}\right)^2], 
\end{align}
Note that $\Exp{x, x^+}[\left(\vec{g}_x-\vec{g}_{x^+}\right)^2]$ is the probability that a positive pair have different labels under the Bayes optimal classifier $g$. 
Because Assumption~\ref{definition:accurate_partition} assumes that the labels can be almost determined by the augmented data, we can show that two augmentations of the same datapoint should rarely produce different labels under the Bayes optimal classifier. We will prove in Lemma~\ref{remark:edge_between_sets} via simple calculation that 
\begin{align}
\vec{g}^\top \laplacian \vec{g} \le N\alpha \label{eqn:11}
\end{align}
(We can sanity-check the special case when $\alpha=0$, that is, the label is determined by the augmentation. In this case, $g(x)=g(x^+)$ for a positive pair $(x, x^+)$ w.p. 1, which implies $\vec{g}^\top \laplacian \vec{g}=\frac{N}{2}\cdot \Exp{x, x^+}[\left(\vec{g}_x-\vec{g}_{x^+}\right)^2]=0$.)

Next, we use equation~\eqref{eqn:11} to link $\vec{g}$ to the eigenvectors of $L$.  Let $\lambda_{k+1}\le \dots \lambda_N$ be the rest of eigenvalues with unit-norm eigenvectors $v_{k+1},\dots, v_N$. Let $\Pi \triangleq \sum_{i=1}^k v_iv_i^\top$ and $\Pi_\perp \triangleq \sum_{i=k+1}^N v_iv_i^\top$ be the projection operators onto the subspaces spanned by the first $k$ and the last $N-k$ eigenvectors, respectively. Equation~\eqref{eqn:11} implies that $\vec{g}$ has limited projection to the subspace of $\Pi_\perp$: 
\begin{align}
N\alpha\ge \vec{g}^\top \laplacian \vec{g} =\left(\Pi\vec{g} + \Pi_\perp\vec{g}\right)^\top \laplacian \left(\Pi\vec{g} + \Pi_\perp\vec{g}\right) \ge \left(\Pi_\perp\vec{g}\right)^\top \laplacian\left(\Pi_\perp\vec{g}\right)
\ge \lambda_{k+1} \norm{\Pi_\perp\vec{g}}_2^2,
\end{align} 
where the first inequality follows from dropping the $\|(\Pi\vec{g})^\top L \Pi\vec{g}\|_2^2$ and using $\Pi_\perp L \Pi = 0$,  and the second inequality is because that $\Pi_\perp$ only contains eigenvectors with eigenvalue at least $\lambda_{k+1}$.  

Note that $\Pi \vec{g}$ is in the span of eigenvectors $v_1,\dots, v_k$, that is, the column-span of $F^*$. Therefore, there exists $b\in \Real^k$ such that $\Pi \vec{g} = F^* b$. As a consequence, 
\begin{align}
\|\vec{g} - F^* b \|_2^2 =  \norm{\Pi_\perp\vec{g}}_2^2 \le \frac{N\alpha}{\lambda_{k+1}}
\end{align}
By higher-order Cheeger inequality (see Lemma~\ref{lemma:higher_order_cheeger}), we have that $\lambda_{k+1}\gtrsim \rho_{\lceil k/2\rceil}^2$. Then, we obtain the mean-squared  error bound: 
\begin{align}
\frac{1}{N} \|\vec{g} - F^* b\|_2^2 \le {\alpha}/{\rho_{\lceil k/2\rceil}^2}
\end{align}
The steps above demonstrate the gist of the proofs, which are formalized in more generality in Section~\ref{section:proof_main_theorem_generalized}. We will also need two minor steps to complete the proof of Theorem~\ref{thm:combine_with_cheeger_simplified}. First, we can convert the mean-squared error bound to classification error bound: because $F^*b$ is close to the binary vector $\vec{g}$ in mean-squared error, $\id{F^*b>1/2}$ is close to $\vec{g}$ in 0-1 error. (See Claim~\ref{claim:loss_to_error} for the formal argument.) Next, $F^*b$ only gives the prediction of the model given the augmented datapoint. We will show in Section~\ref{sec:proof_of_edge_between_sets_remark} that averaging the predictions on the augmentations of a data ponit will not increase the classification error. 

\section{Experiments}\label{section:experiments}
We test spectral contrastive learning on benchmark vision datasets. We minimize the empirical spectral contrastive loss with an encoder network $f$ and sample fresh augmentation in each iteration. The pseudo-code for the algorithm and more implementation details can be found in Section~\ref{section:experiment_detail}.

\noindent\textbf{Encoder / feature extractor.} The encoder $f$ contains three components: a backbone network, a projection MLP and a projection function. The backbone network is a standard ResNet architecture. The projection MLP is a fully connected network with BN applied to each layer, and ReLU activation applied to each except for the last layer. The projection function takes a vector and projects it to a sphere ball with radius $\sqrt{\ballr}$, where $\ballr>0$ is a hyperparameter that we tune in experiments. 
We find that using a projection MLP and a projection function improves the performance.

\noindent\textbf{Linear evaluation protocol.} Given the pre-trained encoder network, we follow the standard linear evaluation protocol~\cite{chen2020exploring} and train a supervised linear classifier on frozen representations, which are from the ResNet’s global average pooling layer. 

\noindent\textbf{Results.} We report the accuracy on CIFAR-10/100~\cite{krizhevsky2009learning} and Tiny-ImageNet~\cite{le2015tiny} in Table~\ref{table:results}. Our empirical results show that spectral contrastive learning achieves better performance than two popular baseline algorithms SimCLR~\cite{chen2020simple} and SimSiam~\cite{chen2020exploring}. In Table~\ref{table:results_imagenet} we report results on ImageNet~\cite{imagenet_cvpr09} dataset, and show that our algorithm achieves similar performance as other state-of-the-art methods. We note that our algorithm is much more principled than previous methods and doesn't rely on large batch sizes (SimCLR~\cite{chen2020simple}), momentum encoders (BYOL~\cite{grill2020bootstrap} and MoCo~\cite{he2020momentum}) or additional tricks such as stop-gradient (SimSiam~\cite{chen2020exploring}).

\begin{table*}[!htb]
	\begin{center}
		{\renewcommand{\arraystretch}{1.1}
			\begin{tabular}{l|r r r|r r r|r r r}
				\toprule
				Datasets & \multicolumn{3}{c |}{CIFAR-10} & \multicolumn{3}{c |}{CIFAR-100} &  \multicolumn{3}{c}{Tiny-ImageNet} \\
				\midrule
				Epochs & \multicolumn{1}{c}{200} & \multicolumn{1}{c}{400} & \multicolumn{1}{c |}{800} & \multicolumn{1}{c}{200} & \multicolumn{1}{c}{400} & \multicolumn{1}{c |}{800} & \multicolumn{1}{c}{200} & \multicolumn{1}{c}{400} & \multicolumn{1}{c}{800} \\
				\midrule
				SimCLR (repro.) & 83.73 & 87.72 & 90.60 & 54.74 & 61.05 & 63.88 & \textbf{43.30} & \textbf{46.46}  & 48.12 \\
				SimSiam  (repro.)  & 87.54 & \textbf{90.31} & 91.40
				& 61.56 & 64.96 & 65.87
				& 34.82 & 39.46 & 46.76\\
				\midrule
				Ours    & \textbf{88.66} & 90.17& \textbf{92.07}
				& \textbf{62.45} & \textbf{65.82} & \textbf{66.18}
				& 41.30 & 45.36 & \textbf{49.86}\\
				\bottomrule
		\end{tabular}}
	\end{center}
	\caption{Top-1 accuracy under linear evaluation protocal. \vspace{10pt}
		\label{table:results}
	}

\end{table*}

\begin{table}[!htb]
	\centering
	\small
	\tablestyle{6pt}{1.1}
	\begin{tabular}{c|ccccc}
		& SimCLR & BYOL& MoCo v2 & SimSiam &  Ours \\
		\shline
		acc. (\%) & 66.5 & 66.5 & 67.4 & 68.1 & 66.97\\ 
	\end{tabular}
	\vspace{.5em}
	\caption{ImageNet linear evaluation accuracy with 100-epoch pre-training. All results but ours are reported from~\cite{chen2020exploring}. We use batch size $384$ during pre-training. 
		\label{table:results_imagenet}
	}
\end{table}

\section{Conclusion}
In this paper, we present a novel theoretical framework of self-supervised learning and provide provable guarantees for the learned representation on downstream linear classification tasks. 
We hope the framework could facilitate future theoretical analyses of self-supervised pretraining losses and inspire new methods. It does not capture the potential implicit bias of optimizers but does take into account the inductive bias of the models. By abstracting away the effect of optimization, we can focus on the effect of pretraining losses and their interaction with the structure of the population data. 
Future directions may include designing better pretraining losses and analyzing more fine-grained properties of the learned representations (e.g., as in recent follow-up works~\cite{shen2022connect, haochen2022beyond}),  by potentially leveraging more advanced techniques from spectral graph theory. 

%

\section*{Acknowledgements}
We thank Margalit Glasgow, Ananya Kumar, Jason D. Lee, Sang Michael Xie, and Guodong Zhang for helpful discussions. 
CW acknowledges support from an NSF Graduate Research Fellowship. 
TM acknowledges support of Google Faculty Award and NSF IIS 2045685. 
We also acknowledge the support of HAI and the Google Cloud. Toyota Research Institute ("TRI")  provided funds to assist the authors with their research but this article solely reflects the opinions and conclusions of its authors and not TRI or any other Toyota entity.

\bibliographystyle{plainnat}
\bibliography{all}

\begin{thebibliography}{74}
\providecommand{\natexlab}[1]{#1}
\providecommand{\url}[1]{\texttt{#1}}
\expandafter\ifx\csname urlstyle\endcsname\relax
  \providecommand{\doi}[1]{doi: #1}\else
  \providecommand{\doi}{doi: \begingroup \urlstyle{rm}\Url}\fi

\bibitem[Abbe(2017)]{abbe2017community}
Emmanuel Abbe.
\newblock Community detection and stochastic block models: recent developments,
  2017.

\bibitem[Arora et~al.(2009)Arora, Rao, and Vazirani]{arora2009expander}
Sanjeev Arora, Satish Rao, and Umesh Vazirani.
\newblock Expander flows, geometric embeddings and graph partitioning.
\newblock \emph{Journal of the ACM (JACM)}, 56\penalty0 (2):\penalty0 1--37,
  2009.

\bibitem[Arora et~al.(2019)Arora, Khandeparkar, Khodak, Plevrakis, and
  Saunshi]{arora2019theoretical}
Sanjeev Arora, Hrishikesh Khandeparkar, Mikhail Khodak, Orestis Plevrakis, and
  Nikunj Saunshi.
\newblock A theoretical analysis of contrastive unsupervised representation
  learning.
\newblock \emph{arXiv preprint arXiv:1902.09229}, 2019.

\bibitem[Bachman et~al.(2019)Bachman, Hjelm, and
  Buchwalter]{bachman2019learning}
Philip Bachman, R~Devon Hjelm, and William Buchwalter.
\newblock Learning representations by maximizing mutual information across
  views.
\newblock \emph{arXiv preprint arXiv:1906.00910}, 2019.

\bibitem[Balcan et~al.(2005)Balcan, Blum, and Yang]{balcan2005co}
Maria-Florina Balcan, Avrim Blum, and Ke~Yang.
\newblock Co-training and expansion: Towards bridging theory and practice.
\newblock \emph{Advances in neural information processing systems},
  17:\penalty0 89--96, 2005.

\bibitem[Bansal et~al.(2020)Bansal, Kaplun, and Barak]{bansal2020self}
Yamini Bansal, Gal Kaplun, and Boaz Barak.
\newblock For self-supervised learning, rationality implies generalization,
  provably.
\newblock \emph{arXiv preprint arXiv:2010.08508}, 2020.

\bibitem[Bardes et~al.(2021)Bardes, Ponce, and LeCun]{bardes2021vicreg}
Adrien Bardes, Jean Ponce, and Yann LeCun.
\newblock Vicreg: Variance-invariance-covariance regularization for
  self-supervised learning.
\newblock \emph{arXiv preprint arXiv:2105.04906}, 2021.

\bibitem[Blum and Mitchell(1998)]{blum1998combining}
Avrim Blum and Tom Mitchell.
\newblock Combining labeled and unlabeled data with co-training.
\newblock In \emph{Proceedings of the eleventh annual conference on
  Computational learning theory}, pages 92--100, 1998.

\bibitem[Bobkov et~al.(1997)]{bobkov1997isoperimetric}
Sergey~G Bobkov et~al.
\newblock An isoperimetric inequality on the discrete cube, and an elementary
  proof of the isoperimetric inequality in gauss space.
\newblock \emph{The Annals of Probability}, 25\penalty0 (1):\penalty0 206--214,
  1997.

\bibitem[Brennan et~al.(2020)Brennan, Bresler, and Nagaraj]{brennan2020phase}
Matthew Brennan, Guy Bresler, and Dheeraj Nagaraj.
\newblock Phase transitions for detecting latent geometry in random graphs.
\newblock \emph{Probability Theory and Related Fields}, 178\penalty0
  (3):\penalty0 1215--1289, 2020.

\bibitem[Bromley et~al.(1993)Bromley, Guyon, LeCun, S{\"a}ckinger, and
  Shah]{bromley1993signature}
Jane Bromley, Isabelle Guyon, Yann LeCun, Eduard S{\"a}ckinger, and Roopak
  Shah.
\newblock Signature verification using a" siamese" time delay neural network.
\newblock \emph{Advances in neural information processing systems}, 6:\penalty0
  737--744, 1993.

\bibitem[Bubeck et~al.(2016)Bubeck, Ding, Eldan, and
  R{\'a}cz]{bubeck2016testing}
S{\'e}bastien Bubeck, Jian Ding, Ronen Eldan, and Mikl{\'o}s~Z R{\'a}cz.
\newblock Testing for high-dimensional geometry in random graphs.
\newblock \emph{Random Structures \& Algorithms}, 49\penalty0 (3):\penalty0
  503--532, 2016.

\bibitem[Bump(1998)]{bump1998automorphic}
Daniel Bump.
\newblock \emph{Automorphic forms and representations}.
\newblock Number~55. Cambridge university press, 1998.

\bibitem[Buser(1982)]{buser1982note}
Peter Buser.
\newblock A note on the isoperimetric constant.
\newblock In \emph{Annales scientifiques de l'{\'E}cole normale
  sup{\'e}rieure}, volume~15, pages 213--230, 1982.

\bibitem[Cai et~al.(2021)Cai, Gao, Lee, and Lei]{cai2021theory}
Tianle Cai, Ruiqi Gao, Jason~D Lee, and Qi~Lei.
\newblock A theory of label propagation for subpopulation shift.
\newblock \emph{arXiv preprint arXiv:2102.11203}, 2021.

\bibitem[Caron et~al.(2020)Caron, Misra, Mairal, Goyal, Bojanowski, and
  Joulin]{caron2020unsupervised}
Mathilde Caron, Ishan Misra, Julien Mairal, Priya Goyal, Piotr Bojanowski, and
  Armand Joulin.
\newblock Unsupervised learning of visual features by contrasting cluster
  assignments.
\newblock \emph{arXiv preprint arXiv:2006.09882}, 2020.

\bibitem[Cheeger(1969)]{d63036efc9d24f07b8908864667e28aa}
Jeff Cheeger.
\newblock A lower bound for the smallest eigenvalue of the laplacian.
\newblock In \emph{Proceedings of the Princeton conference in honor of
  Professor S. Bochner}, pages 195--199, 1969.

\bibitem[Chen et~al.(2020{\natexlab{a}})Chen, Kornblith, Norouzi, and
  Hinton]{chen2020simple}
Ting Chen, Simon Kornblith, Mohammad Norouzi, and Geoffrey Hinton.
\newblock A simple framework for contrastive learning of visual
  representations.
\newblock In \emph{International conference on machine learning}, pages
  1597--1607. PMLR, 2020{\natexlab{a}}.

\bibitem[Chen et~al.(2020{\natexlab{b}})Chen, Kornblith, Swersky, Norouzi, and
  Hinton]{chen2020big}
Ting Chen, Simon Kornblith, Kevin Swersky, Mohammad Norouzi, and Geoffrey
  Hinton.
\newblock Big self-supervised models are strong semi-supervised learners.
\newblock \emph{arXiv preprint arXiv:2006.10029}, 2020{\natexlab{b}}.

\bibitem[Chen and He(2020)]{chen2020exploring}
Xinlei Chen and Kaiming He.
\newblock Exploring simple siamese representation learning.
\newblock \emph{arXiv preprint arXiv:2011.10566}, 2020.

\bibitem[Chen et~al.(2020{\natexlab{c}})Chen, Fan, Girshick, and
  He]{chen2020improved}
Xinlei Chen, Haoqi Fan, Ross Girshick, and Kaiming He.
\newblock Improved baselines with momentum contrastive learning.
\newblock \emph{arXiv preprint arXiv:2003.04297}, 2020{\natexlab{c}}.

\bibitem[Chen(2021)]{chen2021almost}
Yuansi Chen.
\newblock An almost constant lower bound of the isoperimetric coefficient in
  the kls conjecture.
\newblock \emph{Geometric and Functional Analysis}, 31\penalty0 (1):\penalty0
  34--61, 2021.

\bibitem[Chung and Graham(1997)]{chung1997spectral}
Fan~RK Chung and Fan~Chung Graham.
\newblock \emph{Spectral graph theory}.
\newblock Number~92. American Mathematical Soc., 1997.

\bibitem[Dao et~al.(2019)Dao, Gu, Ratner, Smith, De~Sa, and
  R{\'e}]{dao2019kernel}
Tri Dao, Albert Gu, Alexander Ratner, Virginia Smith, Chris De~Sa, and
  Christopher R{\'e}.
\newblock A kernel theory of modern data augmentation.
\newblock In \emph{International Conference on Machine Learning}, pages
  1528--1537. PMLR, 2019.

\bibitem[Dasgupta et~al.(2002)Dasgupta, Littman, and
  McAllester]{dasgupta2002pac}
Sanjoy Dasgupta, Michael~L Littman, and David McAllester.
\newblock Pac generalization bounds for co-training.
\newblock \emph{Advances in neural information processing systems}, 1:\penalty0
  375--382, 2002.

\bibitem[Deng et~al.(2009)Deng, Dong, Socher, Li, Li, and
  Fei-Fei]{imagenet_cvpr09}
J.~Deng, W.~Dong, R.~Socher, L.-J. Li, K.~Li, and L.~Fei-Fei.
\newblock {ImageNet: A Large-Scale Hierarchical Image Database}.
\newblock In \emph{CVPR09}, 2009.

\bibitem[Devroye et~al.(2018)Devroye, Mehrabian, and Reddad]{devroye2018total}
Luc Devroye, Abbas Mehrabian, and Tommy Reddad.
\newblock The total variation distance between high-dimensional gaussians.
\newblock \emph{arXiv preprint arXiv:1810.08693}, 2018.

\bibitem[Eckart and Young(1936)]{eckart1936approximation}
Carl Eckart and Gale Young.
\newblock The approximation of one matrix by another of lower rank.
\newblock \emph{Psychometrika}, 1\penalty0 (3):\penalty0 211--218, 1936.

\bibitem[Golowich et~al.(2018)Golowich, Rakhlin, and Shamir]{golowich2018size}
Noah Golowich, Alexander Rakhlin, and Ohad Shamir.
\newblock Size-independent sample complexity of neural networks.
\newblock In \emph{Conference On Learning Theory}, pages 297--299. PMLR, 2018.

\bibitem[Grill et~al.(2020)Grill, Strub, Altch{\'e}, Tallec, Richemond,
  Buchatskaya, Doersch, Pires, Guo, Azar, et~al.]{grill2020bootstrap}
Jean-Bastien Grill, Florian Strub, Florent Altch{\'e}, Corentin Tallec,
  Pierre~H Richemond, Elena Buchatskaya, Carl Doersch, Bernardo~Avila Pires,
  Zhaohan~Daniel Guo, Mohammad~Gheshlaghi Azar, et~al.
\newblock Bootstrap your own latent: A new approach to self-supervised
  learning.
\newblock \emph{arXiv preprint arXiv:2006.07733}, 2020.

\bibitem[Guggenheimer(1977)]{guggenheimer1977applicable}
Heinrich~Walter Guggenheimer.
\newblock \emph{Applicable Geometry: Global and Local Convexity}.
\newblock RE Krieger Publishing Company, 1977.

\bibitem[HaoChen et~al.(2022)HaoChen, Wei, Kumar, and Ma]{haochen2022beyond}
Jeff~Z HaoChen, Colin Wei, Ananya Kumar, and Tengyu Ma.
\newblock Beyond separability: Analyzing the linear transferability of
  contrastive representations to related subpopulations.
\newblock \emph{arXiv preprint arXiv:2204.02683}, 2022.

\bibitem[He et~al.(2020)He, Fan, Wu, Xie, and Girshick]{he2020momentum}
Kaiming He, Haoqi Fan, Yuxin Wu, Saining Xie, and Ross Girshick.
\newblock Momentum contrast for unsupervised visual representation learning.
\newblock In \emph{Proceedings of the IEEE/CVF Conference on Computer Vision
  and Pattern Recognition}, pages 9729--9738, 2020.

\bibitem[Henaff(2020)]{henaff2020data}
Olivier Henaff.
\newblock Data-efficient image recognition with contrastive predictive coding.
\newblock In \emph{International Conference on Machine Learning}, pages
  4182--4192. PMLR, 2020.

\bibitem[Hjelm et~al.(2018)Hjelm, Fedorov, Lavoie-Marchildon, Grewal, Bachman,
  Trischler, and Bengio]{hjelm2018learning}
R~Devon Hjelm, Alex Fedorov, Samuel Lavoie-Marchildon, Karan Grewal, Phil
  Bachman, Adam Trischler, and Yoshua Bengio.
\newblock Learning deep representations by mutual information estimation and
  maximization.
\newblock In \emph{International Conference on Learning Representations}, 2018.

\bibitem[Kannan et~al.(2004)Kannan, Vempala, and Vetta]{kannan2004clusterings}
Ravi Kannan, Santosh Vempala, and Adrian Vetta.
\newblock On clusterings: Good, bad and spectral.
\newblock \emph{Journal of the ACM (JACM)}, 51\penalty0 (3):\penalty0 497--515,
  2004.

\bibitem[Krizhevsky and Hinton(2009)]{krizhevsky2009learning}
Alex Krizhevsky and Geoffrey Hinton.
\newblock Learning multiple layers of features from tiny images.
\newblock 2009.

\bibitem[Le and Yang(2015)]{le2015tiny}
Ya~Le and Xuan Yang.
\newblock Tiny imagenet visual recognition challenge.
\newblock \emph{CS 231N}, 7:\penalty0 7, 2015.

\bibitem[Lee et~al.(2014)Lee, Gharan, and Trevisan]{lee2014multiway}
James~R Lee, Shayan~Oveis Gharan, and Luca Trevisan.
\newblock Multiway spectral partitioning and higher-order cheeger inequalities.
\newblock \emph{Journal of the ACM (JACM)}, 61\penalty0 (6):\penalty0 1--30,
  2014.

\bibitem[Lee et~al.(2020)Lee, Lei, Saunshi, and Zhuo]{lee2020predicting}
Jason~D Lee, Qi~Lei, Nikunj Saunshi, and Jiacheng Zhuo.
\newblock Predicting what you already know helps: Provable self-supervised
  learning.
\newblock \emph{arXiv preprint arXiv:2008.01064}, 2020.

\bibitem[Lee and Vempala(2016)]{lee2016eldan}
Yin~Tat Lee and Santosh~S Vempala.
\newblock Eldan's stochastic localization and the kls conjecture: Isoperimetry,
  concentration and mixing.
\newblock \emph{arXiv preprint arXiv:1612.01507}, 2016.

\bibitem[Lei et~al.(2015)Lei, Rinaldo, et~al.]{lei2015consistency}
Jing Lei, Alessandro Rinaldo, et~al.
\newblock Consistency of spectral clustering in stochastic block models.
\newblock \emph{Annals of Statistics}, 43\penalty0 (1):\penalty0 215--237,
  2015.

\bibitem[Leighton and Rao(1999)]{leighton1999multicommodity}
Tom Leighton and Satish Rao.
\newblock Multicommodity max-flow min-cut theorems and their use in designing
  approximation algorithms.
\newblock \emph{Journal of the ACM (JACM)}, 46\penalty0 (6):\penalty0 787--832,
  1999.

\bibitem[Liu et~al.(2021)Liu, Mohanty, Schramm, and Yang]{liu2021testing}
Siqi Liu, Sidhanth Mohanty, Tselil Schramm, and Elizabeth Yang.
\newblock Testing thresholds for high-dimensional sparse random geometric
  graphs.
\newblock \emph{arXiv preprint arXiv:2111.11316}, 2021.

\bibitem[Louis and Makarychev(2014)]{louis2014approximation}
Anand Louis and Konstantin Makarychev.
\newblock Approximation algorithm for sparsest k-partitioning.
\newblock In \emph{Proceedings of the twenty-fifth annual ACM-SIAM symposium on
  Discrete algorithms}, pages 1244--1255. SIAM, 2014.

\bibitem[Louis et~al.(2011)Louis, Raghavendra, Tetali, and
  Vempala]{louis2011algorithmic}
Anand Louis, Prasad Raghavendra, Prasad Tetali, and Santosh Vempala.
\newblock Algorithmic extensions of cheeger’s inequality to higher
  eigenvalues and partitions.
\newblock In \emph{Approximation, Randomization, and Combinatorial
  Optimization. Algorithms and Techniques}, pages 315--326. Springer, 2011.

\bibitem[McSherry(2001)]{mcsherry2001spectral}
Frank McSherry.
\newblock Spectral partitioning of random graphs.
\newblock In \emph{Proceedings 42nd IEEE Symposium on Foundations of Computer
  Science}, pages 529--537. IEEE, 2001.

\bibitem[Misra and Maaten(2020)]{misra2020self}
Ishan Misra and Laurens van~der Maaten.
\newblock Self-supervised learning of pretext-invariant representations.
\newblock In \emph{Proceedings of the IEEE/CVF Conference on Computer Vision
  and Pattern Recognition}, pages 6707--6717, 2020.

\bibitem[Mitrovic et~al.(2020)Mitrovic, McWilliams, Walker, Buesing, and
  Blundell]{mitrovic2020representation}
Jovana Mitrovic, Brian McWilliams, Jacob Walker, Lars Buesing, and Charles
  Blundell.
\newblock Representation learning via invariant causal mechanisms.
\newblock \emph{arXiv preprint arXiv:2010.07922}, 2020.

\bibitem[Mohri et~al.(2018)Mohri, Rostamizadeh, and
  Talwalkar]{mohri2018foundations}
Mehryar Mohri, Afshin Rostamizadeh, and Ameet Talwalkar.
\newblock \emph{Foundations of machine learning}.
\newblock MIT press, 2018.

\bibitem[Nadler et~al.(2009)Nadler, Srebro, and Zhou]{nadler2009semi}
Boaz Nadler, Nathan Srebro, and Xueyuan Zhou.
\newblock Semi-supervised learning with the graph laplacian: The limit of
  infinite unlabelled data.
\newblock \emph{Advances in neural information processing systems},
  22:\penalty0 1330--1338, 2009.

\bibitem[Ng et~al.(2001)Ng, Jordan, and Weiss]{ng2001spectral}
Andrew Ng, Michael Jordan, and Yair Weiss.
\newblock On spectral clustering: Analysis and an algorithm.
\newblock \emph{Advances in neural information processing systems},
  14:\penalty0 849--856, 2001.

\bibitem[Oord et~al.(2018)Oord, Li, and Vinyals]{oord2018representation}
Aaron van~den Oord, Yazhe Li, and Oriol Vinyals.
\newblock Representation learning with contrastive predictive coding.
\newblock \emph{arXiv preprint arXiv:1807.03748}, 2018.

\bibitem[Penrose(2003)]{penrose2003random}
Mathew Penrose.
\newblock \emph{Random geometric graphs}, volume~5.
\newblock OUP Oxford, 2003.

\bibitem[Schiebinger et~al.(2015)Schiebinger, Wainwright, and
  Yu]{schiebinger2015geometry}
Geoffrey Schiebinger, Martin~J Wainwright, and Bin Yu.
\newblock The geometry of kernelized spectral clustering.
\newblock \emph{The Annals of Statistics}, 43\penalty0 (2):\penalty0 819--846,
  2015.

\bibitem[Shen et~al.(2022)Shen, Jones, Kumar, Xie, HaoChen, Ma, and
  Liang]{shen2022connect}
Kendrick Shen, Robbie Jones, Ananya Kumar, Sang~Michael Xie, Jeff~Z HaoChen,
  Tengyu Ma, and Percy Liang.
\newblock Connect, not collapse: Explaining contrastive learning for
  unsupervised domain adaptation.
\newblock \emph{arXiv preprint arXiv:2204.00570}, 2022.

\bibitem[Shi and Malik(2000)]{shi2000normalized}
Jianbo Shi and Jitendra Malik.
\newblock Normalized cuts and image segmentation.
\newblock \emph{IEEE Transactions on pattern analysis and machine
  intelligence}, 22\penalty0 (8):\penalty0 888--905, 2000.

\bibitem[Sohn(2016)]{sohn2016improved}
Kihyuk Sohn.
\newblock Improved deep metric learning with multi-class n-pair loss objective.
\newblock In \emph{Proceedings of the 30th International Conference on Neural
  Information Processing Systems}, pages 1857--1865, 2016.

\bibitem[Tian et~al.(2019)Tian, Krishnan, and Isola]{tian2019contrastive}
Yonglong Tian, Dilip Krishnan, and Phillip Isola.
\newblock Contrastive multiview coding.
\newblock \emph{arXiv preprint arXiv:1906.05849}, 2019.

\bibitem[Tian et~al.(2020{\natexlab{a}})Tian, Sun, Poole, Krishnan, Schmid, and
  Isola]{tian2020makes}
Yonglong Tian, Chen Sun, Ben Poole, Dilip Krishnan, Cordelia Schmid, and
  Phillip Isola.
\newblock What makes for good views for contrastive learning.
\newblock \emph{arXiv preprint arXiv:2005.10243}, 2020{\natexlab{a}}.

\bibitem[Tian et~al.(2020{\natexlab{b}})Tian, Yu, Chen, and
  Ganguli]{tian2020understanding}
Yuandong Tian, Lantao Yu, Xinlei Chen, and Surya Ganguli.
\newblock Understanding self-supervised learning with dual deep networks.
\newblock \emph{arXiv preprint arXiv:2010.00578}, 2020{\natexlab{b}}.

\bibitem[Tosh et~al.(2020)Tosh, Krishnamurthy, and Hsu]{tosh2020contrastive}
Christopher Tosh, Akshay Krishnamurthy, and Daniel Hsu.
\newblock Contrastive estimation reveals topic posterior information to linear
  models.
\newblock \emph{arXiv:2003.02234}, 2020.

\bibitem[Tosh et~al.(2021)Tosh, Krishnamurthy, and Hsu]{tosh2021contrastive}
Christopher Tosh, Akshay Krishnamurthy, and Daniel Hsu.
\newblock Contrastive learning, multi-view redundancy, and linear models.
\newblock In \emph{Algorithmic Learning Theory}, pages 1179--1206. PMLR, 2021.

\bibitem[Tsai et~al.(2020)Tsai, Wu, Salakhutdinov, and Morency]{tsai2020self}
Yao-Hung~Hubert Tsai, Yue Wu, Ruslan Salakhutdinov, and Louis-Philippe Morency.
\newblock Self-supervised learning from a multi-view perspective.
\newblock \emph{arXiv preprint arXiv:2006.05576}, 2020.

\bibitem[Wang and Isola(2020)]{wang2020understanding}
Tongzhou Wang and Phillip Isola.
\newblock Understanding contrastive representation learning through alignment
  and uniformity on the hypersphere.
\newblock In \emph{International Conference on Machine Learning}, pages
  9929--9939. PMLR, 2020.

\bibitem[Wei et~al.(2020)Wei, Shen, Chen, and Ma]{wei2020theoretical}
Colin Wei, Kendrick Shen, Yining Chen, and Tengyu Ma.
\newblock Theoretical analysis of self-training with deep networks on unlabeled
  data.
\newblock \emph{arXiv preprint arXiv:2010.03622}, 2020.

\bibitem[{Wikipedia contributors}(2020)]{enwiki:986771357}
{Wikipedia contributors}.
\newblock Hilbert–schmidt integral operator --- {Wikipedia}{,} the free
  encyclopedia, 2020.
\newblock URL
  \url{https://en.wikipedia.org/w/index.php?title=Hilbert%E2%80%93Schmidt_integral_operator&oldid=986771357}.
\newblock [Online; accessed 21-July-2021].

\bibitem[Wu et~al.(2018)Wu, Xiong, Yu, and Lin]{wu2018unsupervised}
Zhirong Wu, Yuanjun Xiong, Stella~X Yu, and Dahua Lin.
\newblock Unsupervised feature learning via non-parametric instance
  discrimination.
\newblock In \emph{Proceedings of the IEEE Conference on Computer Vision and
  Pattern Recognition}, pages 3733--3742, 2018.

\bibitem[Xie et~al.(2019)Xie, Dai, Hovy, Luong, and Le]{xie2019unsupervised}
Qizhe Xie, Zihang Dai, Eduard Hovy, Minh-Thang Luong, and Quoc~V Le.
\newblock Unsupervised data augmentation for consistency training.
\newblock \emph{arXiv preprint arXiv:1904.12848}, 2019.

\bibitem[Ye et~al.(2019)Ye, Zhang, Yuen, and Chang]{ye2019unsupervised}
Mang Ye, Xu~Zhang, Pong~C Yuen, and Shih-Fu Chang.
\newblock Unsupervised embedding learning via invariant and spreading instance
  feature.
\newblock In \emph{Proceedings of the IEEE/CVF Conference on Computer Vision
  and Pattern Recognition}, pages 6210--6219, 2019.

\bibitem[Zbontar et~al.(2021)Zbontar, Jing, Misra, LeCun, and
  Deny]{zbontar2021barlow}
Jure Zbontar, Li~Jing, Ishan Misra, Yann LeCun, and St{\'e}phane Deny.
\newblock Barlow twins: Self-supervised learning via redundancy reduction.
\newblock \emph{arXiv preprint arXiv:2103.03230}, 2021.

\bibitem[Zemel and Carreira-Perpi{\~n}{\'a}n(2004)]{zemel2004proximity}
Richard Zemel and Miguel Carreira-Perpi{\~n}{\'a}n.
\newblock Proximity graphs for clustering and manifold learning.
\newblock \emph{Advances in neural information processing systems}, 17, 2004.

\bibitem[Zhu et~al.(2003)Zhu, Ghahramani, and Lafferty]{zhu2003semi}
Xiaojin Zhu, Zoubin Ghahramani, and John~D Lafferty.
\newblock Semi-supervised learning using gaussian fields and harmonic
  functions.
\newblock In \emph{Proceedings of the 20th International conference on Machine
  learning (ICML-03)}, pages 912--919, 2003.

\bibitem[Zimmermann et~al.(2021)Zimmermann, Sharma, Schneider, Bethge, and
  Brendel]{zimmermann2021contrastive}
Roland~S Zimmermann, Yash Sharma, Steffen Schneider, Matthias Bethge, and
  Wieland Brendel.
\newblock Contrastive learning inverts the data generating process.
\newblock \emph{arXiv preprint arXiv:2102.08850}, 2021.

\end{thebibliography}

\newpage
\appendix

\section{Experiment details}\label{section:experiment_detail}

The pseudo-code for our empirical algorithm is summarized in Algorithm~\ref{alg:spectral_contrastive_learning}.

\begin{algorithm}[!thb]\caption{Spectral Contrastive Learning}\label{alg:spectral_contrastive_learning}
	\begin{algorithmic}[1]
		\Require{batch size $N$, structure of encoder network $f$}		
		\For{sampled minibatch $\{\bar{x}_i\}_{i=1}^N$}
		\For{$i\in\{1, \cdots, N\}$}
		\State draw two augmentations ${x}_i= \textup{aug}(\bar{x}_i)$ and ${x}'_i  =  \textup{aug}(\bar{x}_i)$.	
		\State compute $z_i = f({x}_i)$ and $z_i' = f({x}_i')$.
		\EndFor
		
		\State compute loss 
		$
		\mathcal{L} = -\frac{2}{N} \sum_{i=1}^N z_i^\top z_i' + \frac{1}{N(N-1)} \sum_{i\ne j} (z_i^\top z_j')^2
		$
		\State update $f$ to minimize $\mathcal{L}$
		\EndFor		
		\State \Return encoder network $f(\cdot)$
	\end{algorithmic}
\end{algorithm}

Our results with different hyperparameters on CIFAR-10/100 and Tiny-ImageNet are listed in Table~\ref{table:full_results}.
\begin{table*}[!thb]
	\begin{center}
		{\renewcommand{\arraystretch}{1.25}
			\begin{tabular}{l|r r r|r r r|r r r}
				\toprule
				Datasets & \multicolumn{3}{c |}{CIFAR-10} & \multicolumn{3}{c |}{CIFAR-100} &  \multicolumn{3}{c}{Tiny-ImageNet} \\
				\midrule
				Epochs & \multicolumn{1}{c}{200} & \multicolumn{1}{c}{400} & \multicolumn{1}{c |}{800} & \multicolumn{1}{c}{200} & \multicolumn{1}{c}{400} & \multicolumn{1}{c |}{800} & \multicolumn{1}{c}{200} & \multicolumn{1}{c}{400} & \multicolumn{1}{c}{800} \\
				\midrule
				SimCLR (repro.) & 83.73 & 87.72 & 90.60 & 54.74 & 61.05 & 63.88 & \textbf{43.30} & \textbf{46.46}  & 48.12 \\
				SimSiam  (repro.)  & 87.54 & \textbf{90.31} & 91.40
				& 61.56 & 64.96 & 65.87
				& 34.82 & 39.46 & 46.76\\
				\midrule
				Ours  ($\mu=1$)         & 86.47 & 89.90 & \textbf{92.07}
				& 59.13 & 63.83 & 65.52
				& 28.76 & 33.94 & 40.82\\
				Ours   ($\mu=3$)      & 87.72 & 90.09 & 91.84
				& 61.05 & 64.79 & \textbf{66.18}
				& 40.06 & 42.52 & \textbf{49.86}\\
				Ours   ($\mu=10$)      & \textbf{88.66} & 90.17 & 91.01
				& \textbf{62.45} & \textbf{65.82} & 65.16
				& 41.30 & 45.36 & 47.84\\				
				\bottomrule
		\end{tabular}}
	\end{center}
	\caption{Top-1 accuracy under linear evaluation protocal. \vspace{10pt}
		\label{table:full_results}
	}
\end{table*}

\textbf{Additional details about the encoder.}
For the backbone network, we use the CIFAR variant of ResNet18 for CIFAR-10 and CIFAR-100 experiments and use ResNet50 for Tiny-ImageNet and ImageNet experiments. For the projection MLP, we use a 2-layer MLP with hidden and output dimensions 1000 for CIFAR-10, CIFAR100, and Tiny-ImageNet experiments. We use a 3-layer MLP with hidden and output dimension 8192 for ImageNet experiments. We set $\mu=10$ in the ImageNet experiment, and set $\mu\in\{1, 3, 10\}$ for the CIFAR-10/100 and Tiny-ImageNet experiments.

\textbf{Training the encoder.} We train the neural network using SGD with momentum 0.9. The learning rate starts at 0.05 and decreases to 0 with a cosine schedule. On CIFAR-10/100 and Tiny-ImageNet we use weight decay 0.0005 and train for 800 epochs with batch size 512. On ImageNet we use weight decay 0.0001 and train for 100 epochs with batch size 384. We use 1 GTX 1080 GPU for CIFAR-10/100 and Tiny-ImageNet experiments, and use 8 GTX 1080 GPUs for ImageNet experiments.

\textbf{Linear evaluation protocol.} 
We train the linear head using SGD with batch size 256 and weight decay 0 for 100 epochs, learning rate starts at 30.0 and is decayed by 10x at the 60th and 80th epochs.

\textbf{Image transformation details.} We use the same augmentation strategy as described in~\cite{chen2020exploring}.

\section{Proofs for Section~\ref{section:framework}}\label{section:proof_for_main_result}

We first prove a more generalized version of Theorem~\ref{thm:combine_with_cheeger_simplified} in section~\ref{section:proof_main_theorem_generalized}, and then prove Theorem~\ref{thm:combine_with_cheeger_simplified} in Section~\ref{sec:proof_of_edge_between_sets_remark}. 

\subsection{A generalized version of Theorem~\ref{thm:combine_with_cheeger_simplified}}\label{section:proof_main_theorem_generalized}
For the proof we will follow the convention in literature~\cite{lee2014multiway} and define the \textit{normalized Laplacian matrix} as follows:
\begin{definition}
	Let $G =(\adata, w)$ be the augmentation graph defined in Section~\ref{section:augmentation_graph}. The \textit{normalized Laplacian matrix} of the graph is defined as $\laplacian = {I}-{D}^{-1/2}{A}{D}^{-1/2}$, where ${A}$ is the adjacency matrix with ${A}_{xx'} = \wpair{x}{x'}$ and ${D}$ is a diagonal matrix with ${D}_{xx} = \wnode{x}$.
\end{definition}

It is easy to see that $\laplacian = I - \norA$ where $\norA$ is the normalized adjacency matrix defined in Section~\ref{section:augmentation_graph}. Therefore, when $\lambda_i$ is the $i$-th smallest eigenvalue of $\laplacian$, $1-\lambda_i$ is the $i$-th largest eigenvalue of $\norA$.

We call a function defined on augmented data $\hat{y}: \adata \rightarrow [r]$ an \textit{extended labeling function}. Given an extended labeling function, we define the following quantity that describes the difference between extended labels of two augmented data of the same natural datapoint:
\begin{align}\label{equation:define_phi_y}
\phi^{\hat{y}} :=  \sum_{x, x'\in \adata}\wpair{x}{x'}\cdot\id{\hat{y}(x)\ne\hat{y}(x')}.
\end{align}
We also define the following quantity that describes the difference between extended label of an augmentated datapoint and the ground truth label of the corresponding natural datapoint:
\begin{align}\label{equation:define_delta}
\Delta(y, \hat{y}) := \Pr_{x\sim \pndata, \tilde{x}\sim \aug{x}} \left(\hat{y}(\tilde{x})\ne y(x)\right).
\end{align}

Recall the spectral contrastive loss defined in Section~\ref{section:spectral_contrastive_loss} is:
\begin{align*}
\Loss{f} := \Exp{x_1\sim\pndata, x_2\sim\pndata,\atop
	x\sim\aug{x_1}, x^+\sim\aug{x_1}, x'\sim\aug{x_2}} \left[- 2\cdot f(x)^\top f(x^+) + \left(f(x)^\top f(x')\right)^2\right] .
\end{align*}

We first state a more general version of Theorem~\ref{thm:combine_with_cheeger_simplified} as follows.
\begin{theorem}\label{thm:combine_with_cheeger}
	Assume the set of augmented data $\adata$ is finite. Let $\globalminf \in \arg\min_{f: \adata\rightarrow\Real^k}\Loss{f}$ be a minimizer of the population spectral contrastive loss $\Loss{f}$ with $k\in\mathcal{Z}^+$. Let $k'\ge r$ such that $k+1=(1+\zeta)k'$, where $\zeta\in(0, 1)$ and $k'\in\mathcal{Z}^+$.
	Then, there exists a linear probe $\matrixw^* \in \Real^{r\times k}$ and a universal constant $c$ such that the linear probe predictor satisfies
	\begin{align*}
	\Exp{\bar{x}\sim\pndata, {x}\sim\aug{\bar{x}}}\left[\norm{\vec{{y}}(\bar{x}) - \matrixw^*\globalminf({x})}_2^2\right]\le  c \cdot \left(\textup{poly}(1/\zeta)\cdot\log(k+1)\cdot \frac{\phi^{\hat{y}}}{\rho_{k'}^2} + \Delta(y, \hat{y})\right),
	\end{align*} 
	where $\vec{{y}}(\bar{x})$ is the one-hot embedding of ${y}(\bar{x})$ and $\rho_{k'}$ is the sparsest $m$-partition defined in Definition~\ref{definition:multi_way_expansion_constant}.
	Furthermore, the error of the linear probe predictor can be bounded by
	\begin{align*}
	\Pr_{\bar{x}\sim \pndata, {x}\sim\aug{\bar{x}}} \left(\pred_{\globalminf, \matrixw^* }({x}) \ne y(\bar{x})\right) \le 2c \cdot \left(\textup{poly}(1/\zeta)\cdot\log(k+1)\cdot \frac{\phi^{\hat{y}}}{\rho_{k'}^2} + \Delta(y, \hat{y})\right). 
	\end{align*}
	
	Also, if we let $\lambda_i$ be the $i$-th smallest eigenvalue of the normalized Laplacian matrix of the graph of the augmented data, we can find a matrix $B^*$ satisfying the above equations with norm bound $\norm{\matrixw^* }_F\le 1/(1-\lambda_{k})$.
\end{theorem}

We provide the proof for Theorem~\ref{thm:combine_with_cheeger} below.

Let $\lambda_1, \lambda_2, \cdots, \lambda_k, \lambda_{k+1}$ be the $k+1$ smallest eigenvalues of the Laplacian matrix $L$. The following theorem gives a theoretical guarantee similar to Theorem~\ref{thm:combine_with_cheeger} except for that the bound depends on $\lambda_{k+1}$:

\begin{theorem}\label{theorem:error_rate_with_eigenvectors}
	Assume the set of augmented data $\adata$ is finite.
	Let $\globalminf\in\arg\min_{f: \adata\rightarrow\Real^k} $ be a minimizer of the population spectral contrastive loss $\Loss{f}$ with $k\in\mathcal{Z}^+$. 
	Then, for any labeling function $\hat{y}: \adata\rightarrow[r]$ there exists a linear probe $\matrixw^* \in \Real^{r\times k}$ with norm $\norm{\matrixw^* }_F\le 1/(1-\lambda_{k})$ such that
	\begin{align*}
	\Exp{\bar{x}\sim\pndata, {x}\sim\aug{\bar{x}}}\left[\norm{\vec{{y}}(\bar{x}) - \matrixw^*\globalminf({x})}_2^2\right]\le  \frac{\phi^{\hat{y}}}{\lambda_{k+1}}+ 4\Delta(y, \hat{y}),
	\end{align*} 
	where $\vec{{y}}(\bar{x})$ is the one-hot embedding of ${y}(\bar{x})$. Furthermore, the error can be bounded by
	\begin{align*}
	\Pr_{\bar{x}\sim \pndata, {x}\sim\aug{\bar{x}}} \left(\pred_{\globalminf, \matrixw^* }({x}) \ne y(\bar{x})\right) \le \frac{2\phi^{\hat{y}}}{\lambda_{k+1}}+ 8\Delta(y, \hat{y}). 
	\end{align*}
\end{theorem}

We defer the proof of Theorem~\ref{theorem:error_rate_with_eigenvectors} to Section~\ref{section:proof_of_error_rate_with_eigenvectors}. 

To get rid of the dependency on $\lambda_{k+1}$, we use following higher-order Cheeger's inequality from~\cite{louis2014approximation}. 
\begin{lemma}[Proposition 1.2  in \cite{louis2014approximation}]\label{lemma:higher_order_cheeger}
	Let $G=(V, w)$ be a weight graph with $|V|=N$.  Then, for any $t\in[N]$ and $\zeta>0$ such that $(1+\zeta)t\in[N]$, there exists a partition $S_1, S_2, \cdots, S_{t}$ of $V$ with 
	\begin{align*}
	\phi_G(S_i)\lesssim \textup{poly}({1}/{\zeta})\sqrt{\lambda_{(1+\zeta)t}\log t},
	\end{align*}
	where $\phi_G(\cdot)$ is the Dirichlet conductance defined in Definition~\ref{definition:dirichlet_conductance}.
\end{lemma}

Now we prove Theorem~\ref{thm:combine_with_cheeger} by combining Theorem\ref{theorem:error_rate_with_eigenvectors} and Lemma~\ref{lemma:higher_order_cheeger}.
\begin{proof}[Proof of Theorem~\ref{thm:combine_with_cheeger}]
	Let $G = (\adata, w)$ be the augmentation graph. In Lemma~\ref{lemma:higher_order_cheeger} let $(1+\zeta)t=k+1$ and $t=k'$ we have: there exists partition $S_1,\cdots, S_{k'}\subset \adata$ such that $\phi_G(S_i)\lesssim \textup{poly}(1/\zeta)\sqrt{\lambda_{k+1}\log {(k+1)}}$ for $\forall i \in[k']$. 
	By Definition~\ref{definition:multi_way_expansion_constant}, we have $\rho_{k'} \le \max_{i\in [k']} \phi_G(S_i)  \lesssim \textup{poly}(1/\zeta)\sqrt{\lambda_{k+1}\log {(k+1)}}$, which leads to $\frac{1}{\lambda_{k+1}}\lesssim \textup{poly}(1/\zeta)\cdot{log(k+1)}\cdot \frac{1}{\rho_{k'}^2}$. Plugging this bound to  Theorem~\ref{theorem:error_rate_with_eigenvectors} finishes the proof.
\end{proof}

\subsection{Proof of Theorem~\ref{thm:combine_with_cheeger_simplified}}\label{sec:proof_of_edge_between_sets_remark}
We will use the following lemma which gives a connection between $\phi^{\hat{y}}$, $\Delta(y, \hat{y})$ and Assumption~\ref{definition:accurate_partition}.
\begin{lemma}\label{remark:edge_between_sets}
	Let $G=(\adata, w)$ be the augmentation graph, $r$ be the number of underlying classes. Let $S_1, S_2, \cdots, S_r$ be the partition induced by the classifier $g$ in Assumption~\ref{definition:accurate_partition}. Then, there exists an extended labeling function $\hat{y}$ such that 
	\begin{align*}
	\Delta(y, \hat{y}) \le \alpha
	\end{align*}
	and 
	\begin{align*}
	\phi^{\hat{y}}  = \sum_{x, x'\in \adata} \wpair{x}{x'}\cdot \id{\hat{y}(x)\ne \hat{y}(x')}\le 2\alpha.
	\end{align*}
\end{lemma}
\begin{proof}[Proof of Lemma~\ref{remark:edge_between_sets}]
	We define function $\hat{y}:\adata\rightarrow [r]$ as follows: for an augmented data $x\in\adata$,  we use function $\hat{y}(x)$ to represent the index of set that $x$ is in, i.e., $x\in S_{\hat{y}(x)}$. By Assumption~\ref{definition:accurate_partition} it is easy to see $\Delta(y, \hat{y})\le \alpha$. 
	On the other hand, we have
	\begin{align*}
	\phi^{\hat{y}} = &\sum_{x, x'\in\adata}\wpair{x}{x'} \id{\hat{y}(x)\ne \hat{y}(x')}\\
	=&\sum_{x, x'\in \adata}\Exp{\bar{x}\sim\pndata}\left[  \augp{x}{\bar{x}} \augp{x'}{\bar{x}} \cdot \id{\hat{y}(x)\ne \hat{y}(x')}\right]\\
	\le & \sum_{x, x'\in \adata}\Exp{\bar{x}\sim\pndata}\left[ \augp{x}{\bar{x}} \augp{x'}{\bar{x}} \cdot \left(\id{\hat{y}(x)\ne y(\bar{x})} + \id{\hat{y}(x')\ne y(\bar{x})}\right)\right]\\
	= & 2\cdot \Exp{\bar{x}\sim\pndata}\left[ \augp{x}{\bar{x}} \cdot \id{\hat{y}(x)\ne y(\bar{x})}\right]\\
	= & 2\cdot \Pr_{\bar{x}\sim\pndata, x\sim \aug{\bar{x}}}\left(x\notin S_{y(\bar{x})}\right) = 2\alpha.
	\end{align*}
	Here the inequality is because when $\hat{y}(x)\ne \hat{y}(x')$, there must be $\hat{y}(x)\ne y(\bar{x})$ or $\hat{y}(x')\ne y(\bar{x})$.
\end{proof}

Now we give the proof of Theorem~\ref{thm:combine_with_cheeger_simplified} using Lemma~\ref{remark:edge_between_sets} and Theorem~\ref{thm:combine_with_cheeger}.

\begin{proof}[Proof of Theorem~\ref{thm:combine_with_cheeger_simplified}]
	Let $S_1, S_2, \cdots, S_r$ be the partition of $\adata$ induced by the classifier $g$ given in Assumption~\ref{definition:accurate_partition}.
	Define function $\hat{y}:\adata\rightarrow [r]$ as follows: for an augmented datapoint $x\in\adata$,  we use function $\hat{y}(x)$ to represent the index of set that $x$ is in, i.e., $x\in S_{\hat{y}(x)}$.  Let $k'=\lfloor\frac{k}{2}\rfloor$ in Theorem~\ref{thm:combine_with_cheeger}, we have $
	\Pr_{\bar{x}\sim \pndata, {x}\sim\aug{\bar{x}}} \left(\pred_{\globalminf, \matrixw^* }({x}) \ne y(\bar{x})\right) \lesssim \log(k)\cdot \frac{\phi^{\hat{y}}}{\rho_{\lfloor k/2\rfloor}^2} + \Delta(y, \hat{y}). 
	$ By Lemma~\ref{remark:edge_between_sets} we have $\phi^{\hat{y}}\le2\alpha$ and $\Delta(y, \hat{y}) \le \alpha$, so we have $
	\Pr_{\bar{x}\sim \pndata, {x}\sim\aug{\bar{x}}} \left(\pred_{\globalminf, \matrixw^* }({x}) \ne y(\bar{x})\right) \lesssim \frac{\alpha}{\rho_{\lfloor k/2\rfloor}^2}\cdot \log(k). 
	$ Notice that by definition of ensembled linear probe predictor, $\npred_{\globalminf, \matrixw ^*}(\bar{x})\ne y(\bar{x})$ happens only if more than half of the augmentations of $\bar{x}$ predicts differently from $y(\bar{x})$, so we have $\Pr_{\bar{x}\sim\pndata}\left(\npred_{\globalminf, \matrixw ^*}(\bar{x})\ne y(\bar{x})\right) \le 2\Pr_{\bar{x}\sim \pndata, {x}\sim\aug{\bar{x}}} \left(\pred_{\globalminf, \matrixw^* }({x}) \ne y(\bar{x})\right) \lesssim \frac{\alpha}{\rho_{\lfloor k/2\rfloor}^2}\cdot \log(k)$.
\end{proof}

\subsection{Proof of Theorem~\ref{theorem:error_rate_with_eigenvectors}}\label{section:proof_of_error_rate_with_eigenvectors}

The proof of Theorem~\ref{theorem:error_rate_with_eigenvectors} contains two steps. First, we show that when the feature extractor is composed of the minimal eigenvectors of the normalized Laplacian matrix $L$, we can achieve good linear probe accuracy. Then we show that minimizing $\Loss{f}$ gives us a feature extractor equally good as the eigenvectors.

For the first step, we use the following lemma which shows that the smallest eigenvectors of $\laplacian$ can approximate any function on $\adata$ up to an error proportional to the Rayleigh quotient of the function.

\begin{lemma}\label{lemma:approximate_vector_of_small_rayleight_quotient}
	Let $\laplacian$ be the normalized Laplacian matrix of some graph $G$. Let $\sizead=|\adata|$ be total number of augmented data, ${\eigv}_i$ be the $i$-th smallest unit-norm eigenvector of $\laplacian$ with eigenvalue $\lambda_i$ (make them orthogonal in case of repeated eignevalues). Let $R({u}) := \frac{{u}^\top \laplacian {u}}{{u}^\top {u}}$ be the Rayleigh quotient of a vector ${u}\in \Real^{\sizead}$ . Then, for any $k\in \mathcal{Z}^+$ such that $ k<\sizead$ and $\lambda_{k+1}>0$, there exists a vector ${\vectorw}\in \Real^{k}$ with norm $\norm{\vectorw}_2\le \norm{{u}}_2$ such that 
	\begin{align*}
	\norm{u - \sum_{i=1}^k {\vectorw}_{i}{\eigv}_i }_2^2 \le \frac{R({u})}{\lambda_{k+1}} \norm{{u}}_2^2.
	\end{align*} 
\end{lemma}

\begin{proof}[Proof of Lemma~\ref{lemma:approximate_vector_of_small_rayleight_quotient}]
	We can decompose the vector ${u}$ in the eigenvector basis as: 
	\begin{align*}
	{u} = \sum_{i=1}^{\sizead} \zeta_i {\eigv}_i.
	\end{align*}
	We have 
	\begin{align*}
	R({u}) = \frac{\sum_{i=1}^{\sizead} \lambda_i \zeta_i^2}{\norm{{u}}_2^2}.
	\end{align*}
	Let ${\vectorw}\in\Real^k$ be the vector such that ${\vectorw}_{ i}=\zeta_i$. Obviously we have $\norm{{\vectorw}}_2^2 \le \norm{{u}}_2^2$. Noticing that 
	\begin{align*}
	\norm{{u} - \sum_{i=1}^{k}{\vectorw}_{i}{\eigv}_i}_2^2 = \sum_{i=k+1}^{\sizead} \zeta_i^2 \le \frac{R({u})}{\lambda_{k+1}} \norm{{u}}_2^2,
	\end{align*}
	which finishes the proof.
\end{proof}

We also need the following claim about the Rayleigh quotient $R({u})$ when ${u}$ is a vector defined by an extended labeling function $\hat{y}$.
\begin{claim}\label{corollary:sparsecut_to_rayleigh_quotient}
	In the setting of Lemma~\ref{lemma:approximate_vector_of_small_rayleight_quotient}, let $\hat{y}$ be an extended labeling function. Fix $i\in [r]$. Define function $u_i^{\hat{y}}(x) := \sqrt{\wnode{x}} \cdot \id{\hat{y}(x) = i} $ and ${u}_i^{\hat{y}}$ is the corresponding vector in $\Real^{\sizead}$. Also define the following quantity:
	\begin{align*}
	\phi_i^{\hat{y}} := \frac{
		\sum_{x, x'\in \adata} \wpair{x}{x'}\cdot \id{(\hat{y}(x) = i\land \hat{y}(x')\ne i) \text{ or }(\hat{y}(x) \ne i \land \hat{y}(x')= i)} 
	}{
		\sum_{x\in \adata} \wnode{x}\cdot  \id{\hat{y}(x) = i}
	}.
	\end{align*}
	Then, we have 
	\begin{align*}
	R(u_i^{\hat{y}})=\frac{1}{2}\phi_i^{\hat{y}}.
	\end{align*}
\end{claim}

\begin{proof}[Proof of Claim~\ref{corollary:sparsecut_to_rayleigh_quotient}]
	Let $f$ be any function $\adata \rightarrow \Real$, define function $u(x) := \sqrt{\wnode{x}} \cdot f(x)$. Let ${u}\in\Real^{\sizead}$ be the vector corresponding to $u$. Let ${A}$ be the adjacency matrix with ${A}_{xx'} = \wpair{x}{x'}$ and ${D}$ be the diagonal matrix with ${D}_{xx} = \wnode{x}$. By definition of Laplacian matrix, we have
	\begin{align*}
	{u}^\top \laplacian {u} &=  \norm{{u}}_2^2 - {u}^\top  {D}^{-1/2}{A}{D}^{-1/2} {u}\\
	&= \sum_{x\in\adata}\wnode{x} f(x)^2 - \sum_{x, x'\in \adata}\wpair{x}{x'}  f(x) f(x')\\
	&= \frac{1}{2}\sum_{x, x'\in \adata}\wpair{x}{x'}\cdot \left(f(x)-f(x')\right)^2.
	\end{align*}
	Therefore we have
	\begin{align*}
	R({u}) &= \frac{{u}^\top \laplacian {u}}{{u}^\top {u}}\\
	&= \frac{1}{2}\cdot \frac{\sum_{x, x'\in \adata}\wpair{x}{x'}\cdot \left(f(x)-f(x')\right)^2}{\sum_{x\in\adata}\wnode{x}\cdot f(x)^2}.
	\end{align*}
	Setting $f(x) = \id{\hat{y}(x) = i}$ finishes the proof.
\end{proof}

To see the connection between the feature extractor minimizing the population spectral contrastive loss $\Loss{f}$ and the feature extractor corresponding to eigenvectors of the Laplacian matrix, we use the following lemma which states that the minimizer of the matrix approximation loss defined in Section~\ref{section:spectral_contrastive_loss} is equivalent to the minimizer of population spectral contrastive loss up to a data-wise scaling. 

\begin{lemma}\label{lemma:equivalent_between_two_losses}
	Let ${f}: \adata\rightarrow\Real^k$ be a feature extractor, matrix
	${{F}}\in\Real^{N\times k}$ be such that its $x$-th row is $\sqrt{\wnode{x}}\cdot {f}(x)$. Then, ${{F}}$ is a minimizer of $\Lossmc{{F}}$ if and only if 
	$f$ is a minimizer of the population spectral contrastive loss $\Loss{f}$. 
\end{lemma}

%

\begin{proof}[Proof of Lemma~\ref{lemma:equivalent_between_two_losses}]
	Notice that
	\begin{align}\label{equation:derive_loss}
	\Lossmc{{F}} &= \norm{({I}-\laplacian)-{F}{F}^\top}_F^2\nonumber\\
	&= \sum_{x, x'\in \adata} \left(\frac{\wpair{x}{x'}}{\sqrt{\wnode{x}\wnode{x'}}} -\sqrt{w_xw_{x'}} f(x)^\top f(x')\right)^2\nonumber\\
	&= \sum_{x, x'\in\adata} w_xw_{x'}\left(f(x)^\top f(x')\right)^2 - 2\sum_{x, x'\in \adata} {\wpair{x}{x'}} f(x)^\top f(x') + \norm{{I}-\laplacian}_F^2.
	\end{align}
	
	Recall that the definition of spectral contrastive loss is
	\begin{align*}
	\Loss{f} := -2\cdot \Exp{x, x^+} \left[f(x)^\top f(x^+) \right]
	+ \Exp{x, x^-}\left[\left(f(x)^\top f(x^-) \right)^2\right],
	\end{align*}   
	where $(x, x^+)$ is a random positive pair, $(x, x^-)$ is a random negative pair. We can rewrite the spectral contrastive loss as
	\begin{align}\label{equation:derive_lossmc}
	\Loss{f} &= -2 \sum_{x, x'\in \adata}\wpair{x}{x'} \cdot f(x)^\top f(x') + \sum_{x, x'\in \adata} \wnode{x}\wnode{x'}\cdot\left(f(x)^\top f(x') \right)^2.
	\end{align}
	
	Compare Equation~\eqref{equation:derive_loss} and Equation~\eqref{equation:derive_lossmc}, we see they only differ by a constant, which finishes the proof.
\end{proof}

%

Note that the minimizer of matrix approximation loss is exactly the largest eigenvectors of ${I}-{L}$ (also the smallest eigenvectors of ${L}$) due to Eckart–Young–Mirsky theorem, Lemma~\ref{lemma:equivalent_between_two_losses} indicates that the minimizer of $\Loss{f}$ is equivalent to the smallest eigenvectors of $\laplacian$ up to data-wise scaling.

The following claim shows the relationship between quadratic loss and prediction error. 

\begin{claim}\label{claim:loss_to_error}
  Let $f:\adata\rightarrow\Real^k$ be a feature extractor, $B\in\Real^{k\times k}$ be a linear head. Let $g_{f, B}$ be the predictor defined in Section~\ref{section:framework}. Then, for any $x\in\adata$ and label $y\in[k]$, we have 
  \begin{align*}
  	\norm{\vec{y}-Bf(x)}_2^2 \ge \frac{1}{2}\cdot\id{y\ne g_{f, B}(x)},
  \end{align*}
  where $\vec{y}$ is the one-hot embedding of $y$.
\end{claim}

\begin{proof}
	When $y\ne g_{f, B}(x)$, by the definition of $g_{f, B}$ we know that there exists another $y'\ne y$ such that $(Bf(x))_{y'}\ge(Bf(x))_{y}$. In this case,
	\begin{align}
	\norm{\vec{y}-Bf(x)}_2^2 &\ge \left(1-(Bf(x))_{y}\right)^2 + (Bf(x))_{y'}^2\\
	&\ge \frac{1}{2} \left(1-(Bf(x))_{y}+(Bf(x))_{y'}\right)^2\\
	&\ge \frac{1}{2},
	\end{align}
	where the first inequality is by omitting all the dimensions in the $\ell_2$ norm other than the $y$-th and $y'$-th dimensions, the second inequality is by Jensen's inequality, and the third inequality is because $(Bf(x))_{y'}\ge(Bf(x))_{y}$. 
	This proves the inequality in the claim when $y\ne g_{f, B}(x)$.
	Finally, we finish the proof by noticing that the inequality in this claim obviously holds when $y=g_{f, B}(x)$.
\end{proof}

Now we are ready to prove Theorem~\ref{theorem:error_rate_with_eigenvectors} by combining  Lemma~\ref{lemma:approximate_vector_of_small_rayleight_quotient}, Claim~\ref{corollary:sparsecut_to_rayleigh_quotient}, Lemma~\ref{lemma:equivalent_between_two_losses} and Claim~\ref{claim:loss_to_error}.

\begin{proof}[Proof of Theorem~\ref{theorem:error_rate_with_eigenvectors}]	
	Let ${F}_{\scf}=[{\eigv}_1, {\eigv}_2, \cdots, {\eigv}_k]$ be the matrix that contains the smallest $k$ eigenvectors of $\laplacian$ as columns. For each $i\in [r]$, we define function $u_i^{\hat{y}}(x) := \sqrt{\wnode{x}} \cdot \id{\hat{y}(x) = i} $ and ${u}_i^{\hat{y}}$ be the corresponding vector in $\Real^{\sizead}$.	By Lemma~\ref{lemma:approximate_vector_of_small_rayleight_quotient}, there exists a vector ${\vectorw}_i\in \Real^k$ with norm bound $\norm{{\vectorw}_i}_2\le \norm{{u}_i^{\hat{y}}}_2$ such that 
	\begin{align}\label{equation:bound_on_indicator_vector}
	\norm{{u}_i^{\hat{y}} - {F}_{\scf} \vectorw_i}_2^2 &\le \frac{R({u}_i^{\hat{y}})}{\lambda_{k+1}} \norm{{u}_i^{\hat{y}}}_2^2.
	\end{align}
	
	By Claim~\ref{corollary:sparsecut_to_rayleigh_quotient}, we have 
	\begin{align*}
	R({u}_i^{\hat{y}}) = \frac{1}{2}\phi_i^{\hat{y}} =\frac{1}{2}\cdot \frac{
		\sum_{x, x' \in \adata} \wpair{x}{x'}\cdot \id{(\hat{y}(x) = i\land \hat{y}(x')\ne i) \text{ or }(\hat{y}(x) \ne i \land \hat{y}(x')= i)} 
	}{
		\sum_{x\in \adata} \wnode{x}\cdot  \id{\hat{y}(x) = i}
	}.
	\end{align*}
	So we can rewrite Equation~\eqref{equation:bound_on_indicator_vector} as:
	\begin{align}\label{equation:bound_on_indicator_vector_plus}
	\norm{{u}_i^{\hat{y}} - {F}_{\scf} {\vectorw}_i}_2^2 &\le \frac{\phi_i^{\hat{y}}}{2\lambda_{k+1}}  \cdot \sum_{x\in \adata} \wnode{x}\cdot \id{\hat{y}(x)=i} \nonumber\\
	&=  \frac{1}{2\lambda_{k+1}} \sum_{x, x' \in \adata}\wpair{x}{x'}\cdot \id{(\hat{y}(x) = i\land \hat{y}(x')\ne i) \text{ or }(\hat{y}(x) \ne i \land \hat{y}(x')= i)}.
	\end{align}
	
	Let matrix ${U} = [{u}_1^{\hat{y}}, \cdots, {u}_r^{\hat{y}}]$ contains all ${u}_i^{\hat{y}}$ as columns, and let $u:\adata\rightarrow\Real^r$ be the corresponding feature extractor. 
	Define matrix $\matrixw \in\Real^{\sizead\times k}$ such that
	$\matrixw ^\top = [{\vectorw}_1, \cdots, {\vectorw}_r]$. Summing Equation~\eqref{equation:bound_on_indicator_vector_plus} over all $i\in [r]$ and by the definition of $\phi^{\hat{y}}$ we have
	\begin{align}\label{equation:error_on_all_data}
	\norm{{U} - {F}_{\scf}\matrixw ^\top}_F^2 \le \frac{1}{2\lambda_{k+1}} \sum_{x, x' \in \adata}\wpair{x}{x'}\cdot \id{\hat{y}(x) \ne \hat{y}(x')} = \frac{\phi^{\hat{y}}}{2\lambda_{k+1}},
	\end{align}
	where 
	\begin{align*}
	\norm{\matrixw }_F^2 = \sum_{i=1}^{r}\norm{{\vectorw}_i}_2^2 \le \sum_{i=1}^{r}\norm{{u}_i^{\hat{y}}}_2^2 = \sum_{x\in\adata}\wnode{x}=1.
	\end{align*}
	
	Now we come back to the feature extractor $\globalminf$ that minimizes the spectral contrastive loss function $\Loss{f}$. By Lemma~\ref{lemma:equivalent_between_two_losses}, matrix ${F}^*$ that contains $\sqrt{\wnode{x}}\cdot \globalminf(x)$ as its $x$-th row is a minimizer of $\Lossmc{F}$. By Eckard-Young-Mirsky theorem, we have 
	\begin{align*}
	{F}^* = {F}_{\scf} {D}_\lambda {Q},
	\end{align*}
	where ${Q}$ is an orthonormal matrix and 
	\begin{align*}
	{D}_\lambda = \begin{bmatrix}
	\sqrt{1-\lambda_1} & & & \\
	& \sqrt{1-\lambda_2} & & \\
	& & \cdots & \\
	& & & \sqrt{1-\lambda_k}
	\end{bmatrix}.
	\end{align*}
	
	Let 
	\begin{align*}
	\matrixw ^* = \matrixw {D}_\lambda^{-1}{Q}^{-1},
	\end{align*}
	and let $\vec{{y}}(\bar{x})$ be the one-hot embedding of ${y}(\bar{x})$, $\vec{\hat{y}}({x})$ be the one-hot embedding of $\hat{y}({x})$, we have
	\begin{align*}
	&\Exp{\bar{x}\sim\pndata, {x}\sim\aug{\bar{x}}}\left[\norm{\vec{{y}}(\bar{x}) - \matrixw ^*\globalminf({x})}_2^2\right]\\
	\le &2 \Exp{\bar{x}\sim\pndata, {x}\sim\aug{\bar{x}}}\left[\norm{\vec{\hat{y}}({x}) - \matrixw ^*\globalminf({x})}_2^2\right] + 2\Exp{\bar{x}\sim\pndata, {x}\sim\aug{\bar{x}}}\left[\norm{\vec{\hat{y}}({x}) - \vec{y}(\bar{x})}_2^2\right] \\
	=& 2\sum_{{x}\in\adata} \wnode{{x}}\cdot  \norm{\vec{\hat{y}}({x}) - \matrixw ^*\globalminf({x})}_2^2 + 4\Delta(y, \hat{y}) \tag{because $\wnode{x}$ is the probability of $x$}\\
	=& 2\norm{{U}-{F}^* {\matrixw ^*}^\top}_F^2+ 4\Delta(y, \hat{y})\tag{rewrite in matrix form}\\
	=& 2\norm{{U} - {F}_{\scf}\matrixw ^\top}_F^2 + 4\Delta(y, \hat{y})\tag{by definition of $\matrixw^*$}\\
	\le & \frac{\phi^{\hat{y}}}{\lambda_{k+1}}+ 4\Delta(y, \hat{y}).\tag{by Equation~\eqref{equation:error_on_all_data}}
	\end{align*}
	
	To bound the error rate, we first notice that Claim~\ref{claim:loss_to_error} tells us that for any ${x}\in\adata$,
	\begin{align}\label{equation:error_one_data}
	\norm{\vec{{y}}(\bar{x}) - \matrixw^* \globalminf({x})}_2^2\ge  \frac{1}{2}  \cdot \id{\pred_{\globalminf, \matrixw^* }({x}) \ne {y}(\bar{x})}.
	\end{align}
	
	Now we bound the error rate on $\adata$ as follows:
	\begin{align*}
	&\Pr_{\bar{x}\sim \pndata, {x}\sim\aug{\bar{x}}} \left(\pred_{\globalminf, \matrixw^* }({x}) \ne y(\bar{x})\right)\\
	\le & 2 \Exp{\bar{x}\sim\pndata, {x}\sim\aug{\bar{x}}}\left[\norm{\vec{{y}}(\bar{x}) - \matrixw ^*\globalminf({x})}_2^2\right] \tag{by Equation~\eqref{equation:error_one_data}}\\
	\le &  \frac{2\phi^{\hat{y}}}{\lambda_{k+1}} + 8\Delta(y, \hat{y}).
	\end{align*}
	
	Finally we bound the norm of $B^*$ as
	\begin{align*}
	\norm{B^*}_F^2 = Tr\left(B^*{B^*}^\top\right) = Tr\left(\matrixw D_\lambda^{-2} \matrixw^\top\right) \le \frac{1}{1-\lambda_{k}}\norm{\matrixw}_F^2= \frac{1}{1-\lambda_{k}}.
	\end{align*}
\end{proof}

\section{Proofs for Section~\ref{section:gaussian_example}}\label{section:proof_of_gaussian_example}

\subsection{Proof of Proposition~\ref{proposition:rho_lower_bound}}\label{section:proof_rho_lower_bound}

\begin{proof}[Proof of Proposition~\ref{proposition:rho_lower_bound}]
	Let $B_\sigma$ be the uniform distribution over a ball with radius $\sigma$.
	Let $S_1, S_2, \cdots, S_{m+1}$ be a partition of the Euclidean space. There must be some $i\in[m+1]$ such that $\Pr_{{x}\sim P_j, \xi\sim B_\sigma}[{x}+\xi \in S_i]\le\frac{1}{2}$ for all $j\in[m]$. Thus, we know that 
	\begin{align}\label{eqn:rho_lower_bound_1}
		\rho_{m+1}\ge \phi_G(S_i)\ge \min_{j\in[m]} \frac{\Pr_{{x}\sim P_j, \xi \sim B_\sigma, \xi' \sim B_\sigma}[{x}+\xi \in S_i \land {x}+\xi' \notin S_i]}{\Pr_{{x}\sim P_j, \xi\sim B_\sigma}[{x}+\xi \in S_i]}.
	\end{align}
	For $x\in\Real^d$,  we use $P(S_i|x)$ as a shorthand for $\Pr_{\xi\sim B_\sigma}(x+\xi\in S_i)$. Let
	\begin{align}
		R := \left\{x \bigg\vert  \Pr(S_i|x) \ge \frac{2}{3}\right\}.
	\end{align}
	On one hand, suppose $\int_{x\notin R} P_j(x) P(S_i|x)dx\ge \frac{1}{2} \Pr_{{x}\sim P_j, \xi\sim B_\sigma}[{x}+\xi \in S_i]$, we can lower bound the numerator in the RHS of Equation~\eqref{eqn:rho_lower_bound_1} as
	\begin{align}
		\Pr_{{x}\sim P_j, \xi \sim B_\sigma, \xi' \sim B_\sigma}[{x}+\xi \in S_i \land {x}+\xi' \notin S_i]  &\ge \int_{x\notin R} P_j(x) P(S_i|x) (1-P(S_i|x)) dx \\
		&\ge \frac{1}{6} \Pr_{{x}\sim P_j, \xi\sim B_\sigma}[{x}+\xi \in S_i],
	\end{align}
	hence the RHS of Equation~\eqref{eqn:rho_lower_bound_1} is at least $1/6$.
	
	On the other hand, suppose $\int_{x\notin R} P_j(x) P(S_i|x)dx< \frac{1}{2} \Pr_{{x}\sim P_j, \xi\sim B_\sigma}[{x}+\xi \in S_i]$, we have
	\begin{align}
		\int_{x\in R} P_j(x) P(S_i|x)dx =\Pr_{{x}\sim P_j, \xi\sim B_\sigma}[{x}+\xi \in S_i] - \int_{x\notin R} P_j(x) P(S_i|x)dx \ge \frac{1}{2}\Pr_{{x}\sim P_j, \xi\sim B_\sigma}[{x}+\xi \in S_i],
	\end{align}
	hence the denominator of the RHS of Equation~\eqref{eqn:rho_lower_bound_1} can be upper bounded by 
	\begin{align}\label{eqn:rho_lower_bound_2}
		\Pr_{{x}\sim P_j, \xi\sim B_\sigma}[{x}+\xi \in S_i] \le 2 \int_{x\in R} P(S_i|x)P_j(x)dx \le 2\int_{x\in R}P_j(x)dx.
	\end{align}
	Define 
	\begin{align}
		N(R) := \left\{x \bigg \vert \norm{x-a}_2 \le \frac{\sigma}{6} \text{ for some } a\in R\right\}.
	\end{align}
	For two Gaussian distributions with variance $\sigma^2\cdot \mathcal{I}_{d\times d}$ and centers at most $\frac{\sigma}{6}$ far from each other, their TV-distance is at most $\frac{1}{6}$ (see the first equation on Page 5  of~\cite{devroye2018total}), hence for any $x\in N(R)$, we have $P(S_i|x) \ge \frac{2}{3}-\frac{1}{6} = \frac{1}{2}$. We can now lower bound the numerator in the RHS of Equation~\eqref{eqn:rho_lower_bound_1} as:
	\begin{align}\label{eqn:rho_lower_bound_3}
		\Pr_{{x}\sim P_j, \xi \sim B_\sigma, \xi' \sim B_\sigma}[{x}+\xi \in S_i \land {x}+\xi' \notin S_i]  &\ge \int_{x\in N(R)\backslash R} P_j(x)P(S_i|x)(1-P(S_i|x))dx\nonumber\\
		&\ge \frac{1}{6} \int_{x\in N(R)\backslash R} P_j(x) dx.
	\end{align} 
	Notice that $\Pr_{{x}\sim P_j, \xi\sim B_\sigma}[{x}+\xi \in S_i]\le\frac{1}{2}$ and by the definition of $R$, we know $\int_{x\in R} P_j(x)dx\le \frac{3}{4}$, thus
	\begin{align}\label{eqn:rho_lower_bound_4}
		\int_{x\in R} P_j(x)dx\le \frac{3}{4} \le 3 \int_{x\notin R} P_j(x)dx.
	\end{align}
	Combine Equation~\eqref{eqn:rho_lower_bound_2}, Equation~\eqref{eqn:rho_lower_bound_3} and Equation~\eqref{eqn:rho_lower_bound_4} gives:
	\begin{align}
		\frac{\Pr_{{x}\sim P_j, \xi \sim B_\sigma, \xi' \sim B_\sigma}[{x}+\xi \in S_i \land {x}+\xi' \notin S_i]}{\Pr_{{x}\sim P_j, \xi\sim B_\sigma}[{x}+\xi \in S_i]} \ge \frac{1}{36} \frac{\int_{x\in N(R)\backslash R} P_j(x) P(S_i|x)dx}{\min\{\int_{x\in R}P_j(x)dx, \int_{x\in R}P_j(x)dx\}}.
	\end{align}
	Notice that (using the definition of surface area~\cite[chapter 4]{guggenheimer1977applicable})
	\begin{align}
		\lim_{\sigma\rightarrow0^+} \frac{1}{\sigma}\cdot  \frac{\int_{x\in N(R)\backslash R} P_j(x) P(S_i|x)dx}{\min\{\int_{x\in R}P_j(x)dx, \int_{x\in R}P_j(x)dx\}} \ge \frac{1}{6} h_{P_j},
	\end{align}
	we have that as $\sigma\rightarrow 0^+$, 
	\begin{align}
		\frac{\rho_{m+1}}{\sigma} \ge \frac{1}{216}\min_{j\in[m]} h_{P_j},
	\end{align}
	which finishes the proof.
	
\end{proof}

\subsection{Proof of Theorem~\ref{theorem:gaussian_example}}\label{section:proof_Gaussian_thm}
In this section, we give a proof of Theorem~\ref{theorem:gaussian_example}. 

The following lemma shows that the augmented graph for Example~\ref{example:gaussian} satisfies Assumption~\ref{definition:accurate_partition} with some bounded $\alpha$.

\begin{lemma}\label{lemma:gaussian_example_lemma1}
	In the setting of Theorem~\ref{theorem:gaussian_example}, the data distribution satisfies Assumption~\ref{definition:accurate_partition} with $\alpha\le\frac{1}{\textup{poly}(d')}$.
\end{lemma}

\begin{proof}[Proof of Lemma~\ref{lemma:gaussian_example_lemma1}]
	For any $z\sim \mathcal{N}(\mu_i, \frac{1}{d'}\cdot I_{d'\times d'})$ and any $j\ne i$, by the tail bound of gaussian distribution we have
	\begin{align*}
	\Pr_{z\sim \mathcal{N}(\mu_i, \frac{1}{d'}\cdot I_{d'\times d'})} \left((z-\mu_i)^\top \left(\frac{\mu_j-\mu_i}{\norm{\mu_j-\mu_i}_2}\right)\lesssim \frac{\sqrt{\log d}}{\sqrt{d'}}\right) \ge 1-\frac{1}{\textup{poly}(d)}.
	\end{align*}
	
	Also, for $\xi\sim\mathcal{N}(0, \frac{1}{d}\cdot I_{d\times d})$, when $\sigma\le\frac{1}{\sqrt{d}}$ we have
	\begin{align*}
	\Pr_{\xi\sim \mathcal{N}(0, \frac{\sigma^2}{d}\cdot I_{d'\times d'})} \left(\norm{\xi}_2\lesssim \frac{\sqrt{\log d}}{\sqrt{d}}\right) \ge 1-\frac{1}{\textup{poly}(d)}.
	\end{align*}
	
	Notice that $\norm{Q^{-1}(Q(z)+\xi) - z}_2\le\kappa \norm{\xi}$, we can set $\norm{\mu_i-\mu_j}\gtrsim\kappa \frac{\sqrt{\log d}}{\sqrt{d}}$.
	Therefore, when $\norm{\mu_i-\mu_j}\gtrsim \kappa \frac{\sqrt{\log d}}{\sqrt{d'}}$ we can combine the above two cases and have 
	\begin{align*}
	\Pr_{z\sim \mathcal{N}(\mu_i, \frac{1}{d'}\cdot I_{d'\times d'}), \xi\sim \mathcal{N}(0, \frac{\sigma^2}{d}\cdot I_{d'\times d'})}\left(P_i(z)>P_j(Q^{-1}(Q(z)+\xi))\right) \ge 1-\frac{1}{\textup{poly}(d)}.
	\end{align*}
	
	Since $r\le d$, we have
	\begin{align*}
	\Pr_{\bar{x}\sim\pndata, {x}\sim\aug{\bar{x}}}({y}({x})\ne y(\bar{x})) 
	\ge 1-\frac{1}{\textup{poly}(d)}.
	\end{align*}
\end{proof}

We use the following lemma to give a lower bound for the sparest $m$-partition of the augmentation graph in Example~\ref{example:gaussian}.

\begin{lemma}\label{lemma:gaussian_example_lemma2}
	In the setting of Theorem~\ref{theorem:gaussian_example}, for any $k'>r$ and $\tau>0$, we have
	\begin{align*}
	\rho_{k'} \ge \frac{c_{\tau/\kappa}}{18} \cdot  \exp\left(-\frac{2c_{\sigma}\tau+\tau^2}{2{\sigma}^2/d}\right),
	\end{align*}
	where 
	\begin{align*}
	c_{\sigma} := {\sigma} \cdot  \Phi_d^{-1}(\frac{2}{3})
	\end{align*}	
	with $\Phi_{d}(z) := \Pr_{\xi\sim\mathcal{N}(0, \frac{1}{d}I_{d\times d})} (\norm{\xi}_2\le z)$, and
	\begin{align*}
	c_{\tau/\kappa} := \min_{p \in[0, \frac{3}{4}]} \frac{\Phi(\Phi^{-1}(p)+{\tau\sqrt{d}/\kappa})}{p} - 1
	\end{align*}	
	with 
	$
	\Phi(z) := \int_{-\infty}^z  \frac{e^{-u^2/2}}{\sqrt{2\pi}}du.
	$
\end{lemma}

The proof of Lemma~\ref{lemma:gaussian_example_lemma2} can be found in Section~\ref{section:proof_of_gaussian_lemma2}. Now we give the proof of Example~\ref{theorem:gaussian_example}.

\begin{proof}[Proof of Theorem~\ref{theorem:gaussian_example}]
	The result on $\alpha$ is directly from Lemma~\ref{lemma:gaussian_example_lemma1}. By concentration inequality, there must exists some universal constant $C>0$ such that for any $d\ge C$, we have $1-\Phi_{d}(\sqrt{\frac{3}{2}}) \le \frac{1}{3}$. When this happens, we have $\Phi_d^{-1}(\frac{2}{3})\le \sqrt{\frac{3}{2}}$. 
	Since for $d\le C$ we can just treat $d$ as constant, we have $\Phi_d^{-1}(\frac{2}{3})\lesssim 1$. Set $\tau = \sigma/d$ in Lemma~\ref{lemma:gaussian_example_lemma2}, we have $\rho_{k'}\gtrsim \frac{\sigma}{\kappa\sqrt{d}}$.
	Set $k'=\lfloor k/2\rfloor$, we apply Theorem~\ref{thm:combine_with_cheeger_simplified} and get the bound we need.
\end{proof}

\subsection{Proof of Lemma~\ref{lemma:gaussian_example_lemma2}}\label{section:proof_of_gaussian_lemma2}
In this section we give a proof for Lemma~\ref{lemma:gaussian_example_lemma2}. We first introduce the following claim which states that for a given subset of augmented data, any two data close in $L_2$ norm cannot have a very different chance of being augmented into this set.

\begin{claim}\label{claim:gaussian_example_claim1}
	In the setting of Theorem~\ref{theorem:gaussian_example}, given a set $S\subseteq \Real^d$. If $x\in \Real^d$ satisfies $\Pr(S|x) := \Pr_{\tilde{x}\sim\aug{x}} (\tilde{x}\in S) \ge\frac{2}{3}$. Then, for any $x'$ such that $\norm{x-x'}_2\le \tau$, we have 
	\begin{align*}
	\Pr(S|x') \ge \frac{1}{3}\cdot  \exp\left(-\frac{2c_{\sigma}\tau+\tau^2}{2{\sigma}^2}\right),
	\end{align*}
	where 
	\begin{align*}
	c_{\sigma} := {\sigma} \cdot  \Phi_d^{-1}(\frac{2}{3}),
	\end{align*}
	with $\Phi_{d}(z) := \Pr_{\xi\sim\mathcal{N}(0, \frac{1}{d}\cdot I_{d\times d})} (\norm{\xi}_2\le z)$.
\end{claim}

\begin{proof}[Proof of Claim~\ref{claim:gaussian_example_claim1}]
	By the definition of augmentation, we know 
	\begin{align*}
	\Pr(S|x) = \Exp{\xi\sim \mathcal{N}(0, \frac{{\sigma}^2}{d}\cdot I_{d\times d})}\left[\id{x+\xi\in S}\right].
	\end{align*}
	By the definition of $c_{\sigma}$, we have 
	\begin{align*}
	\Pr_{\xi\sim \mathcal{N}(0, \frac{{\sigma}^2}{d} \cdot I_{d\times d})} (\norm{\xi}_2\le c_{\sigma})  = \frac{2}{3}.
	\end{align*}
	Since $\Pr(S|x)\ge \frac{2}{3}$ by assumption, we have 
	\begin{align*}
	\Exp{\xi\sim \mathcal{N}(0, \frac{{\sigma}^2}{d} \cdot I_{d\times d})} \left[P(S|x+\xi) \cdot \id{\norm{\xi}_2\le c_{\sigma}}\right] \ge \frac{1}{3}.
	\end{align*}
	Now we can bound the quanity of our interest:
	\begin{align*}
	\Pr(S|x') &= \frac{1}{(2\pi {\sigma}^2/d)^{d/2}}\int_{\xi} e^{-\frac{\norm{\xi}_2^2}{2{\sigma}^2/d}} P(S|x'+\xi)d\xi\\
	&= \frac{1}{(2\pi {\sigma}^2/d)^{d/2}} \int_{\xi} e^{-\frac{\norm{\xi + x - x'}_2^2}{2{\sigma}^2/d}} P(S|x+\xi)d\xi\\
	&\ge \frac{1}{(2\pi {\sigma}^2/d)^{d/2}}\int_{\xi} e^{-\frac{\norm{\xi + x - x'}_2^2}{2{\sigma}^2/d}} P(S|x+\xi) \cdot \id{\norm{\xi}_2\le c_{\sigma}}d\xi\\
	&\ge \frac{1}{(2\pi {\sigma}^2/d)^{d/2}} \int_{\xi} e^{-\frac{2c_{\sigma}\tau+\tau^2 + \norm{\xi}_2^2}{2{\sigma}^2/d}} P(S|x+\xi) \cdot \id{\norm{\xi}_2\le c_{\sigma}}d\xi\\
	&= e^{-\frac{2c_{\sigma}\tau + \tau^2}{2{\sigma}^2/d}} \cdot  \Exp{\xi\sim \mathcal{N}(0, \frac{{\sigma}^2}{d}I_{d\times d})} \left[P(S|x+\xi) \cdot \id{\norm{\xi}_2\le c_{\sigma}}\right]\\
	&\ge \frac{1}{3}\cdot  \exp\left(-\frac{2c_{\sigma}\tau+\tau^2}{2{\sigma}^2/d}\right).
	\end{align*}	
\end{proof}

We now give the proof of Lemma~\ref{lemma:gaussian_example_lemma2}.

\begin{proof}[Proof of Lemma~\ref{lemma:gaussian_example_lemma2}]
	Let $S_1, \cdots, S_{k'}$ be the disjoint sets that gives $\rho_{k'}$ in Definition~\ref{definition:multi_way_expansion_constant}.	First we notice that when $k'>r$, there must exist $t\in[k']$ such that for all $i\in[r]$, we have
	\begin{align}\label{equation:less_than_half}
	\Pr_{x\sim P_i, \tilde{x}\sim \aug{x}}(\tilde{x}\in S_t)\le\frac{1}{2}.
	\end{align}
	WLOG, we assume $t=1$. So we know that 
	\begin{align}\label{equation:rho_g_lower_bound}
	\rho_{k'} &= \max_{i\in[k']} \phi_G(S_i) 
	\ge \phi_G(S_1) 
	\ge \min_{j\in[r]} \frac{\Exp{x\sim P_j} \left[\Pr(S_1|x)(1-\Pr(S_1|x))\right]}{\Exp{x\sim P_j}\left[\Pr(S_1|x)\right]},
	\end{align}
	where \begin{align*}
	\Pr(S|x) := \Pr_{\tilde{x}\sim\aug{x}} (\tilde{x}\in S).
	\end{align*}
	WLOG, we assume $j=1$ minimizes the RHS of Equation~\eqref{equation:rho_g_lower_bound}, so we only need to prove
	\begin{align*}
	\frac{\Exp{x\sim P_1} \left[\Pr(S_1|x)(1-\Pr(S_1|x))\right]}{\Exp{x\sim P_1}\left[\Pr(S_1|x)\right]} \ge \frac{c_{\tau/\kappa}}{18} \cdot  \exp\left(-\frac{2c_{\sigma}\tau+\tau^2}{2{\sigma}^2/d}\right).
	\end{align*} 
	We define the following set 
	\begin{align*}
	\setgaussian := \left\{x \bigg\vert  \Pr(S_1|x) \ge \frac{2}{3}\right\}.
	\end{align*}
	Notice that 
	\begin{align}\label{equation:at_least_one_true}
	\Exp{x\sim P_1}\left[\Pr(S_1|x)\right] &= \int_x P_1(x) \Pr(S_1|x) dx\nonumber\\
	&= \int_{x\in \setgaussian} P_1(x) \Pr(S_1|x) dx + \int_{x\notin \setgaussian} P_1(x) \Pr(S_1|x) dx.
	\end{align}
	We can consider the following two cases.
	
	\textbf{Case 1:} $\int_{x\notin \setgaussian} P_1(x) \Pr(S_1|x) dx \ge \frac{1}{2}\Exp{x\sim P_1}\left[\Pr(S_1|x)\right]$.
	
	This is the easy case because we have 
	\begin{align*}
	\Exp{x\sim P_1} \left[\Pr(S_1|x)(1-\Pr(S_1|x))\right] &\ge \int_{x\notin \setgaussian} P_1(x) \Pr(S_1|x) (1-\Pr(S_1|x))dx\\
	&\ge \frac{1}{3}\int_{x\notin \setgaussian} P_1(x) \Pr(S_1|x) dx\\
	&\ge \frac{1}{6}\Exp{x\sim P_1}\left[\Pr(S_1|x)\right].
	\end{align*}
	
	\textbf{Case 2:} $\int_{x\in \setgaussian} P_1(x) \Pr(S_1|x) dx \ge \frac{1}{2}\Exp{x\sim P_1}\left[\Pr(S_1|x)\right]$.
	
	Define neighbourhood of $\setgaussian$ as \begin{align*}
	N(\setgaussian) := \left\{x \bigg \vert \norm{x-a}_2 \le \tau \text{ for some } a\in \setgaussian\right\}.
	\end{align*}
	We have 
	\begin{align*}
	\Exp{x\sim P_1} \left[\Pr(S_1|x)(1-\Pr(S_1|x))\right] &\ge \int_{x\in N(\setgaussian)\backslash \setgaussian} P_1(x) \Pr(S_1|x) (1-\Pr(S_1|x))dx\\
	&\ge \frac{1}{9} \cdot  \exp\left(-\frac{2c_{\sigma}\tau+\tau^2}{2{\sigma}^2/d}\right)\cdot  \int_{x\in N(\setgaussian)\backslash \setgaussian}P_1(x) dx,
	\end{align*}
	where the second inequality is by Claim~\ref{claim:gaussian_example_claim1}. Notice that 
	\begin{align*}
	\int_{x\in \setgaussian} P_1(x)dx \le \frac{3}{2}\int_{x\in \setgaussian} P_1(x) \Pr(S_1|x)dx \le \frac{3}{2}\int_{x} P_1(x) \Pr(S_1|x)dx \le \frac{3}{4},
	\end{align*}
	where we use Equation~\eqref{equation:less_than_half}. Define set $\widetilde{\setgaussian} := Q^{-1}(\setgaussian)$ be the set in the ambient space corresponding to $\setgaussian$. Define
	\begin{align*}
	\widetilde{N}(\widetilde{\setgaussian} ) :=\left\{x'\in\Real^{d'}\bigg \vert \norm{x'-a}_2 \le \frac{\tau}{\kappa} \text{ for some } a\in \widetilde{\setgaussian}  \right\}
	\end{align*}
	Due to $Q$ being $\kappa$-bi-lipschitz, it is easy to see $\widetilde{N}(\widetilde{\setgaussian} ) \subseteq Q^{-1}\left(N(\setgaussian)\right)$.
	According to the Gaussian isoperimetric inequality~\cite{bobkov1997isoperimetric}, we have  
	\begin{align*}
	\int_{x\in {N}(\setgaussian)\backslash \setgaussian}P_1(x) dx \ge c_{\tau/\kappa} \int_{x\in \setgaussian} P_1(x)dx,
	\end{align*}
	where 
	\begin{align*}
	c_{\tau/\kappa}   := \min_{0 \le p \le 3/4} \frac{\Phi(\Phi^{-1}(p)+\tau\sqrt{d}/\kappa)}{p} - 1,
	\end{align*}
	with $\Phi(\cdot)$ is the Gaussian CDF function defined as
	\begin{align*}
	\Phi(z) := \int_{-\infty}^z  \frac{e^{-u^2/2}}{\sqrt{2\pi}}du.
	\end{align*}
	
	So we have 
	\begin{align*}
	\Exp{x\sim P_1} \left[\Pr(S_1|x)(1-\Pr(S_1|x))\right] &\ge \frac{c_{\tau/\kappa} }{9} \cdot  \exp\left(-\frac{2c_{\sigma}\tau+\tau^2}{2{\sigma}^2/d}\right)\cdot  \int_{x\in \setgaussian}P_1(x) dx\\
	&\ge \frac{c_{\tau/\kappa}}{9} \cdot  \exp\left(-\frac{2c_{\sigma}\tau+\tau^2}{2{\sigma}^2/d}\right)\cdot  \int_{x\in \setgaussian}P_1(x) \Pr(S_1|x) dx\\
	&\ge \frac{c_{\tau/\kappa}}{18} \cdot  \exp\left(-\frac{2c_{\sigma}\tau+\tau^2}{2{\sigma}^2/d}\right)\cdot  \Exp{x\sim P_1}\left[\Pr(S_1|x)\right].
	\end{align*}
	
	By Equation~\eqref{equation:at_least_one_true}, either case 1 or case 2 holds.
	Combining case 1 and case 2, we have 
	\begin{align*}
	\frac{\Exp{x\sim P_1} \left[\Pr(S_1|x)(1-\Pr(S_1|x))\right]}{\Exp{x\sim P_1}\left[\Pr(S_1|x)\right]} &\ge \min\left\{\frac{1}{6}, \frac{c_{\tau/\kappa}}{18} \cdot  \exp\left(-\frac{2c_{\sigma}\tau+\tau^2}{2{\sigma}^2/d}\right)\right\}\\
	&= \frac{c_{\tau/\kappa}}{18} \cdot  \exp\left(-\frac{2c_{\sigma}\tau+\tau^2}{2{\sigma}^2/d}\right).
	\end{align*}
	
\end{proof}

\section{Proofs for Section~\ref{section:finite_sample}}
\subsection{Proof of Theorem~\ref{theorem:nn_generalization}}\label{section:proof_for_nn_generalization}

We restate the empirical spectral contrastive loss defined in Section~\ref{section:finite_sample} as follows:

\begin{definition}[Empirical spectral contrastive loss]
	Consider a dataset $\edata = \{\bar{x}_1, \bar{x}_2, \cdots, \bar{x}_n\}$ containing $n$ data points i.i.d. sampled from $\pndata$. Let $\epndata$ be the uniform distribution over $\edata$. Let $\hat{P}_{\bar{x}, \bar{x}'}$ be the uniform distribution over data pairs $(\bar{x}_i, \bar{x}_j)$ where $i\ne j$. We define the empirical spectral contrastive loss of a feature extractor $f$ as
	\begin{align*}
	\eLoss{n}{f} := -2\Exp{\bar{x}\sim\epndata, \atop {x}\sim\aug{\bar{x}}, {x}'\sim\aug{\bar{x}}}\left[f({x})^\top f({x}')\right] + \Exp{(\bar{x}, \bar{x}')\sim \hat{P}_{\bar{x}, \bar{x}'},\atop {x}\sim \aug{\bar{x}},  {x}'\sim\aug{\bar{x}'}}\left[\left(f({x})^\top f({x}')\right)^2\right].
	\end{align*}
\end{definition}

The following claim shows that $\eLoss{n}{f}$ is an unbiased estimator of population spectral contrastive loss.

\begin{claim}\label{claim:eloss_equal_loss}
	$\eLoss{n}{f}$ is an unbiased estimator of $\Loss{f}$, i.e., 
	\begin{align*}
	\Exp{\edata}\left[\eLoss{n}{f}\right] = \Loss{f}.
	\end{align*}
\end{claim}

\begin{proof}
	This is because
	\begin{align*}
	\Exp{\edata}\left[\eLoss{n}{f}\right] &= -2\cdot \Exp{\edata}\left[\Exp{\bar{x}\sim\epndata, \atop {x}\sim\aug{\bar{x}}, {x}'\sim\aug{\bar{x}}}\left[f({x})^\top f({x}')\right]\right] + \Exp{\edata}\left[\Exp{(\bar{x}, \bar{x}')\sim \hat{P}_{\bar{x}, \bar{x}'},\atop {x}\sim \aug{\bar{x}},  {x}'\sim\aug{\bar{x}'}}\left[\left(f({x})^\top f({x}')\right)^2\right]\right]\\
	&= -2\Exp{\bar{x}\sim\pndata, \atop {x}\sim\aug{\bar{x}}, {x}'\sim\aug{\bar{x}}}\left[f({x})^\top f({x}')\right] + \Exp{\bar{x}\sim \pndata, \bar{x}'\sim \pndata,\atop {x}\sim \aug{\bar{x}},  {x}'\sim\aug{\bar{x}'}}\left[\left(f({x})^\top f({x}')\right)^2\right] 	=\Loss{f}.
	\end{align*}
\end{proof}

To make use of the Radmacher complexity theory, we need to write the empirical loss as the sum of i.i.d. terms, which is achieved by the following sub-sampling scheme:
\begin{definition}\label{definition:tuple}
	Given dataset $\edata$, we sample a subset of tuples as follows: first sample a permutation $\pi: [n]\rightarrow[n]$, then we sample tuples $S=\{(z_i, z_i^+, z_i')\}_{i=1}^{ n/2}$ as follows: 
	\begin{align*}
	&z_i\sim\aug{\bar{x}_{\pi(2i-1)}},\\
	&z_i^+\sim\aug{\bar{x}_{\pi(2i-1)}},\\
	&z_i'\sim\aug{\bar{x}_{\pi(2i)}}.		
	\end{align*}
	We define the following loss on $S$:
	\begin{align*}
	\sLoss{f} := \frac{1}{ n/2}\sum_{i=1}^{ n/2} \left[\left(f(z_i)^\top f(z_i')\right)^2 -2f(z_i)^\top f(z_i^+)\right].
	\end{align*}
\end{definition}

It is easy to see that $\sLoss{f}$ is an unbiased estimator of $\eLoss{n}{f}$:
\begin{claim}\label{claim:esloss_equal_eloss}
	For given $\edata$, if we sample $S$ as above, we have:
	\begin{align*}
	\Exp{S}\left[\sLoss{f}\right] = \eLoss{f}.
	\end{align*}
\end{claim}
\begin{proof}
	This is obvious by the definition of $\sLoss{f}$ and $\eLoss{n}{f}$.
\end{proof}

The following lemma reveals the relationship between the Rademacher complexity of feature extractors and the Rademacher complexity of the loss defined on tuples:
\begin{lemma}\label{lemma:Rademacher}
	Let $\mathcal{F}$ be a hypothesis class of feature extractors from $\adata$ to $\Real^k$. Assume $\norm{f(x)}_\infty \le \boundf$ for all $x\in\adata$. For $i\in[k]$, define $f_i:\adata\rightarrow \Real$ be the function such that $f_i(x)$ is the $i$-th dimension of $f(x)$. Let $\mathcal{F}_i$ be the hypothesis containing $f_i$ for all $f\in\mathcal{F}$. For $n\in\mathcal{Z}^+$, let $\rad{n}(\mathcal{F}_i)$ be the maximal possible empirical Rademacher complexity of $\mathcal{F}_i$ over $n$ data:
	\begin{align*}
	\rad{n}(\mathcal{F}_i) :=\max_{\{x_1, x_2, \cdots, x_n\}}\Exp{\sigma}\left[\sup_{f_i\in\mathcal{F}_i}\left(\frac{1}{n}\sum_{j=1}^n \sigma_j f_i(x_j)\right)\right],
	\end{align*}
	where $x_1, x_2, \cdots, x_m$ are in $\adata$, and $\sigma$ is a uniform random vector in $\{-1, 1\}^n$. Then, the empirical Rademacher complexity on any $n$ tuples $\{(z_i, z_i^+, z_i')\}_{i=1}^{n}$ can be bounded by
	\begin{align*}
	\Exp{\sigma}\left[\sup_{f\in\mathcal{F}}\left(\frac{1}{n}\sum_{j=1}^n \sigma_j \left(\left(f(z_j)^\top f(z_j')\right)^2 -2f(z_j)^\top f(z_j^+)\right)\right)\right] \le (16k^2\boundf^2+16k\boundf)\cdot\max_{i\in[k]}\rad{n}(\mathcal{F}_i).
	\end{align*}
\end{lemma}

\begin{proof}
	\begin{align*}
	&\Exp{\sigma}\left[\sup_{f\in\mathcal{F}}\left(\frac{1}{n}\sum_{j=1}^n \sigma_j \left(\left(f(z_j)^\top f(z_j')\right)^2 -2f(z_j)^\top f(z_j^+)\right)\right)\right] \\
	\le & \Exp{\sigma}\left[\sup_{f\in\mathcal{F}}\left(\frac{1}{n}\sum_{j=1}^n \sigma_j \left(f(z_j)^\top f(z_j')\right)^2\right)\right]  + 2 \Exp{\sigma}\left[\sup_{f\in\mathcal{F}}\left(\frac{1}{n}\sum_{j=1}^n \sigma_j f(z_j)^\top f(z_j^+)\right)\right] \\
	\le & 2k\boundf \Exp{\sigma}\left[\sup_{f\in\mathcal{F}}\left(\frac{1}{n}\sum_{j=1}^n \sigma_j f(z_j)^\top f(z_j')\right)\right] + 2 \Exp{\sigma}\left[\sup_{f\in\mathcal{F}}\left(\frac{1}{n}\sum_{j=1}^n \sigma_j f(z_j)^\top f(z_j^+)\right)\right]\\
	\le & (2k^2\boundf+2k)\max_{z_1, z_2, \cdots, z_{n}\atop z_1', z_2', \cdots, z_{n}' } \max_{i\in[k]}\Exp{\sigma}\left[\sup_{f_i\in\mathcal{F}_i}\left(\frac{1}{n}\sum_{j=1}^n \sigma_j f_i(z_j) f_i(z_{j}')\right)\right],
	\end{align*}
	here the second inequality is by Talagrand's lemma. Notice that for any $z_1, z_2\cdots z_n$ and $z_1', z_2', \cdots, z_n'$ in $\adata$ and any $i\in[k]$ we have
	\begin{align*}
	&\Exp{\sigma}\left[\sup_{f_i\in\mathcal{F}_i}\left(\frac{1}{n}\sum_{j=1}^n \sigma_j f_i(z_j) f_i(z_{j}')\right)\right]\\
	\le & \frac{1}{2} \Exp{\sigma}\left[\sup_{f_i\in\mathcal{F}_i}\left(\frac{1}{n}\sum_{j=1}^n \sigma_j \left(f_i(z_j)+ f_i(z_{j}')\right)^2\right)\right] + \frac{1}{2} \Exp{\sigma}\left[\sup_{f_i\in\mathcal{F}_i}\left(\frac{1}{n}\sum_{j=1}^n \sigma_j \left(f_i(z_j)- f_i(z_{j}')\right)^2\right)\right]\\
	\le & 4\boundf \Exp{\sigma}\left[\sup_{f_i\in\mathcal{F}_i}\left(\frac{1}{n}\sum_{j=1}^n \sigma_j f_i(z_j)\right)\right] + 4\boundf\Exp{\sigma}\left[\sup_{f_i\in\mathcal{F}_i}\left(\frac{1}{n}\sum_{j=1}^n \sigma_j f_i(z_j')\right)\right],
	\end{align*}
	where the first inequaltiy is by Talagrand's lemma. Combine these two equations and we get:
	\begin{align*}
	&\Exp{\sigma}\left[\sup_{f\in\mathcal{F}}\left(\frac{1}{n}\sum_{j=1}^n \sigma_j \left(\left(f(z_j)^\top f(z_j')\right)^2 -2f(z_j)^\top f(z_j^+)\right)\right)\right] \\
	\le& (16k^2\boundf^2+16k\boundf)\max_{z_1, z_2, \cdots, z_{n}} \max_{i\in[k]}\Exp{\sigma}\left[\sup_{f_i\in\mathcal{F}_i}\left(\frac{1}{n}\sum_{j=1}^n \sigma_j f_i(z_j)\right)\right]. 
	\end{align*}
\end{proof}

\begin{proof}[Proof of Theorem~\ref{theorem:nn_generalization}]
	By Claim~\ref{claim:eloss_equal_loss} and Claim~\ref{claim:esloss_equal_eloss}, we know that $\Exp{S}[\sLoss{f}] = \Loss{f}$, where $S$ is sampled by first sampling $\edata$ then sample $S$ according to Definition~\ref{definition:tuple}.
	Notice that when $\edata$ contains $n$ i.i.d. samples natural data, the set of random tuples $S$ contains $n$ i.i.d tuples. Therefore, we can apply generalization bound with Rademacher complexity to get a uniform convergence bound. In particular, by Lemma~\ref{lemma:Rademacher} and notice the fact that $\left(f(z_j)^\top f(z_j')\right)^2 -2f(z_j)^\top f(z_j^+)$ always take values in range $[-2k\boundf^2, 2k\boundf^2+k^2\boundf^4]$, we apply standard generalization analysis based on Rademacher complexity and get:
	with probability at least $1-\delta^2/4$ over the randomness of $\edata$ and $S$, we have for any $f\in\mathcal{F}$, 
	\begin{align}\label{equation:generalization_bound}
	\Loss{f} \le \sLoss{f} + (32k^2\boundf^2+32k\boundf) \max_{i\in[k]}\rad{\npt/2}(\mathcal{F}_i) + (4k\boundf^2 + k^2\boundf^4)\cdot\sqrt{\frac{4\log 2/\delta}{\npt}}.
	\end{align}
	This means with probability at least $1-\delta/2$ over random $\edata$, we have: with probability at least $1-\delta/2$ over random tuples $S$ conditioned on $\edata$, Equation~\eqref{equation:generalization_bound} holds. Since both $\Loss{f}$ and $\eLoss{\npt}{f}$ take value in range $[-2k\boundf^2, 2k\boundf^2+k^2\boundf^4]$, we have: with probability at least $1-{\delta}/{2}$ over random $\edata$, we have for any $f\in \mathcal{F}$, 
	\begin{align*}
	\Loss{f} \le \eLoss{\npt}{f} + (32k^2\boundf^2+32k\boundf)\cdot \max_{i\in[k]}\rad{\npt/2}(\mathcal{F}_i) + (4k\boundf^2 + k^2\boundf^4)\cdot\left(\sqrt{\frac{4\log 2/\delta}{\npt}} +\frac{\delta}{2}\right).
	\end{align*}
	Since negating the functions in a function class doesn't change its Rademacher complexity, we also have the other direction: with probability at least $1-{\delta}/2$ over random $\edata$, we have for any $f\in\mathcal{F}$,
	\begin{align*}
	\Loss{f} \ge \eLoss{\npt}{f} - (32k^2\boundf^2+32k\boundf)\cdot \max_{i\in[k]}\rad{\npt/2}(\mathcal{F}_i) + (4k\boundf^2 + k^2\boundf^4)\cdot\left(\sqrt{\frac{4\log 2/\delta}{\npt}} +\frac{\delta}{2}\right).
	\end{align*}
	Combine them together we get the excess risk bound: with probability at least $1-\delta$, we have
	\begin{align*}
	\Loss{\hat{f}} \le \Loss{\fstarf} + (64k^2\boundf^2+64k\boundf)\cdot \max_{i\in[k]}\rad{\npt/2}(\mathcal{F}_i) + (8k\boundf^2 +2 k^2\boundf^4)\cdot\left(\sqrt{\frac{4\log 2/\delta}{\npt}} +\frac{\delta}{2}\right),
	\end{align*}
	where $\hat{f}$ is minimizer of $\eLoss{\npt}{f}$ in $\mathcal{F}$ and $\fstarf$ is minimizer of $\Loss{f}$ in $\mathcal{F}$. Set $c_1 = 64k^2\boundf^2+64k\boundf$ and $c_2 = 16k\boundf^2 +4k^2\boundf^4$ and notice that $\max_{i\in[k]}\rad{\npt/2}(\mathcal{F}_i)  = \rad{\npt/2}(\mathcal{F})$ finishes the proof.
\end{proof}

\subsection{Generalization bound for spectral contrastive learning with deep neural networks}\label{section:proof_of_nn_norm_contral}

In this section, we examplify Theorem~\ref{theorem:nn_generalization} with the norm-contralled Rademacher complexity bound introduced in~\cite{golowich2018size}, which gives the following theorem. 

\begin{theorem}\label{theorem:nn_generalization_norm_contral}	
	Assume $\adata$ is a subset of Euclidean space $\Real^{d}$ and $\norm{x}_2\le \boundx$ for any $x\in\adata$. 
	Let $\mathcal{F}$ be a hypothesis class of norm-contralled $l$-layer deep neural networks defined as 
	\begin{align*}
	\left\{x\rightarrow P_{\boundf}(W_l\sigma(W_{l-1} \sigma(\cdots \sigma (W_1 x)))): \norm{W_i}_F \le \boundwi{i} \right\}
	\end{align*}
	where $\sigma(\cdot)$ is element-wise ReLU activation,  $P_{\boundf}(\cdot)$ is element-wise projection to interval $[-\boundf, \boundf]$ for some $\boundf>0$, $\boundwi{i}$ is the norm bound of the $i$-th layer, $W_l$ has $k$ rows and $W_1$ has $d$ columns. Then, with probability at least $1-\delta$ over randomness of a dataset with size $2\npt$, we have
	\begin{align*}
	\Loss{\hat{f}} \le \mathcal{L}_{\mathcal{F}}^* + c_1\cdot \frac{\boundx\boundw\sqrt{l}}{\sqrt{\npt}} + c_2\cdot\left(\sqrt{\frac{\log 1/\delta}{\npt}} +{\delta}\right),
	\end{align*}
	where $\hat{f}$ is the minimizer of $\eLoss{2\npt}{f}$ in $\mathcal{F}$, $\mathcal{L}_{\mathcal{F}}^*$ is the minimal $\Loss{f}$ achievable by any function $f\in\mathcal{F}$, $\boundw:=\prod_{i=1}^l \boundwi{i}$, constants $c_1\lesssim k^2\boundf^2+k\boundf$ and $c_2 \lesssim k\boundf^2 + k^2\boundf^4$.
\end{theorem}

\begin{proof}[Proof of Theorem~\ref{theorem:nn_generalization_norm_contral}]
	Consider the following hypothesis class of real-valued neural networks:
	\begin{align*}
	\mathcal{F}_{\textup{real}} \triangleq\left\{x\rightarrow \widehat{W}_l\sigma(W_{l-1} \sigma(\cdots \sigma (W_1 x))): \norm{W_i}_F \le \boundwi{i} \right\}
	\end{align*}
	where $\sigma(\cdot)$ is element-wise ReLU activation and $\boundwi{i}$ is the norm bound of the $i$-th layer defined in the theorem, $W_l$ has $k$ rows and $\widehat{W}_1$ is a vector. By Theorem 1 of~\cite{golowich2018size}, we have
	\begin{align*}
	\rad{\npt}\left(\mathcal{F}_{\textup{real}}\right) \le \frac{\boundx(\sqrt{2\log(2)l}+1)\boundw}{\sqrt{\npt}}.
	\end{align*}
	Let the projection version of this hyposis class be:
	\begin{align*}
	\mathcal{F}_{\textup{real+proj}} \triangleq\left\{x\rightarrow P_{\boundf}(\widehat{W}_l\sigma(W_{l-1} \sigma(\cdots \sigma (W_1 x)))): \norm{W_i}_F \le \boundwi{i} \right\},
	\end{align*}
	where $P_{\boundf}(\cdot)$ projects a real number into interval $[-\boundw, \boundw]$.
	Notice that $P_{\boundf}(\cdot)$ is $1$-Lipschitz, by Telegrand's lemma we have
	\begin{align*}
	\rad{\npt}\left(\mathcal{F}_{\textup{real+proj}}\right) \le \frac{\boundx(\sqrt{2\log(2)l}+1)\boundw}{\sqrt{\npt}}.
	\end{align*}
	
	For each $i\in[k]$, define function $f_i:\adata\rightarrow \Real$ such that $f_i(x)$ is the $i$-th dimension of $f(x)$, define $\mathcal{F}_i$ be the hypothesis class including all $f_i$ for $f\in\mathcal{F}$. Then when $\mathcal{F}$ is the composition of deep neural networks and projection function as defined in the theorem, it is obvious to see that $\mathcal{F}_i = \mathcal{F}_{\textup{real+proj}}$ for all $i\in[k]$. Therefore, by Theorem~\ref{theorem:nn_generalization} we have
	\begin{align*}
	\Loss{\hat{f}} \le \mathcal{L}_{\mathcal{F}}^* + c_1\cdot \frac{\boundx(\sqrt{2\log(2)l}+1)\boundw}{\sqrt{\npt}} + c_2\cdot\left(\sqrt{\frac{\log 2/\delta}{\npt}} +{\delta}\right),
	\end{align*}
	and absorbing the constants into $c_1$ finishes the proof.
\end{proof}

\subsection{Proof of Theorem~\ref{theorem:linear_probe_suboptimal}}\label{section:roof_of_suboptimal_representation_bound}
In this section we give the proof of Theorem~\ref{theorem:linear_probe_suboptimal}. We will first prove the following theorem that characterize the error propagation from pre-training to the downstream task.
\begin{theorem}[Error propagation from pre-training to the downstream task]\label{theorem:erorr_propagation}
	Assume representation dimension $k\ge 4r+2$, Assumption~\ref{definition:accurate_partition} holds for $\alpha>0$ and Assumption~\ref{assumption:realizable} holds.
	Recall $\eigvadj_i$ be the $i$-th largest eigenvalue of the normalized adjacency matrix. 
	Then, for any $\epsilon>0$ and $\empminf\in\mathcal{F}$ such that $\Loss{\empminf}<\Loss{\globalminf} + \epsilon$, we have:
	\begin{align*}
		\eval(\empminf) \lesssim \frac{\alpha}{\rho_{\lfloor k/2\rfloor}^2}\cdot\log k + \frac{k\epsilon}{\Delta_\eigvadj^2} ,
		\end{align*}
	where $\Delta_\eigvadj:=\eigvadj_{\lfloor{3k}/4\rfloor} - \eigvadj_{k}$ is the eigenvalue gap between the $\lfloor{3k}/4\rfloor$-th and the $k$-th eigenvalue. Furthermore, there exists a linear head $\widehat{B}\in\Real^{k\times r}$ that achieves this error and  has norm bound 
	\begin{align}
		\norm{\widehat{\matrixw}}_F \le \frac{2(k+1)}{1-\lambda_{k}}.
	\end{align}
\end{theorem}

We first introduce the following definitions of $\epsilon$-optimal minimizers of matrix approximation loss and population spectral contrastive loss:
\begin{definition}
	We say a function $\hat{f}_{\ma}$ is $\epsilon$-optimal minimizer of matrix approximation loss $\mathcal{L}_{\ma}$ if \begin{align*}
	\Lossmc{\widehat{F}_{\ma}} \le \min_F \Lossmc{F}+\epsilon,
	\end{align*}
	where $\widehat{F}_{\ma}$ is $\hat{f}_{\ma}$ written in the matrix form.
	We say a function $\hat{f}$ is  $\epsilon$-optimal minimizer of spectral contrastive loss $\mathcal{L}$ if \begin{align*}
	\Loss{\hat{f}} \le \min_f \Loss{f}+\epsilon.
	\end{align*}
\end{definition}

We introduce the following generalized version of Theorem~\ref{theorem:error_rate_with_eigenvectors}, which captures the main effects of error in the representation.
\begin{theorem}\label{theorem:generalization_of_error_rate_with_eigenvectors}
	[Generalization of Theorem~\ref{theorem:error_rate_with_eigenvectors}]
	Assume the set of augmented data $\adata$ is finite. Let $\lambda_{i}$ be the $i$-th smallest eigenvalue of the normalize laplacian matrix. 
	Let $\hat{f}\in\arg\min_{f: \adata\rightarrow\Real^k} $ be a $\epsilon$-optimal minimizer of the spectral contrastive loss function $\Loss{f}$ with $k\in\mathcal{Z}^+$. 
	Then, for any labeling function $\hat{y}: \adata\rightarrow[r]$ there exists a linear probe $\widehat{\matrixw} \in \Real^{r\times k}$ with norm bound $\norm{\widehat{\matrixw}}_F \le \frac{2(k+1)}{1-\lambda_{k}}$ such that
		\begin{align*}
		\Exp{\bar{x}\sim\pndata, {x}\sim\aug{\bar{x}}}\left[\norm{\vec{{y}}(\bar{x}) - \widehat{B}\empminf({x})}_2^2\right]\lesssim   \min_{1 \le k'\le k}  \left(\frac{ \phi^{\hat{y}}}{\lambda_{k'+1}}+ \frac{k' \epsilon}{(\lambda_{k+1}-\lambda_{k'})^2 }\right)+ \Delta(y, \hat{y}),
	\end{align*} 
and 	
	\begin{align*}
	\Pr_{\bar{x}\sim \pndata, {x}\sim\aug{\bar{x}}} \left(\pred_{\hat{f}, \widehat{\matrixw} }({x}) \ne y(\bar{x})\right) \lesssim \min_{1 \le k'\le k}  \left(\frac{ \phi^{\hat{y}}}{\lambda_{k'+1}}+ \frac{k' \epsilon}{(\lambda_{k+1}-\lambda_{k'})^2 }\right) + \Delta(y, \hat{y}), 
	\end{align*}
	where $\phi^{\hat{y}}$ and $\Delta(y, \hat{y})$ are defined in Equations~\ref{equation:define_phi_y} and~\ref{equation:define_delta} respectively.
\end{theorem}

The proof of Theorem~\ref{theorem:generalization_of_error_rate_with_eigenvectors} is deferred to Section~\ref{section:proof_of_generalized_error_with_eig}. 

Now we are ready to prove Theorem~\ref{theorem:linear_probe_suboptimal} using Theorem~\ref{theorem:generalization_of_error_rate_with_eigenvectors}. 
\begin{proof}[Proof of Theorem~\ref{theorem:erorr_propagation}]
	In Theorem~\ref{theorem:generalization_of_error_rate_with_eigenvectors} we let $k'=\lfloor\frac{3}{4}k\rfloor$ on the RHS of the bound and get: for any $\hat{y}:\adata\rightarrow[r]$ there exists $\widehat{\matrixw}\in\Real^{r\times k}$ such that 
	\begin{align*}
	\Pr_{x\sim \pndata, \tilde{x}\sim\aug{x}} \left(\pred_{\hat{f}, \widehat{\matrixw} }(\tilde{x}) \ne y(x)\right) \lesssim  \frac{ \phi^{\hat{y}}}{\lambda_{\lfloor\frac{3}{4}k\rfloor+1}}+ \frac{k \epsilon}{(\lambda_{k+1}-\lambda_{\lfloor\frac{3}{4}k\rfloor})^2 } + \Delta(y, \hat{y}).
	\end{align*}
	Let $S_1, S_2, \cdots, S_r$ be the partition of $\adata$ induced by the classifier $g$ in Assumption~\ref{definition:accurate_partition}.
	Define function $\hat{y}:\adata\rightarrow [r]$ as follows: for an augmented datapoint $x\in\adata$,  we use function $\hat{y}(x)$ to represent the index of set that $x$ is in, i.e., $x\in S_{\hat{y}(x)}$. Then by Lemma~\ref{remark:edge_between_sets} we have $\phi^{\hat{y}}\le2\alpha$ and $\Delta(y, \hat{y}) \le \alpha$. In Lemma~\ref{lemma:higher_order_cheeger} let $(1+\zeta)t=\lfloor\frac{3}{4}k\rfloor+1$ and $t=\lfloor \frac{k}{2}\rfloor$, then there is $\zeta\ge0.5$, so we have: there exists a partition $S_1,\cdots, S_{\lfloor \frac{k}{2}\rfloor}\subset \adata$ such that $\phi_G(S_i)\lesssim \sqrt{\lambda_{\lfloor\frac{3}{4}k\rfloor+1}\log {(k)}}$ for $\forall i \in[\lfloor \frac{k}{2}\rfloor]$. 
	By Definition~\ref{definition:multi_way_expansion_constant}, we have $\rho_{\lfloor \frac{k}{2}\rfloor}  \lesssim \sqrt{\lambda_{\lfloor\frac{3}{4}k\rfloor+1}\log {(k)}}$, which leads to $\frac{1}{\lambda_{\lfloor\frac{3}{4}k\rfloor+1}}\lesssim \frac{log(k)}{\rho_{\lfloor \frac{k}{2}\rfloor}^2}$. So we have 
	
	\begin{align*}
	\Pr_{\bar{x}\sim \pndata, {x}\sim\aug{\bar{x}}} \left(\pred_{\hat{f},  \widehat{\matrixw} }({x}) \ne y(\bar{x})\right) &\lesssim  \frac{\alpha}{\rho_{\lfloor \frac{k}{2}\rfloor}^2}\cdot \log(k)+ \frac{k \epsilon}{(\lambda_{k+1}-\lambda_{\lfloor\frac{3}{4}k\rfloor})^2 }\\
	&\lesssim \frac{\alpha}{\rho_{\lfloor \frac{k}{2}\rfloor}^2}\cdot \log(k)+ \frac{k \epsilon}{(\lambda_{k}-\lambda_{\lfloor\frac{3}{4}k\rfloor})^2 }.
	\end{align*}
	
	Notice that by the definition of ensembled linear probe predictor, $\npred_{\hat{f}, \widehat{\matrixw}}(\bar{x})\ne y(\bar{x})$ happens only if more than half of the augmentations of $\bar{x}$ predicts differently from $y(\bar{x})$, so we have $\Pr_{\bar{x}\sim\pndata}\left(\npred_{\hat{f}, \widehat{\matrixw}}\ne y(\bar{x})\right) \le 2\Pr_{\bar{x}\sim \pndata, {x}\sim\aug{\bar{x}}} \left(\pred_{\hat{f},  \widehat{\matrixw} }({x}) \ne y(\bar{x})\right)$ which finishes the proof.
\end{proof}

\begin{proof}[Proof of Theorem~\ref{theorem:linear_probe_suboptimal}]
	Theorem~\ref{theorem:linear_probe_suboptimal} is a direct corollary of Theorem~\ref{theorem:nn_generalization} and Theorem~\ref{theorem:erorr_propagation}.
\end{proof}

\subsection{Proof of Theorem~\ref{theorem:generalization_of_error_rate_with_eigenvectors}}\label{section:proof_of_generalized_error_with_eig}
In this section, we give the proof for Theorem~\ref{theorem:generalization_of_error_rate_with_eigenvectors}.

\begin{lemma}[Generalization of Lemma~\ref{lemma:equivalent_between_two_losses}]\label{lemma:generalization_of_equivalent_between_two_losses}
	Let ${f}: \adata\rightarrow\Real^k$ be a feature extractor, matrix
	${{F}}\in\Real^{N\times k}$ be such that its $x$-th row is $\sqrt{\wnode{x}}\cdot {f}(x)$. Then, ${{F}}$ is an $\epsilon$-optimal minimizer of $\Lossmc{{F}}$ if and only if 
	$f$ is an $\epsilon$-optimal minimizer of the population spectral contrastive loss $\Loss{f}$. 
\end{lemma}
\begin{proof}[Proof of Lemma~\ref{lemma:generalization_of_equivalent_between_two_losses}]
	The proof follows the proof of Lemma~\ref{lemma:equivalent_between_two_losses}.
\end{proof}

We will use the following two lemmas about $\epsilon$-optimal minimizer of $\mathcal{L}_{\ma}$:

\begin{lemma}\label{lemma:properties_of_suboptimal_feature}
	Let $\lambda_i$ be the $i$-th minimal eigenvalue of the normalized Laplacian matrix $\laplacian$ with corrsponding unit-norm eigenvector $\eigv_i$.	
	Let $F\in \Real^{N\times k}$ be an $\epsilon$-optimal minimizer of $\mathcal{L}_{\ma}$. Let $\Pi_{f}\eigv_i$ be the projection of $\eigv_i$ onto the column span of $F$.
	Then, there exists vector $b\in\Real^k$ with norm bound $\norm{b}\le \norm{F}_F/(1-\lambda_{i})$ such that 
	\begin{align}
		\norm{\Pi_{f}\eigv_i - Fb}_2^2 \le \frac{\epsilon}{(1-\lambda_{i})^2}.
	\end{align}
	Furthermore, the norm of $F$ is bounded by 
	\begin{align}
		\norm{F}^2_F\le 2(k+\epsilon).
	\end{align}
\end{lemma}
\begin{proof}[Proof of Lemma~\ref{lemma:properties_of_suboptimal_feature}]
Since columns of $\norA-\Pi_{f}\norA$ and columns of $\Pi_{f}\norA - FF^\top$ are in orthogonal subspaces, we have
\begin{align}
	\norm{\norA - FF^\top}_F^2 &= \norm{\norA-\Pi_{f}\norA}_F^2 + \norm{\Pi_{f}\norA - FF^\top}_F^2.
\end{align}
On one hand, since $\Pi_{f}\norA$ is a rank-$k$ matrix, we know that 
$\norm{\norA-\Pi_{f}\norA}_F^2 \ge \min_F \Lossmc{F}$. On the other hand, by the definition of $\epsilon$-optimal minimizer, we have $\norm{\norA - FF^\top}_F^2 \le \min_F \Lossmc{F}+\epsilon$. Thus, we have 
\begin{align}
	\norm{\Pi_{f}\norA - FF^\top}_F^2\le \epsilon.
\end{align}
Since $\norA =\sum_{i=1}^{N} (1-\lambda_{i})v_iv_i^\top$, we have $v_i = \frac{1}{1-\lambda_{i}}\norA \eigv_i$. Thus, 
\begin{align}
	\Pi_{f} \eigv_i = \frac{1}{1-\lambda_{i}} \Pi_{f} (\norA v_i) = \frac{1}{1-\lambda_{i}} FF^\top v_i + \frac{1}{1-\lambda_{i}}(\Pi_{f}\norA - FF^\top) v_i.
\end{align}
Let $b=\frac{1}{1-\lambda_{i}}F^\top v_i$, we have
\begin{align}
	\norm{\Pi_{f}v_i - Fb}_2^2 &= \frac{1}{(1-\lambda_{i})^2} \norm{(\Pi_{f}\norA - FF^\top)v_i}_2^2\\
	&\le \frac{1}{(1-\lambda_{i})^2} \norm{\Pi_{f}\norA - FF^\top}_F^2\\
	&\le \frac{\epsilon}{(1-\lambda_{i})^2}.
\end{align}

To bound the norm of $F$, we first notice that
\begin{align}
	\norm{\Pi_{f}\norA}_F^2 = \Tr(\norA^2 \Pi_{f}) \le \Tr(\Pi_{f}) = k,
\end{align}
where the inequality uses that fact that $\norA^2$ has operator norm at most $1$. Combine this result with $\norm{\Pi_{f}\norA - FF^\top}_F^2\le \epsilon$ we have
\begin{align}
	\norm{FF^\top}_F\le \sqrt{k} +\sqrt{\epsilon}.
\end{align}
Since $FF^\top$ has rank at most $k$, we can write its SVD deocmposition as $FF^\top = U\Sigma U^\top$ where $U\in\Real^{N\times k}$ and $\Sigma\in\Real{k\times k}$. As a result, we have
\begin{align}
	\norm{F}_F^2 = \Tr(FF^\top) = \Tr(\Sigma) \le \sqrt{k}\sqrt{\Tr(\Sigma^2)} = \sqrt{k} \norm{FF^\top}_F \le k+\sqrt{k\epsilon}\le 2(k+\epsilon).
\end{align}

\end{proof}

\begin{lemma}\label{lemma:projection_to_suboptimal_representation}
	Let $\lambda_i$ be the $i$-th minimal eigenvalue of the normalized Laplacian matrix $\laplacian$ with corrsponding unit-norm eigenvector $\eigv_i$. Let $F\in \Real^{N\times k}$ be an $\epsilon$-optimal minimizer of $\mathcal{L}_{\ma}$. Let $\Pi_{f}^{\perp}\eigv_i$ be the projection of $\eigv_i$ onto the subspace orthogonal to the column span of $F$. Then, for $i\le k$ we have
	\begin{align*}
	\norm{\Pi_{f}^{\perp}\eigv_i}_2^2 \le \frac{\epsilon}{(\lambda_{k+1} - \lambda_{i})^2}.
	\end{align*}  
\end{lemma}
\begin{proof}	

	Recall normalized adjacency matrix $\norA = I-L$. We use $\norA_i$ to denote the $i$-th column of $\norA$. We use $\widehat{A}$ to denote matrix $FF^\top$ and $\widehat{A}_i$ to denote the $i$-th column of $\widehat{A}$. Let $z_1, \cdots, z_k$ be unit-norm orthogonal vectors in the column span of $F$. Since the column span of $\widehat{A}$ is the same as the column span of $F$, we know columns of $\widehat{A}$ are in $span\{z_1, \cdots, z_k\}$. Let $z_{k+1}, \cdots, z_{\sizead}$ be unit-norm orthogonal vectors such that together with $z_1, \cdots, z_k$ they form an orthonormal basis of $\Real^{\sizead}$. We use $\Pi_f$ and $\Pi_f^\perp$ to denote matrices $\sum_{j=1}^k z_jz_j^\top$ and $\sum_{j=k+1}^{\sizead} z_jz_j^\top$ respectively, then for any vector $v\in\Real^{\sizead}$, vectors $\Pi_{f}v$ and $\Pi_{f}^\perp v$ are the projections of $v$ onto the column span of $F$ and its orthogonal space respectively.
	
	We first give a lower bound of $\Lossmc{F}$ as follows:
	\begin{align*}
	\Lossmc{F} &= \norm{\norA - \widehat{A}}_F^2 = \sum_{j=1}^{\sizead} \norm{\norA_j - \widehat{A}_j}_2^2 \ge \sum_{j=1}^{\sizead} \norm{\norA_j - \Pi_f \norA_j}_2^2 \\&= \sum_{j=1}^{\sizead} \norm{\norA_j - \left(\sum_{t=1}^k z_tz_t^\top\right)\norA_j}_2^2 = \sum_{j=1}^{\sizead} \norm{\left(\sum_{t=k+1}^{\sizead} z_tz_t^\top\right)\norA_j}_2^2 \\&= \norm{\left(\sum_{t=k+1}^{\sizead} z_tz_t^\top\right) \norA}_F^2 = \norm{\Pi_{f}^\perp \norA}_F^2.
	\end{align*}
	where the first equality is by definition of $\Lossmc{F}$, the second equality is by writing the Frobenius norm square as the sum of column norm square, the inequality is because $\widehat{A}_j$ must be in the span of $z_1, \cdots, z_k$ while $\Pi_f \norA_j$ is the vector in this span that is closest to $\norA_j$, the third equality is writing the projection function in the matrix form, the fourth equality is because $z_1, \cdots z_d$ are an orthonormal basis, the fifth equality is rewriting to Frobenius norm, and the last equality is by definition of $\Pi_f^\perp$.
	
	Notice that 
	\begin{align*}
	\norm{\Pi_{f}^\perp \norA}_F^2 = Tr\left(\norA^\top {\Pi_{f}^\perp}^\top \Pi_{f}^\perp \norA\right) = Tr\left(\norA^\top \Pi_{f}^\perp \norA\right) = Tr\left(\norA\norA^\top \Pi_{f}^\perp\right).
	\end{align*}	
	
	We can rewrite the above lower bound as
	\begin{align*}
	\Lossmc{F} \ge Tr\left(\norA\norA^\top \Pi_{f}^\perp\right)
	= Tr\left(\sum_{j=1}^{\sizead} (1-\lambda_j)^2 \eigv_j\eigv_j^\top \sum_{t=k+1}^{\sizead}z_tz_t^\top\right)
	= \sum_{j=1}^{\sizead}\sum_{t=k+1}^{\sizead} (1-\lambda_j)^2 \langle\eigv_j, z_t\rangle^2.
	\end{align*}
	
	We define variable $S_j \triangleq \sum_{t=1}^j \sum_{l=k+1}^d \langle \eigv_t, z_l \rangle^2$ for any $j\in[\sizead]$. Also denote $\lambda_{d+1}=1$. We have the following equality:
	\begin{align*}
	\sum_{j=1}^{\sizead}\sum_{t=k+1}^{\sizead} (1-\lambda_j)^2 \langle\eigv_j, z_t\rangle^2 = \sum_{j=1}^{\sizead} \left((1-\lambda_j)^2 - (1-\lambda_{j+1})^2 \right)S_j.
	\end{align*} 
	Notice that $S_j\ge 0$ and also when $i\le j \le k$, we have $S_j \ge \norm{\Pi_f^\perp \eigv_i}_2^2$, we have 
	\begin{align*}
	\sum_{j=1}^{\sizead}\sum_{t=k+1}^{\sizead} (1-\lambda_j)^2 \langle\eigv_j, z_t\rangle^2 \ge \left((1-\lambda_i)^2 - (1-\lambda_{k+1})^2 \right) \norm{\Pi_f^\perp \eigv_i}_2^2 + \sum_{j=k+1}^{\sizead} \left((1-\lambda_j)^2 - (1-\lambda_{j+1})^2 \right)S_j,
	\end{align*}
	where we replace every $S_j$ with $0$ when $j< k$, replace $S_j$ with $\norm{\Pi_f^\perp \eigv_i}_2^2$ when $i\le j \le k$, and keep $S_j$ when $j\ge k+1$. Now notice that 
	\begin{align*}
	S_{\sizead} = \sum_{t=1}^{\sizead} \sum_{l=k+1}^{\sizead} \langle \eigv_t, z_l \rangle^2 = \sum_{l=k+1}^{\sizead} \sum_{t=1}^{\sizead}  \langle \eigv_t, z_l \rangle^2 = \sum_{l=k+1}^{\sizead} \norm{z_l}_2^2 = \sizead-k,
	\end{align*}
	and also 
	\begin{align*}
	S_{j+1} - S_j =  \sum_{l=k+1}^{\sizead} \langle \eigv_{j+1}, z_l \rangle^2 \le \sum_{l=1}^{\sizead} \langle \eigv_{j+1}, z_l \rangle^2 = 1,
	\end{align*} 
	there must be $S_j \ge j-k$ when $j\ge k+1$. So we have 
	\begin{align*}
	&\ \ \sum_{j=1}^{\sizead}\sum_{t=k+1}^{\sizead} (1-\lambda_j)^2 \langle\eigv_j, z_t\rangle^2 \\&\ge \left((1-\lambda_i)^2 - (1-\lambda_{k+1})^2 \right) \norm{\Pi_f^\perp \eigv_i}_2^2 + \sum_{j=k+1}^{\sizead} \left((1-\lambda_j)^2 - (1-\lambda_{j+1})^2 \right)(j-k)\\
	&= \left((1-\lambda_i)^2 - (1-\lambda_{k+1})^2 \right) \norm{\Pi_f^\perp \eigv_i}_2^2 + \sum_{j=k+1}^{\sizead} (1-\lambda_{j})^2\\
	&= \left((1-\lambda_i)^2 - (1-\lambda_{k+1})^2 \right) \norm{\Pi_f^\perp \eigv_i}_2^2 +   \min_{F\in\Real^{N\times k}} \Lossmc{F},
	\end{align*}
	where the last equality is by Eckart–Young–Mirsky Theorem. So we know 
	\begin{align}
		\Lossmc{F}\ge \left((1-\lambda_i)^2 - (1-\lambda_{k+1})^2 \right) \norm{\Pi_f^\perp \eigv_i}_2^2 +   \min_{F\in\Real^{N\times k}} \Lossmc{F},
	\end{align}
	which implies that 
	\begin{align}
		\norm{\Pi_f^\perp \eigv_i}_2^2 \le \frac{\epsilon}{(1-\lambda_i)^2 - (1-\lambda_{k+1})^2}\le \frac{\epsilon}{(\lambda_{k+1} - \lambda_{i})^2}.
	\end{align}
\end{proof}

The following lemma generalizes Lemma~\ref{lemma:approximate_vector_of_small_rayleight_quotient}.

\begin{lemma}
	[Generalization of Lemma~\ref{lemma:approximate_vector_of_small_rayleight_quotient}]\label{lemma:generalization_of_approximate_vector_of_small_rayleight_quotient}
	Let $\laplacian$ be the normalized Laplacian matrix of graph $G=(\adata, w)$, where $|\adata|=N$. Let $f:\adata \rightarrow \Real^k$ be an $\epsilon$-optimal minimizer of $\Lossmc{f}$. Let $F$ be the matrix form of $f$ and $F_i$ is the $i$-th column of $F$. Let $R({u}) := \frac{{u}^\top \laplacian {u}}{{u}^\top {u}}$ be the Rayleigh quotient of a vector ${u}\in \Real^{\sizead}$ . Then, for any $k\in \mathcal{Z}^+$ such that $ k<\sizead$, there exists a vector ${\vectorw}\in \Real^{k}$ such that 
	\begin{align*}
	\norm{{u} - Fb }_2^2 \le \min_{1 \le k'\le k}\left(\frac{3R(u)}{\lambda_{k'+1}} + \frac{6k' \epsilon}{ (\lambda_{k+1}- \lambda_{k'})^2} \right)\norm{u}_2^2.
	\end{align*} 
Furethermore, the norm of $b$ is upper bounded by 
\begin{align}
	\norm{b}_2\le \frac{2(k+1)}{1-\lambda_{k}} \norm{u}_2.
\end{align}
\end{lemma}
\begin{proof}
	Let $k'$ be the choice that minimizes the right hand side. We use $p_\eigv(u)$ to denote the projection of $u$ onto the span of $\eigv_1, \cdots, \eigv_{k'}$. We denote the coefficients as $p_\eigv(u) = \sum_{i=1}^{k'} \rho_i \eigv_i$. For every $i\in[k']$, let $b_i$ be the vector in Lemma~\ref{lemma:properties_of_suboptimal_feature}. Define vector $b=\sum_{i=1}^{k'}\rho_i b_i$.
	
	We use $p_{\eigv, f}(u)$ to denote the projection of $p_\eigv(u)$ onto the span of $f_1, \cdots, f_k$. Then we know that 
	\begin{align}\label{eq1:lemma:generalization_of_approximate_vector_of_small_rayleight_quotient}
	\norm{u-Fb}_2^2 \le 3 \norm{u-p_\eigv(u)}_2^2 + 3\norm{p_\eigv(u) - p_{\eigv, f}(u)}_2^2 + 3\norm{p_{\eigv, f}(u) - Fb}_2^2. 
	\end{align}
	
	By the proof of Lemma~\ref{lemma:approximate_vector_of_small_rayleight_quotient}, we know that 
	\begin{align}\label{eq2:lemma:generalization_of_approximate_vector_of_small_rayleight_quotient}
	\norm{u-p_\eigv(u)}_2^2 \le \frac{R(u)}{\lambda_{k'+1}}\norm{u}_2^2.
	\end{align}
	
	For the second term, we have
	\begin{align}
	\norm{p_\eigv(u) - p_{\eigv f}(u)}_2^2 &= \norm{\Pi_{f}^\perp p_\eigv(u)}_2^2\nonumber\\
	&= \norm{\sum_{i=1}^{k'}\Pi_{f}^\perp \eigv_i \eigv_i^\top u}_2^2\nonumber\\
	&\le\left(\sum_{i=1}^{k'} \norm{\Pi_{f}^\perp \eigv_i }_2^2\right) \cdot \left( \sum_{i=1}^{k'} (\eigv_i^\top u)^2\right)\nonumber\\
	&\le \frac{k' \epsilon}{ (\lambda_{k+1}-\lambda_{k'})^2} \norm{u}_2^2,
	\label{eq3:lemma:generalization_of_approximate_vector_of_small_rayleight_quotient}
	\end{align}
	where the first inequality if by Cauchy–Schwarz inequality
	and the second inequality if by Lemma~\ref{lemma:projection_to_suboptimal_representation}. 
	
	For the third term, we have
	\begin{align}\label{eq4:lemma:generalization_of_approximate_vector_of_small_rayleight_quotient}
		\norm{p_{\eigv, f}(u) - Fb}_2^2 &= \norm{\sum_{i=1}^{k'} \rho_i (\Pi_{f}v_i - Fb_i)}_2^2\nonumber\\
		&\le k' \sum_{i=1}^{k'} \rho_i^2 \norm{\Pi_{f}v_i - Fb_i}_2^2\nonumber\\
		&\le \frac{k'\epsilon}{(1-\lambda_{k'})^2} \norm{u}_2^2,
	\end{align}
	where the first inequality is by Cauchy-Schwarz inequality, and the second inequality is by Lemma~\ref{lemma:properties_of_suboptimal_feature}.
	Plugging Equation~\eqref{eq2:lemma:generalization_of_approximate_vector_of_small_rayleight_quotient},  Equation~\eqref{eq3:lemma:generalization_of_approximate_vector_of_small_rayleight_quotient}, and Equation~\eqref{eq4:lemma:generalization_of_approximate_vector_of_small_rayleight_quotient} into Equation~\eqref{eq1:lemma:generalization_of_approximate_vector_of_small_rayleight_quotient} finishes the proof.
	
	To bound the norm of $b$, we use Lemma~\ref{lemma:properties_of_suboptimal_feature} and have 
	\begin{align}
		\norm{b}_2^2 = \norm{\sum_{i=1}^{k'} \rho_i b_i}_2^2
		\le k' \sum_{i=1}^{k'} \rho_i^2 \norm{b_i}_2^2
		\le \frac{k' \norm{u}_2^2}{(1-\lambda_{k'})^2}\norm{F}_F^2
		\le \frac{2k'(k+1) }{(1-\lambda_{k'})^2} \norm{u}_2^2.
	\end{align}
	
\end{proof}

Now we prove Theorem~\ref{theorem:generalization_of_error_rate_with_eigenvectors} using the above lemmas.

\begin{proof}[Proof of Theorem~\ref{theorem:generalization_of_error_rate_with_eigenvectors}]
	Let $\widehat{F}\in\Real^{N\times k}$ be such that its $x$-th row is $\sqrt{w_x}\cdot \hat{f}(x)$. By Lemma~\ref{lemma:generalization_of_equivalent_between_two_losses}, $\widehat{F}$ is an $\epsilon$-optimal minimizer of $\Lossmc{F}$.
	
	For each $i\in [r]$, we define the function $u_i(x) = \id{\hat{y}(x)=i}\cdot \sqrt{\wnode{x}}$. Let $u:\adata\rightarrow\Real^k$ be the function such that $u(x)$ has $u_i$ at the $i$-th dimension. By Lemma~\ref{lemma:generalization_of_approximate_vector_of_small_rayleight_quotient}, there exists a vector $b_i\in\Real^k$ such that
	\begin{align*}
	\norm{u_i - \widehat{F}b_i}_2^2 \le \min_{1 \le k'\le k}\left(\frac{3R(u_i)}{\lambda_{k'+1}} + \frac{6k' \epsilon}{(\lambda_{k+1} - \lambda_{k'})^2} \right) \norm{u_i}_2^2
	\end{align*}
	Let matrices $U = [u_1, \cdots, u_r]$ and $\widehat{B}^\top = [b_1, \cdots, b_r]$. We sum the above equation over all $i\in [r]$ and get 
	\begin{align}
	\norm{U - \widehat{F}\widehat{B}^\top}_F^2 &\le \sum_{i=1}^r \min_{1 \le k'\le k}\left(\frac{3R(u_i)}{\lambda_{k'+1}} + \frac{6k' \epsilon}{(\lambda_{k+1} - \lambda_{k'})^2} \right) \norm{u_i}_2^2\nonumber\\
	&\le \min_{1 \le k'\le k} \sum_{i=1}^r \left(\frac{3R(u_i)}{\lambda_{k'+1}}\norm{u_i}_2^2 + \frac{6k' \epsilon}{(\lambda_{k+1} - \lambda_{k'})^2} \norm{u_i}_2^2\right) .\label{eq1:theorem: generalization_of_error_rate_with_eigenvectors}
	\end{align}
	Notice that 
	\begin{align}
	\sum_{i=1}^r R(u_i) \norm{u_i}_2^2 &= \sum_{i=1}^r \frac{1}{2}\phi_i^{\hat{y}} \sum_{x\in \adata} \wnode{x}\cdot \id{\hat{y}(x) = i}\nonumber\\
	&= \frac{1}{2}\sum_{i=1}^r \sum_{x, x' \in \adata} \wpair{x}{x'}\cdot \id{(\hat{y}(x) = i\land \hat{y}(x')\ne i) \text{ or }(\hat{y}(x) \ne i \land \hat{y}(x')= i)}\nonumber\\
	&= \frac{1}{2} \sum_{x, x' \in \adata} \wpair{x}{x'}\cdot \id{\hat{y}(x)\ne \hat{y}(x')} = \frac{1}{2}\phi^{\hat{y}},\label{eq2:theorem: generalization_of_error_rate_with_eigenvectors}
	\end{align}
	where the first equality is by Claim~\ref{corollary:sparsecut_to_rayleigh_quotient}. On the other hand, we have
	\begin{align}
	\sum_{i=1}^r \norm{u_i}_2^2 = \sum_{i=1}^r \sum_{x\in \adata}\wnode{x}\cdot \id{\hat{y}(x) = i} = \sum_{x\in \adata}\wnode{x} =1.\label{eq3:theorem: generalization_of_error_rate_with_eigenvectors}
	\end{align}
	Plugging Equation~\eqref{eq2:theorem: generalization_of_error_rate_with_eigenvectors} and Equation~\eqref{eq3:theorem: generalization_of_error_rate_with_eigenvectors} into Equation~\eqref{eq1:theorem: generalization_of_error_rate_with_eigenvectors} gives us
	\begin{align*}
	\norm{ U-\widehat{F}\widehat{B}^\top}_F^2\le \min_{1 \le k'\le k}  \left(\frac{3\phi^{\hat{y}}}{2\lambda_{k'+1}} + \frac{3k' \epsilon}{(\lambda_{k+1} - \lambda_{k'})^2}\right).
	\end{align*}
	
	Notice that by definition of $u(x)$, we know that prediction $\pred_{\hat{f}, \widehat{B}}(x)\ne \hat{y}(x)$ only happens if $\norm{u(x) - \widehat{B}\hat{f}(x)}_2^2 \ge \frac{\wnode{x}}{2}$. Hence we have 
	\begin{align*}
	\sum_{x\in \adata} \frac{1}{2} \wnode{x} \cdot \id{\pred_{\hat{f}, \widehat{B}}(x) \ne \hat{y}(x)} \le \sum_{x\in \adata} \norm{u(x) - \widehat{B}\hat{f}(x)}_2^2 = \norm{U-\widehat{F}\widehat{B}^\top}_F^2.
	\end{align*}
	
	Now we are ready to bound the error rate on $\adata$:
	\begin{align*}
	\Pr_{x\sim \adata}(\pred_{\hat{f}, \widehat{B}}(x) \ne \hat{y}(x))
	=\sum_{x\in \adata} \wnode{x} \cdot \id{\pred_{\hat{f}, \widehat{B}}(x) \ne \hat{y}(x)}\\
	\le  2 \cdot \norm{U-\hat{f}\widehat{B}^\top}_F^2
	\le \min_{1 \le k'\le k}  \left(\frac{3\phi^{\hat{y}}}{\lambda_{k'+1}} + \frac{6k' \epsilon}{(\lambda_{k+1} - \lambda_{k'})^2}\right).
	\end{align*}
	Here for the equality we are using the fact that $\Pr(x) = \wnode{x}$. We finish the proof by noticing that by the definition of $\Delta(y, \hat{y})$:
	\begin{align*}
	\Pr_{\bar{x}\sim \pndata, {x}\sim\aug{\bar{x}}} \left(\pred_{\hat{f}, \widehat{\matrixw} }({x}) \ne y(\bar{x})\right) &\le \Pr_{\bar{x}\sim \pndata, {x}\sim\aug{\bar{x}}} \left(\pred_{\hat{f}, \widehat{\matrixw} }({x}) \ne \hat{y}({x})\right) + \Pr_{\bar{x}\sim \pndata, {x}\sim\aug{\bar{x}}} \left(y(\bar{x}) \ne \hat{y}({x})\right)\\
	&\le \min_{1 \le k'\le k}  \left(\frac{3 \phi^{\hat{y}}}{\lambda_{k'+1}}+ \frac{6k' \epsilon}{(\lambda_{k+1}-\lambda_{k'})^2 }\right) + \Delta(y, \hat{y}).
	\end{align*}

The norm of $\widehat{B}$ can be bounded using Lemma~\ref{lemma:generalization_of_approximate_vector_of_small_rayleight_quotient} as:
\begin{align}
	\norm{\widehat{B}}_F \le \frac{2(k+1)}{1-\lambda_{k}} \sqrt{\sum_{i=1}^r \norm{u_i}_2^2} =  \frac{2(k+1)}{1-\lambda_{k}}.
\end{align}
\end{proof}

\section{Proofs for Section~\ref{section:linear_probe_finite_sample}}\label{section:proof_linear_probe_finite_sample}
In this section we give the proof of Theorem~\ref{theorem:end_to_end}. 
\begin{proof}[Proof of Theorem~\ref{theorem:end_to_end}]
	Let $\empminf$ be the minimizer of the empirical spectral contrastive loss. Let $\epsilon = \Loss{\empminf}-\Loss{\globalminf}$. 
	We abuse notation and use $y_i$ to denote $y(\bar{x}_i)$, and let $z_i = \empminf(x_i)$.
	We first study the average empirical Rademacher complexity of the capped quadratic loss on a dataset $\{(z_i, y_i)\}_{i=1}^\nds$, where $(z_i, y_i)$ is sampled as in Section~\ref{section:linear_probe_finite_sample}:
	\begin{align*}
	\rad{\nds}(\ell) :=& \Exp{\{(z_i, y_i)\}_{i=1}^\nds}\Exp{\sigma}\left[\sup_{\norm{\matrixw}_F\le C_k} \frac{1}{\nds}\left[\sum_{i=1}^{\nds}\sigma_i\ell((z_i, y_i), \matrixw)\right]\right]\\
	\le& 2r \Exp{\{(z_i, y_i)\}_{i=1}^\nds}\Exp{\sigma}\left[\sup_{\norm{\vectorw}_2\le C_k} \frac{1}{\nds}\left[\sum_{i=1}^{\nds}\sigma_iw^\top z_i\right]\right] \\
	\le& 2rC_k \sqrt{\frac{\Exp{}[\norm{z_i}^2]}{\nds}} \le  2rC_k\sqrt{\frac{2(k+\epsilon)}{\nds}},
	\end{align*} 
	where the first inequality uses Talagrand's lemma and the fact that $\ell_{\sigma}$ is $2$-Lipschitz, the second inequality is by standard Rademacher complexity of linear models, the third inequality is by the feature norm bound in Lemma~\ref{lemma:properties_of_suboptimal_feature}.
	
	By Theorem~\ref{theorem:generalization_of_error_rate_with_eigenvectors} and follow the proof of Theorem~\ref{theorem:erorr_propagation}, we know that there exists a linear probe $\widehat{B}^*$ with norm bound $\norm{\widehat{B}^*}_F\le C_k$ such that 
	\begin{align*}
	\Exp{\bar{x}\sim\pndata, {x}\sim\aug{\bar{x}}}\left[\ell\left((\empminf({x}), y(\bar{x})), \widehat{B}^*\right)\right]\lesssim\frac{\alpha}{\rho_{\lfloor \frac{k}{2}\rfloor}^2}\cdot \log(k)+ \frac{k \epsilon}{(\lambda_{k}-\lambda_{\lfloor\frac{3}{4}k\rfloor})^2 }.
	\end{align*}
	Let $\widehat{\matrixw}$ be the minimizer of $\sum_{i=1}^\nds \ell\left((z_i, y_i), \matrixw\right)$ subject to $\norm{{\matrixw}}_F\le C_k$, then by standard generalization bound, we have: with probability at least $1-\delta$, we have 
	\begin{align*}
	\Exp{\bar{x}\sim\pndata, {x}\sim\aug{\bar{x}}}\left[\ell\left((\empminf({x}), y(\bar{x})), \widehat{\matrixw}\right)\right]\lesssim  \frac{\alpha}{\rho_{\lfloor \frac{k}{2}\rfloor}^2}\cdot \log(k)+ \frac{k \epsilon}{(\lambda_{k}-\lambda_{\lfloor\frac{3}{4}k\rfloor})^2 } + \frac{rC_k\sqrt{k+\epsilon}}{\sqrt{\nds}} + \sqrt{\frac{\log 1/\delta}{\nds}}.
	\end{align*}
	Notice that $y(\bar{x})\ne\pred_{\empminf, \widehat{\matrixw}}({x})$ only if $\ell\left((\empminf({x}), y(\bar{x})), \widehat{\matrixw}\right)\ge\frac{1}{2}$, we have that when $\epsilon<1$ the error bound
	\begin{align*}
	\Pr_{\bar{x}\sim \pndata, {x}\sim\aug{\bar{x}}} \left(\pred_{\empminf, \widehat{\matrixw} }({x}) \ne y(\bar{x})\right) \lesssim \frac{\alpha}{\rho_{\lfloor \frac{k}{2}\rfloor}^2}\cdot \log(k)+ \frac{k \epsilon}{(\lambda_{k}-\lambda_{\lfloor\frac{3}{4}k\rfloor})^2 } + \frac{rC_k\sqrt{k}}{\sqrt{\nds}} + \sqrt{\frac{\log 1/\delta}{\nds}}.
	\end{align*}
	The result on $\npred_{\empminf, \widehat{\matrixw}}$ naturally follows by the definition of $\npred$.	
	When $\epsilon>1$ clearly the bound is also true since LHS is always smaller than $1$, so we know that the above bound is true for any $\epsilon$. Plug in the bound for $\epsilon$ from Theorem~\ref{theorem:nn_generalization} finishes the proof.
\end{proof}

\section{Formal statements for population with infinite supports}\label{sec:infinite_data}
In the main body of the paper, we make the simplifying assumption that the set of augmented data $\adata$ is finite (but could be exponential in dimension). Although this is a reasonable assumption given that modern computers store data with finite bits so the possible number of all data has to be finite, one might wonder whether our theory can be generalized to the case where $\adata$ is infinite (e.g., the entire Euclidean space $\Real^d$ for some integer $d>0$). In this section, we show that our theory can be straightforwardly extended to the case when $\adata$ has infinite supports with some additional regularity conditions. In fact, almost all proofs remain the same as long as we replace sum by integral, finite graph by an infinite graph, adjacency matrix by adjacency operator, and eigenvectors by the eigenfunctions. 


For simplicity, we consider the case when $\adata=\Real^d$ is the set of all augmented data.\footnote{When $\adata$ is a subset of $\Real^d$ equipped with a base measure $\mu$, then we will need to replace every $dx$ by $d\mu$ in the formulation below.} The weight matrix $w_{xx'}$ now becomes a weight function $w:\adata\times \adata\rightarrow \Real$. As usual, let $w(x,x')$ be the marginal probability of generating the pair $x$ and $x'$ from a random natural datapoint $\bar{x}\sim \pndata$. Or in other words, $w$ is the p.d.f. of the joint distribution of a random positive pair. For any $u\in\adata$, define the marginal weight function $w(u)\triangleq \int w(u, z) dz$. A sufficient (but not necessary) condition for our theory to hold is as follows:
\begin{assumption}[Regularity conditions]\label{assumption:regularity}
	The distribution $w$ satisfies the following conditions:
	
	(i) For any $u\in\adata$, the marginal distribution is well-defined and bouned $w(u) = \int w(u, z) dz<\infty$. 
	
	(ii) There exists $B>0$ such that for every $u, v\in\adata$, the conditional probability with respect to one variable is upper bounded by the marginal probability of the other variable $\frac{w(u, v)}{w(u)}\le B\cdot  w(v)$.
\end{assumption}

We note that our bound does not depend on value of $B$---we only the existence of $B$ for a qualitative purpose.  When the regularity conditions above hold, we will show that there exists an eigenfunction of the infinite adjacency graph is an analog to the eigenvectors of Laplacian that we introduced in Section~\ref{section:proof_for_main_result}. 

Let $L_2(\Real^d)$ be the set of all $L_2$ integratable functions $L_2(\Real^d)\triangleq\{f: \Real^d\rightarrow \Real\  \vert\ \int f(z)^2dz<\infty\}$. 
 For functions $f, g\in L_2(\Real^d)$, define their inner product as $\langle f, g\rangle \triangleq \int f(z)g(z)dz$. Note that $\ell_2(\Real^d)$ is a Hilbert space. 

To generalize the Laplacian matrix and eigenvectors to the infinite-size $\adata$ setting, we consider the notions of Laplacian operators and eigenfunctions. Let $H: L_2(\Real^d)\rightarrow L_2(\Real^d)$ be a linear operator, a function $f\in L_2(\Real^d)$ is an eigenfunction of $H$ if $H(f)(u) = \lambda f(u)$ for any $u\in\adata$, where $\lambda\in\Real$ is the corresponding eigenvalue. We define the Laplacian operator as $L: L_2(\Real^d)\rightarrow L_2(\Real^d)$ such that for every $u\in\adata$ and function $f\in L_2(\Real^d)$, we have 
\begin{align}
L(f)(u)  = f(u) - \int \frac{w(u, v)}{\sqrt{w(u)w(v)}}f(v)dv.
\end{align}

The following theorem shows the existence of eigenfunctions of the Laplacian operator.
\begin{theorem}[Existence of Eigenfunctions]\label{theorem:eigenfunction}
	When Assumption~\ref{assumption:regularity} is satisfied, there exists an orthonormal basis $\{f_i\}_{i=1}^\infty$ of $L_2(\Real^d)$ such that $L(f_i) = \lambda_if_i$. Furthermore, the eigenvalues  satisfy $\lambda_i\in [0, 1]$ and $\lambda_i\le \lambda_{i+1}$ for any $i\ge 0$.
\end{theorem}
\begin{proof}[Proof of Theorem~\ref{theorem:eigenfunction}]
	Define kernel function $k(u, v)\triangleq \frac{w(u, v)}{\sqrt{w(u)w(v)}}$, we have
	\begin{align}
	\int k(u, v)^2 dudv = \int \frac{w(u, v)^2}{w(u)w(v)} dudv \le B \int w(u, v)dudv = B<\infty.
	\end{align}
	Let $I$ be the identity operator, then $L-I$ is a Hilbert–Schmidt integral operator~\cite{enwiki:986771357}, so the spectral theorem~\cite{bump1998automorphic} applies to $L-I$ hence also applies to $L$. By the spectral theorem, there exists an orthonormal basis $\{f_i\}_{i=1}^\infty$ of $L_2(\Real^d)$ such that $L(f_i) = \lambda_if_i$. 
	
	Notice that 
	\begin{align}
	\lambda_i = \langle f_i, L(f_i)\rangle = \langle f_i, f_i\rangle - \int \frac{w(u, v)}{\sqrt{w(u)w(v)}} f_i(u)f_i(v)dudv.
	\end{align} 
	On the one hand, since $w(u, v)\ge0$ and $\langle f_i, f_i\rangle=1$, we have $\lambda_i\le 1$. On the other hand, notice that by Cauchy-Schwart inequality, 
	\begin{align}
	\int \frac{w(u, v)}{\sqrt{w(u)w(v)}} f_i(u)f_i(v)dudv &\le \sqrt{\int f_i(u)^2 \frac{w(u, v)}{w(u)}dudv\cdot \int f_i(v)^2 \frac{w(u, v)}{w(v)}dudv} = 
	\langle f_i, f_i\rangle,
	\end{align}
	so $\lambda_i\ge 0$, which finishes the proof.
\end{proof}

Given the existence of eigenfunctions guaranteed by Theorem~\ref{theorem:eigenfunction}, our results Theorem~\ref{thm:combine_with_cheeger_simplified}, Theorem~\ref{theorem:linear_probe_suboptimal} and Theorem~\ref{theorem:end_to_end} can all be easily generalized to the infinite-size $\adata$ case following exactly the same proof. For example, in the context of Lemma~\ref{lem:spectral-contrastive-loss}, $u_x$ will be replaced by $u(x):\Real^d \rightarrow \Real$ which belongs to $L_2(\Real^d)$, and $f(x) = w(x)^{1/2} u(x)$ as a result belongs to $L_2(w)$. Let $\Lossmc{f} = \langle u, Lu\rangle_{L_2(\Real^d)} = \langle f, Lf\rangle_{L_2(w)}$. The rest of the derivations follows by replacing the sum in equation~\eqref{equation:loss_ma_derive} by integral (w.r.t to Lebesgue measure).  More details on the normalized Laplacian operator and spectral clustering can be found in~\cite{schiebinger2015geometry}.

We omit the proof for simplicity.

\end{document}